\documentclass[11pt]{article}
\usepackage{custom_tex}

\title{Bayes-Optimal Fair Classification 
with Linear Disparity Constraints via Pre-, In-, and Post-processing}

\author 
{Xianli Zeng\footnote{
Department of Statistics and Data Science, Xiamen University.		\texttt{xlzeng@xmu.edu.cn}.},\,\,
Kevin Jiang,\footnote{Department of Statistics and Data Science, University of Pennsylvania.	 
\texttt{kcjiang@wharton.upenn.edu}.}\,\,
Guang Cheng\footnote{Dept of Statistics and Data Science, University of California, Los Angeles. 
\texttt{guangcheng@ucla.edu}.}\,\,
and
Edgar Dobriban\footnote{
Department of Statistics and Data Science, University of Pennsylvania.		\texttt{dobriban@wharton.upenn.edu}.}}

\date{\today}

\begin{document}
 
\begin{sloppypar}
\maketitle

\begin{abstract}
Machine learning algorithms
may have disparate impacts on protected groups. 
To address this, 
we develop methods for Bayes-optimal  fair classification, 
aiming to minimize classification error subject to 
given group fairness constraints.
We introduce the notion of \emph{linear disparity measures}, which are linear functions of a probabilistic classifier;
and \emph{bilinear disparity measures},
which are also linear in the group-wise regression functions.
We show that several popular 
    disparity measures---the deviations from demographic parity, equality of opportunity, and predictive equality---are bilinear.
    
     We find the form of Bayes-optimal fair classifiers under a single linear disparity measure, by uncovering a connection with the Neyman-Pearson lemma.
    For bilinear disparity measures, 
    {we are able to find the explicit form of Bayes-optimal fair classifiers as group-wise thresholding rules with explicitly characterized thresholds.}
    We develop similar algorithms for when protected attribute cannot be used at the prediction phase.
    Moreover, 
     we obtain analogous theoretical characterizations of optimal classifiers for a multi-class protected attribute and for equalized odds.  
   
    Leveraging our theoretical results, we design methods
that learn fair Bayes-optimal classifiers under bilinear disparity constraints. 
Our methods cover
three popular approaches to fairness-aware classification, via pre-processing (Fair Up- and Down-Sampling), 
in-processing
(Fair cost-sensitive Classification)
and 
post-processing
(a Fair Plug-In Rule).
Our methods control disparity directly while achieving near-optimal fairness-accuracy tradeoffs. 
{We show empirically  that our methods have state-of-the-art performance compared to existing algorithms.
In particular, our pre-processing method can a reach higher accuracy than prior pre-processing methods at low disparity levels.}
\end{abstract}

\tableofcontents

\section{Introduction}
\label{introduction}

Machine learning
is increasingly deployed in algorithmic decision-making systems in high-stakes domains deeply impacting human lives, including in education \citep{tsai2020precision}, finance \citep{MA201824}, healthcare \citep{MQ2017}, and judiciary systems \citep{JJSL2016}. While machine learning can improve decision-making efficiency, 
it has also surfaced many potential ethical risks.

One significant concern is algorithmic bias. Recent studies have revealed that without considering fairness, machine learning algorithms often make 
decisions that disadvantage vulnerable demographic groups, thereby exacerbating social injustice and potentially violating human rights \citep{JJSL2016,FBC2016,bhanot2021problem,corbett2018measure}.

In light of these risks, algorithmic fairness has attracted widespread attention from the public, government, and academia. Organizations like the White House \citep{WH2016}, and UNESCO \citep{UNESCO2020}  have all called for considering fairness when applying automated decision-making. 
In response, the emerging field of fair machine learning has seen significant development. 
An increasing amount of research focuses on defining quantitative fairness metrics for various  applications \citep[e.g.,][etc]{CKP2009,Cynthiafairness2012,HPS2016,ZafarPreference2017,PleissCali2017,KusnerCounterfactual2017,Zhang_Bareinboim_2018,Narasimhan_Cotter_Gupta_Wang_2020,NandyRecommender2022}, designing algorithms that protect minority groups \citep[e.g.,][etc]{kamiran2012data,JKMR2016,ZLM2018,PKG2019,CHS2020,xian2023fair}, providing public software to help practitioners assess and improve fairness \citep[e.g.,][etc]{AIF360,Weerts_Fairlearn_Assessing_and_2023},
and studying the theoretical underpinnings of algorithmic fairness \citep[e.g.,][etc]{CSFG2017,MW2018,SC2021,zeng2022fair}. 
We refer the readers to e.g., \cite{mehrabi2021survey,caton2020fairness}, 
etc,
for reviews of advances in fair machine learning.

In this paper,  we 
develop methods for fair classification.  
In fair classification, the \emph{fair Bayes-optimal classifier} is the method with 
the best possible accuracy 
under a given disparity constraint. 
This method serves as an ultimate  ``best possible" classifier. 
We consider the question of
deriving Bayes-optimal classifiers under given group fairness constraints, 
which aim to equalize various quantities across protected groups.

\begin{table}[t]
  \centering
  \caption{Our contributions in comparison with prior theoretical work for Bayes-optimal classifiers:   [1]: \cite{CSFG2017}; [2]: \cite{MW2018}; [3]: \cite{CCH02019}; [4]: \cite{JiangWasserstein2020}; [5]: \cite{SC2021}; [6]: \cite{Weietal20a} [7]: 
\cite{chen2023post}; [8]: \cite{XuFair2023}. The comparison is made based on scope of the theoretical framework, the fairness metrics considered, and whether theoretically optimal algorithms are proposed. Within each category, several criteria are considered. Our methods satisfy all desired criteria, while the prior works satisfy only some of them.}
  \label{table:compare}
  \begin{tabular}{lccccccccc}
    \hline
References &   [1] & [2] & [3] & [4] &[5] &[6] & [7] & [8] &This Work\\\hline
    \multicolumn{10}{c}{\textbf{Scope of Theoretical Framework}} \\\hline
 {Approximate Fairness}                     &            & \checkmark &            & \checkmark &            & \checkmark & \checkmark & \checkmark & \checkmark \\
 {Explicit Form}                            &            &            & \checkmark & \checkmark & \checkmark &            &            & \checkmark & \checkmark \\
 {Multiple Constraints}                     &            &            &            &            &            & \checkmark & \checkmark &            & \checkmark \\
 {Pareto Analysis}                          &            & \checkmark &            &            &            &            &            & \checkmark & \checkmark \\
 {No Protected Attrib. $A$ at test time}    &            & \checkmark &            & \checkmark &            & \checkmark & \checkmark &            & \checkmark \\\hline
    \multicolumn{10}{c}{\textbf{Fairness Metrics Considered}} \\\hline
Demographic Parity                          & \checkmark & \checkmark &            & \checkmark & \checkmark & \checkmark & \checkmark & \checkmark & \checkmark \\
Equality of Opportunity                     & \checkmark & \checkmark & \checkmark &            &            &            & \checkmark &            & \checkmark \\
Predictive Equality                         & \checkmark &            &            &            &            &            &            &            & \checkmark \\
Equalized Odds                              &            &            &            &            &            & \checkmark & \checkmark &            & \checkmark \\\hline
    \multicolumn{10}{c}{\textbf{Theoretically Optimal Algorithms}}\\\hline
Pre-processing                              &            &            &            &            &            & \checkmark &            & \checkmark & \checkmark \\
In-processing                               &            &            &            & \checkmark &            &            &            &            & \checkmark \\
Post-processing                             & \checkmark & \checkmark & \checkmark & \checkmark & \checkmark & \checkmark & \checkmark & \checkmark & \checkmark \\\hline
  \end{tabular}
\end{table}

Prior work by
\cite{HPS2016} showed that for oblivious fairness measures (which only depend on the joint distributions of the outcome, the protected group, and the class prediction), and for perfect fairness, the Bayes-optimal classifiers can always be expressed in terms of the group-wise regression functions; but without giving an explicit algorithm to find them.
Further, given any predictor $\widehat Y$, they also show how to obtain a predictor that is a function of $\widehat Y$ and the protected attribute that achieves perfect equalized odds, but which may not be Bayes-optimal.

Further,
\cite{CSFG2017} showed 
that, for exact fairness with some specific disparity measures, 
fair Bayes-optimal classifiers 
are group-wise thresholding rules, but did not identify their exact form (see \Cref{table:compare} for a summary of, and comparison with, related work).  
For the specific fairness criteria of demographic parity and equality of opportunity (see \Cref{sec:pre}), 
\cite{MW2018} 
characterized fair Bayes-optimal classifiers 
by a connection to cost-sensitive risks (again without an explicit form). 
Explicit 
forms of fair Bayes-optimal classifiers were derived 
for 
perfect demographic parity in \cite{CCH02019} and 
for perfect equality of opportunity in 
 \cite{SC2021}; see also \cite{HPS2016} for a previous characterization. 

 By considering the Wasserstein distance as the disparity measure, \cite{JiangWasserstein2020,Silviageneral2020,xian2023fair} 
  proved that the fair classification problem is equivalent to a  Wasserstein-barycenter problem, and the perfectly fair Bayes-optimal classifier can be derived via optimal transport. 
  Follow-up work by \cite{XuFair2023} further characterized the Pareto Frontier---i.e., the optimal fairness-accuracy tradeoff---using the Wasserstein geodesic.
 Rather than considering the fair Bayes-optimal classifier,  \cite{Weietal20a} considered the estimation of the Bayes-optimal score function and derived the optimal transformed score function that satisfies fairness constraints. 
In a selective classification framework, \cite{rava2021burden} studied how to minimize 
outcomes without a decision, 
while equalizing  the  false selection rate among different protected groups.
More recently, 
in closely related work,
\cite{chen2023post} 
{characterized Bayes-optimal classifiers as linear combinations of certain basis functions, without finding their explicit forms. We provide a detailed comparison with this work after Theorem \ref{thm-opt-dp}.}

See \Cref{table:compare} for a quick overview of comparison with related work.

Despite this progress, there is no systematic
approach to derive explicit forms of Bayes-optimal fair classifiers for general disparity measures and arbitrary disparity levels.
There is also no unified algorithmic framework to learn these classifiers from data.

In this paper, 
we propose a unified theoretical framework for deriving Bayes-optimal classifiers under various fairness measures, by leveraging a novel connection with the Neyman-Pearson argument for optimal hypothesis testing.  
We define a new notion of \emph{linear disparity measures}, which are linear functions of a probabilistic classifier (See Definition \ref{ldm}), 
and show that
finding the fair Bayes optimal classifier can then be recast as maximizing a linear functional subject to linear constraints.  
Since this is fundamentally similar to maximizing power subject to level constraints in optimal hypothesis testing,
we can derive the explicit form of a fair Bayes-optimal classifier with any given level of disparity, for any given linear disparity measure.

{While prior work has explored linear or bilinear representations of disparity, our framework differs in its precise definition of linearity, the choice of objective (e.g., misclassification error vs. other losses), and the classification setting (e.g., binary vs. multiclass). \cite{agarwal2018reductions} use a weaker notion of linearity that does not require linearity in the classifier itself; 
our results can thus be more explicit in terms of establishing the precise forms of the classifiers.
\cite{celis2019classification} define a more general class of fractional linear constraints.
In comparison, 
we highlight the bilinear case in the linear constraint setting, which enables us to establish the explicit form of the Bayes-optimal classifiers for 
equality of opportunity and predictive parity at a user-specified approximate level.
\cite{Weietal20a} use similar disparity constraints but minimize cross-entropy rather than misclassification error;
and finally, \cite{alghamdi2022beyond} work in a multiclass setting with ratio-based disparity measures and optimize divergence from a base classifier rather than mis-classification error.}

We further show that our general theoretical framework can be  extended to a wide range of scenarios (see \Cref{sec:ex}):
handling multiple fairness constraints (such as equalized odds, ensuring both equality of opportunity and predictive equality); and
for the common scenario when the protected attribute cannot be used at the prediction phase.  

In addition to deriving fair Bayes-optimal classifiers, 
we also investigate the \emph{fair Pareto frontier}, which characterizes the optimal tradeoff between fairness and accuracy. 
This Pareto frontier comprises a set of fair Bayes-optimal classifiers. 
We study the tradeoff function that characterizes the best achievable misclassification rate for a given disparity level. We prove that for linear disparity measures, 
the tradeoff function is convex, highlighting that the marginal cost of fairness increases for smaller disparity levels. {\cite{wang2023aleatoric} study a different yet similar fair Pareto frontier, and we comment on their work and its relation to ours in Section \ref{fro}.}

In practice, Bayes-optimal classifiers need to be estimated from a finite dataset.
For this, we develop methods 
covering three popular approaches to fairness-aware classification: pre-processing, in-processing, and post-processing \citep[e.g.,][]{caton2020fairness}. 

\begin{itemize}
    \item 
\emph{Pre-processing methods} reduce biases implicit in the training data, and train 
classifiers 
on the debiased data. Examples include transformations \citep[e.g.,][etc]{FFMS2015,KJ2016,JL2019,CFDV2017}, fair representation learning \citep[e.g.,][etc]{ZWKT2013,CKYMR2016,CMJW2019}, and fair generative models \citep[e.g.,][etc]{XYZ2018,SHC2019,Jang_Zheng_Wang_2021}. These methods are convenient to apply, as they do not change the training procedure. However, as argued for instance in \cite{NARBSB2019}, disparity could persist
even after pre-processing.

    \item 
\emph{In-processing algorithms} handle the fairness constraint during the training process. 
A commonly applied strategy is to incorporate fairness measures as a regularization term into the optimization objective \citep[e.g.,][etc]{GCAM2016,ZBR2017,NH2018,CLKVN2019,CJG2019,OnetoGeneral2020,CHS2020}. 
However, fairness measures are often non-convex and even non-differentiable with respect to model parameters, so that scalable training is challenging.
In the alternative approach of adversarial learning \citep[e.g.,][etc]{ZLM2018,CFC2018,XWYZW2019,LV2019}, 
the ability of the classifier to predict
the protected attribute is minimized. 
Although this can achieve promising results through careful design, its training process 
can lack stability, 
as a min-max optimization problem is required to be solved \citep{CHS2020}. 
Other in-processing algorithms include 
domain-based training \citep{WQK2020}, 
where the protected attribute is explicitly encoded and its effect is mitigated. 

    \item 
\emph{Post-processing methods} fit 
classifiers to the training data in the standard way, 
and aim to mitigate disparities in the model output. 
The most frequently used post-processing algorithms include group-wise thresholding \citep[e.g.,][etc]{BJA2016,CSFG2017,VSG2018,MW2018,CCH02019,alabdulmohsin2020fair,SC2021,denis2021fairness,JSW2022}  and post-processing via optimal transport \citep[e.g.,][etc]{JiangWasserstein2020,Silviageneral2020,xian2023fair,XuFair2023}.
 In group-wise thresholding, after estimating the conditional probability of $Y=1$ for each protected group (i.e., the group-wise regression functions), the method classifies $\widehat Y=1$ when these probabilities exceed certain group-specific thresholds. On the other hand, the optimal transport-based method transforms these conditional probabilities into a fair  Bayes-optimal classifier by employing optimal transport algorithms.


\end{itemize}

Motivated by our formulas for 
fair Bayes-optimal classifiers, we design three algorithms
that  mitigate 
algorithmic bias, 
via pre-processing (Fair Up- and Down-Sampling or FUDS), 
in-processing
(Fair cost-sensitive Classification or FCSC)
and 
post-processing
(Fair Plug-In Rule or FPIR).

For the pre-processing algorithm  FUDS, we characterize the perturbed data distribution for which the unconstrained Bayes-optimal classifier 
equals the fair Bayes-optimal classifier on the original target distribution. 
We prove that this perturbation can be obtained by adjusting the proportion of each demographic group defined by the label and protected attribute. 
This enables us to design an easy-to-use 
Fair Up- and Down-Sampling pre-processing algorithm.

For the in-processing algorithm FCSC, 
we observe that the fair Bayes-optimal classifier adjusts thresholds for each protected group to mitigate disparities. 
Leveraging that cost-sensitive classification achieves a similar effect, 
we show that the fair Bayes-optimal classifier can be achieved through a carefully designed group-wise cost-sensitive risk, leading to
Fair cost-sensitive Classification (FCSC).

Finally, for the post-processing algorithm FPIR, 
since we derived the explicit form of fair Bayes-optimal classifiers,
we can design a two-stage plug-in rule to estimate it.
In the first stage, we estimate the feature-conditional probability of a positive outcome 
$Y=1$ for each protected group. In the second stage, the optimal fair decision boundary can be estimated by solving a one-dimensional search problem.


We summarize our contributions below.
We also provide a detailed comparison 
of our contributions with prior theoretical work for fair Bayes-optimal classifiers in Table \ref{table:compare}.
\begin{enumerate}
    \item \textbf{Unifying framework for fair Bayes-optimal classifiers with linear disparity measures:}
    We provide a unified framework for deriving Bayes-optimal classifiers under group fairness constraints. 
    We introduce the notions of linear and bilinear disparity measures (Definitions \ref{ldm} and \ref{bldm}) and
    show that several popular 
    disparity measures---the deviations from demographic parity, equality of opportunity, predictive equality, and equalized odds---are bilinear (Proposition \ref{prop:expression of dl}).
    We characterize Bayes-optimal classifiers under linear disparity measures (Theorem \ref{thm-opt-dp}), by uncovering a connection with the Neyman-Pearson lemma. 
    For bilinear disparity measures,
    {we are able to find the explicit form of Bayes-optimal fair classifiers as group-wise thresholding rules with explicitly characterized thresholds. In particular, this explicit form appears to be novel when a nonzero user-specified disparity level is specified.}
    
    \item \textbf{Extensions:} 
    We illustrate that our approach can be extended to handle multiple fairness constraints (such as equalized odds (\Cref{thm-fb-eq-odd}), which requires ensuring both equality of opportunity and predictive equality), and demographic parity with a multi-class protected attribute (\Cref{thm-opt-dp-multi}).
    The result for equalized odds requires an intricate and lengthy argument, which unravels several fundamental properties of equality of opportunity and predictive equality.
    We  also show how to handle cost-sensitive classification error (Corollary \ref{cor-opt-cost-sen}).
    Further, we derive Bayes-optimal fair classifiers  for the common scenario when the protected attribute cannot be used at the prediction phase (Corollary \ref{cor-opt-dp}).
    
Finally, we give the form of Bayes-optimal fair classifiers 
under the general distributional assumptions that
the features can belong to the decision boundary with nonzero probability; such as for discrete features. 
In this case, the optimal classifiers must be carefully randomized (Theorem \ref{thm-opt-dp-degenerate}).

    \item \textbf{Pareto frontier for fair classification and fairness-accuracy tradeoff:}
We characterize the fair Pareto frontier that represents the optimal trade-off between accuracy and fairness. 
We show that the tradeoff function between misclassification and disparity is convex for linear disparity measures, indicating that the marginal cost of fairness increases as the level of disparity decreases.
    \item \textbf{Theoretically optimal fair classification via pre-, in-, and post-processing:}
We propose pre-, in-, and post-processing algorithms aiming to recover the Bayes-optimal classifiers. 
Specifically we introduce 
fair up-/down-sampling as a pre-processing method (\Cref{fuds}), 
fair cost-sensitive classification as an in-processing method (\Cref{sec:fcsc}), 
and a fair plug-in rule as a post-processing method (\Cref{sec:fpir}). 
Collectively, these methods establish a cohesive methodological framework for fair classification with linear disparity constraints.
    \item \textbf{Empirical evaluation:}
    We evaluate our methods in numerical simulations and experiments on empirical datasets (\Cref{sim}).
{We compare with several existing methods, such as
the disparate impact remover \citep{FFMS2015},
 FAWOS \citep{salazar2021fawos}, adaptive reweighting algorithm \citep{chai2022fairness},
KDE-based constrained optimization \citep{CHS2020},
adversarial training  \citep{ZLM2018}, reductions \citep{agarwal2018reductions}, MinDiff regularization \citep{prost2019toward},
post-processing through flipping \citep{chen2023post}, post-processing through optimal transport \citep{xian2023fair}, and FRAPPÉ \citep{tifrea2023frappe};
on standard and more recent benchmark datasets}.
{We show empirically  that our methods have state-of-the-art performance compared to existing algorithms.
In particular, our pre-processing method can a reach higher accuracy than prior pre-processing methods at low disparity levels.}
Our numerical results can be reproduced with the code provided at \url{https://github.com/XianliZeng/Bayes-Optimal-Fair-Classification}.

\end{enumerate}

This paper significantly extends the results from our previous unpublished manuscript
\citep{zeng2022bayes}, 
introducing the notion of linear disparity measures, handling multiple constraints, and developing pre- and in-processing algorithms. 
Thus, this work supersedes \cite{zeng2022bayes}.

\textbf{Paper organization.} In Section \ref{sec:pre}, we review background concepts and notations for fair classification. 
Section \ref{gnpfc} establishes the connection between the generalized Neyman-Pearson Lemma and  fair Bayes-optimal classifiers. Our main results are presented in Section \ref{fpfsec}, providing explicit forms of fair Bayes-optimal classifiers under linear and bilinear group fairness measures.
We characterize the fair Pareto frontier in \Cref{fro}. 
In Section \ref{sec:alg}, we 
propose methods for estimating fair Bayes-optimal classifiers through pre-, in- and post-processing methods. 
Simulation studies and empirical data analysis in Section \ref{sim} assess the finite sample performance of the proposed methods. 
We provide concluding remarks in Section \ref{sec:dis}. 
All proofs are in the supplementary material. 

\section{Classification with a Protected Attribute}\label{Pre_and _Not}
In fair classification problems, two types of feature are observed: the usual feature $X\in\X$ for some feature space $\X$, and the protected (or, protected) feature\footnote{We consider a binary protected attribute  here and extend our results to multi-class protected attributes in \Cref{sec:ex}.} 
$A\in\A=\{0,1\}$
with respect to which we aim to be fair. 
Here, we consider a binary classification problem with labels in $\mathcal{Y}=\{0,1\}$. 
For example, in a credit lending setting, 
$X$ may refer to common features such as 
 education level and income, $A$ may contain the race or gender of the individual and $Y$ may correspond to the status of repayment or defaulting on a loan.

{\bf Notation and Conventions.} 
For all $a\in \A$,  $x\in \mX$ and $y\in\mathcal{Y}$,
we denote $p_{a}:=\P(A=a)$; $p_{a,y}:=\P(A=a,Y=y)$; $\eta_{a}(x):=\P(Y=1\mid A=a,X=x)$. 
Further, for all $a\in \A$ and $y\in\mathcal{Y}$, we denote by $\P_X(\cdot)$, $\P_{X\mid A=a}(\cdot)$ and $\P_{X\mid A=a,Y=y}(\cdot)$ the marginal distribution function of $X$, the conditional distribution function of $X$ given $A=a$, and the conditional distribution of $X$ given $A=a,Y=y$, respectively. 

For two scalars $a,b$, we denote
their maximum by 
$\max\{a,b\}$ or
$a\vee b$, 
and 
their minimum by 
$\min\{a,b\}$ or
$a\wedge b$.
All quantities considered will be measurable with respect to appropriate sigma-algebras; and measurability may not always be mentioned in what follows.

\subsection{Preliminaries}\label{sec:pre}

A \emph{randomized classifier} outputs a prediction $\widehat{Y}\in\{0,1\}$ with a certain probability based on the usual features $X$ and the protected attribute $A$.
Let $\Bern(p)$ be the Bernoulli distribution with success probability $p \in [0,1]$, 
and let $\mF$ be the set of measurable functions
  $f:\X\times\mA \to [0,1]$.
    We denote by $\widehat{Y}_f=\widehat{Y}_f(x,a) \in \{0,1\}$ the prediction induced by the classifier $f$, which can be a random variable.
\begin{definition}[Randomized Classifier]
  A randomized classifier
  $f \in \mF$ gives, for any $x\in \mX$ and $a\in\mA$, the probability $f(x,a)$
  of predicting $\widehat{Y}_f=1$ when observing $X=x$ and $A=a$, i.e., 
  $\widehat{Y}_f\mid X=x,A=a\sim \Bern(f(x,a))$. 
 \end{definition}

Without a fairness constraint, a \emph{Bayes-optimal classifier} minimizes the misclassification rate, 
and is defined as
any randomized classifier $f^\star$
satisfying
$f^\star\in\underset{f\in\mF}{\text{argmin}}\, \P(Y\neq \widehat{Y}_f)$. 
The following classical result characterizes Bayes-optimal classifiers in terms of the class-conditional probability functions $\eta_a$ for which\footnote{We will sometimes write $\eta^Y_{A=a}:=\eta_a$ for $a\in \mathcal{A}$.}
$\eta_{a}(x)=\P(Y=1\mid A=a,X=x)$ for all $x,a$ 
\citep[see e.g.,][etc]{DevroyeGL96}.
We denote by $I(\cdot)$ the indicator function, which equals unity if its argument is true, and zero otherwise.
\begin{proposition}\label{prop:ba-op} 
All Bayes-optimal classifiers $f^\star\in\mF$ have the form
$f^\star(x,a)=I\left(\eta_a(x)>1/2\right)+\tau(x,a)I\left(\eta_a(x)=1/2\right),$
for all $(x,a)\in \X\times \{0,1\}$, where $\tau: \X\times\A\to[0,1]$ is any measurable function. 
\end{proposition}

While Bayes-optimal classifiers are theoretically the best method, 
they depend on the usually unknown
the class-conditional probability functions $\eta_a$
of the population. 
However, Proposition \ref{prop:ba-op}
is still useful, because it 
suggests an approach to classification, by estimating the functions $\eta_a$, for all $a$, and defining the classifier in terms of the level sets of the estimates, see e.g., \cite{tsy2007}.

The Bayes-optimal classifier does not take fairness into account.
To mitigate unfairness, a number of notions of parity have been considered, and we list
below the ones we consider in this paper. 

  \begin{definition}[Demographic Parity \citep{CKP2009}]
  A classifier $f$ satisfies demographic parity if its prediction $\widehat{Y}_f$
   is probabilistically independent of the protected attribute $A$, so that
   $\P_{X\mid A=1}\lsb \widehat{Y}_f  = 1\rsb  =\P_{X\mid A=0}\lsb \widehat{Y}_f  = 1\rsb .$
  \end{definition}

  \begin{definition}[Equality of Opportunity \citep{HPS2016}]
  A classifier $f$ satisfies equality of opportunity if it achieves the same true positive rate  among  protected groups:
   $\P_{X\mid A=1,Y=1}\lsb\widehat{Y}_f  = 1 \rsb  =\P_{X\mid A=0,Y=1}\lsb\widehat{Y}_f  = 1 \rsb .$
  \end{definition}

  \begin{definition}[Predictive Equality \citep{CSFG2017}] 
  A classifier $f$ satisfies predictive equality   if it achieves the same false positive rate among protected groups:
   $\P_{X\mid A=1,Y=0}\lsb\widehat{Y}_f  = 1 \rsb  =\P_{X\mid A=0,Y=0}\lsb\widehat{Y}_f  = 1 \rsb .$
  \end{definition}

  \begin{definition}[Equalized Odds \citep{HPS2016}]\label{eod}
  A classifier $f$ satisfies equalized odds  if it satisfies  both equality of opportunity and predictive equality, achieving the same true positive rate and false positive rate among protected groups:
   $\P_{X\mid A=1,Y=1}\lsb\widehat{Y}_f  = 1 \rsb  =\P_{X\mid A=0,Y=1}\lsb\widehat{Y}_f  = 1 \rsb,$ and    $\P_{X\mid A=1,Y=0}\lsb\widehat{Y}_f  = 1 \rsb  =\P_{X\mid A=0,Y=0}\lsb\widehat{Y}_f  = 1 \rsb .$

  \end{definition}
    
  
  In applications, perfect fairness may require a large sacrifice of accuracy, 
 and a ``limited'' disparate impact could be preferred. 
 Following for instance \cite{CHS2020}, 
 we use the difference in the quantities
 that are equalized under perfect parity
 to measure disparity: 

\begin{definition}[Disparity Measures]
    We consider the following disparity measures:
    Demographic Disparity $\textup{(DD)}$,
    Disparity of Opportunity $\textup{(DO)}$, 
Predictive Disparity 
 $\textup{(PD)}$,
defined for probabilistic classifiers $f$ as follows:
\begin{align}\label{eq:disparity level}
      \textup{DD}(f)&= \P_{X\mid A = 1}(\widehat{Y}_f = 1) -\P_{X\mid A = 0}\lsb \widehat{Y}_f  = 1\rsb ;\nonumber\\
  \textup{DO}(f)&= \P_{X\mid A = 1,Y = 1}(\widehat{Y}_f = 1) -\P_{X\mid A = 0,Y=1}\lsb \widehat{Y}_f  = 1\rsb ;\\
   \textup{PD}(f)&= \P_{X\mid A = 1,Y=0}\lsb \widehat{Y}_f  = 1\rsb  -\P_{X\mid A = 0,Y=0}\lsb \widehat{Y}_f  = 1\rsb .\nonumber 
\end{align}
\end{definition} 
Since equalized odds consists of two constraints, we will later use the maximum of the absolute values of the corresponding disparities as the measure of unfairness.

 In the rest of this paper, we will use $\textup{Dis}:\mF\to [0,1]$ to refer to a 
 generic disparity measure, 
 such as any of the three measures 
 from \eqref{eq:disparity level}.  Taking $K\ge1$ disparity measures $\textup{Dis}_{k}:\mF\to[0,1], k=1,...,K$ into account, we say a classifier $f$ satisfies $\delta$-disparity if  $\max_{k=1}^K|\textup{Dis}_k(f)|\le \delta$.
We next define the most accurate---equivalently, least inaccurate---classifiers that satisfy $\delta$-disparity.
\begin{definition}[Fair Bayes-optimal Classifier]\label{fbo}
Consider 
 any $K\ge 1$ 
fairness measures $\textup{Dis}_k:\mF\to [0,1]$, $k=1,...,K$. 
Then, a $\delta$-fair Bayes-optimal classifier $f_{\textup{Dis},\delta}^\star$ is defined 
by minimizing the misclassification error 
$R(f) :=  \P\lsb Y\neq \widehat{Y}_f\rsb$
over   all classifiers that satisfy  $\delta$-disparity: 
$$f_{\textup{Dis},\delta}^\star\in {\argmin}_{f\in\mF} \lbb R(f): \max_{k=1}^K|\textup{Dis}_k(f)|\le \delta\rbb.$$
 \end{definition}
 
Just like for the unconstrained case,
fair Bayes-optimal classifiers are theoretically the best method, but they do not directly lead to a feasible method. 
To be able to use classifiers inspired by fair Bayes-optimal ones, 
it would be helpful to know if there is a characterization similar to that from the unconstrained case in Proposition \ref{prop:ba-op}. We now turn to establishing this, by making a connection with the Neyman-Pearson lemma.

\section{Generalized Neyman-Pearson Lemma and Fair Classification}
\label{gnpfc}

\subsection{Generalized Neyman-Pearson Lemma }
To derive Bayes-optimal classifiers under group fairness, we establish a connection with the 
Neyman-Pearson lemma \citep{NP1933}; 
a theoretical result originally designed to derive most powerful tests with a given type I error.
We use a slightly generalized version, which applies to optimization with linear constraints. 
For completeness and the reader's convenience, we present the full statement here; as this lemma is fundamental to our results.

\begin{lemma}[Generalized Neyman-Pearson Lemma \citep{lehmann2005testing, Shao2003}]\label{NP_lemma}
Let $\phi_0,\phi_1, ..., \phi_{m}$ be $m+1$ real-valued functions defined on a Euclidean space $\X$. Assume they are $\nu$-integrable for a $\sigma$-finite measure $\nu$.
 Let $f^\star \in \mF$ be any function of the form
\begin{align}
   \begin{split}
  f^\star(x)=\left\{\begin{array}{lcc}
   1,&& \phi_0(x)>\sum_{i=1}^mc_i\phi_i(x);\\
  \tau(x)&& \phi_0(x)=\sum_{i=1}^mc_i\phi_i(x);\\
   0,&& \phi_0(x)<\sum_{i=1}^mc_i\phi_i(x),
   \end{array}\right.
 \end{split}
 \end{align}
 where $0\le\tau(x)\le 1$ for all $x\in \X$.
 For given constants $t_1, ...,t_m \in \mathbb{R}$, 
 let $\mathcal{F}_{\le}$ be the class of measurable functions $f: \X\to \mathbb{R}$  satisfying
 \begin{equation}\label{con1}
 \int_{\X} f\phi_i d\nu\le t_i,\ i\in\{1,2,...,m\}.
 \end{equation}
 and $\mathcal{F}_{=}$ be the set of functions in $\mathcal{F}_{\le}$ satisfying \eqref{con1}  with all inequalities
 replaced by equalities. 
 \begin{itemize}[]
     \item (1) If $f^\star\in\mathcal{F}_{=}$, then
 \begin{equation}\label{NPLEM1}
f^\star \in \underset{f\in\mathcal{F}_{=}}{\argmax}\int_{\X}f\phi_0d\nu.
 \end{equation}
 Moreover, if $\nu(\{x:  \phi_0(x)=\sum_{i=1}^mc_i\phi_i(x)\})=0$,  
 for all $f'\in\underset{f\in\mathcal{F}_{=}}{\argmax}\int_{\X}f\phi_0d\nu$, $f'=f^\star$ almost everywhere with respect to $\nu$.
 \item  (2) Moreover, if $c_i\ge 0$ for all $i=1,\ldots,m   $, then
  \begin{equation}\label{NPLEM2}
f^\star \in  \underset{f\in\mathcal{F}_{\le}}{\argmax}\int_{\X} f\phi_0 d\nu. \end{equation}
 Moreover,  if $\nu(\{x:  \phi_0(x)=\sum_{i=1}^mc_i\phi_i(x)\})=0$,   for all $f'\in\underset{f\in\mathcal{F}_{\le}}{\argmax}\int_{\X}f\phi_0d\nu$, we have $f'(x)=f^\star(x)$ almost everywhere with respect to $\nu$.
 \end{itemize}

\end{lemma}

\subsection{Fairness Measures with Linear  and Bilinear Constraints}
To characterize Bayes-optimal fair classifiers,
we want to find classifiers with the highest accuracy given a disparity level. 
Recalling the class-conditional regression functions  such that for $a\in \A$ and  $x\in \mX$,
$\eta_{a}(x)=\P(Y=1\mid A=a,X=x)$,
the misclassification rate can be expressed as a linear functional of the classifier $f$ {(see Lemma \ref{lem:misclassification} for the argument)}, via
\begin{equation*}
 R(f)
=\int_{\A}\int_{\X}f(x,a) (1-2\eta_a(x))d\P_{X,A}(x,a)+ C_{\P}, 
\end{equation*}
with $C_\P=\int_{\A}\int_{\X} \eta_a(x)d\P_{X,A}(x,a)$. 
As $C_\P$ is independent of $f$, 
{maximizing accuracy over $f$ is equivalent to maximizing $\int_{\A}\int_{\X}f(x,a) (2\eta_a(x)-1)d\P_{X,A}(x,a)$}.
Therefore, 
a $\delta$-fair Bayes-optimal classifier for a disparity measure Dis  can be equivalently defined as
\begin{equation}\label{eq:exp acc} 
f_{\textup{Dis},\delta}^\star\in \underset{f\in\mF}{\argmax}\,\lbb \int_{\A}\int_{\X}f(x,a) (2\eta_a(x)-1)d\P_{X,A}(x,a):   |\textup{Dis}(f)|\le\delta\rbb.
\end{equation}
{Hence, the objective is linear in $f$.}
Specifically, taking $\nu=\P_{X,A}$
and 
$\phi_0(x,a)=2\eta_a(x)-1$ for all $x,a$
in the objective from \eqref{NPLEM2} in  the Neyman-Peason lemma,
we find that it reduces to the objective in \eqref{eq:exp acc}, 
where the constraint set $\mF_{\le}$ from \eqref{NP_lemma} remains to be specified. 

As a result, the  fair Bayes-optimal classifiers
can be characterized by the generalized Neyman-Pearson lemma, \emph{if the constraints are linear in the classifiers}. 
Motivated by this observation, 
we introduce the notion of 
\emph{linear disparity measures}.


\begin{definition}[Linear Disparity Measure]\label{ldm}
We call a disparity measure
$\textup{Dis}:\mF\to [0,1]$ \emph{linear} if
 for all $\P$,
there is a weighting function $ w_{\textup{Dis},\P}:\X\times \A\to \R$  such that for all $f\in \mF$,
\begin{equation}\label{eq:exp dis}
   \textup{Dis}(f)= 
   \int_{\A}\int_{\X}f(x,a)   w_{\textup{Dis},\P}(x,a) d\P_{X,A}(x,a).
\end{equation}
\end{definition}

Furthermore, we call a linear disparity measure \emph{bilinear} 
if its weighting function is linear in 
the group-wise regression functions
$\eta_a$:

\begin{definition}[Bilinear Disparity Measure]\label{bldm}
A linear disparity measure
$\textup{Dis}:\mF\to [0,1]$ is  \emph{bilinear} if
 for all $\P$,
there is $s_{\textup{Dis},\P,a}$ and $b_{\textup{Dis},\P,a}$ depending on $a\in\mA$  such that  for all $x\in \mX$, 
$ w_{\textup{Dis},\P}(x,a)= s_{\textup{Dis},\P,a} \eta_a(x) +b_{\textup{Dis},\P,a}$.
\end{definition}

Hereafter, we will omit the subscript $\P$ in our notation for simplicity.
We next show that the disparity measures from \eqref{eq:disparity level} are bilinear, with specific weighting functions.
\begin{proposition}[Classical Disparity Measures are Bilinear]\label{prop:expression of dl}
    The disparity measures  $\textup{DD}$, $\textup{DO}$, and $\textup{PD}$ from \eqref{eq:disparity level} are bilinear with weighting functions defined for all $x,a$ by
\begin{equation}\label{eq:disparity level expression}
   w_\textup{DD}(x,a)= \frac{(2a-1)}{p_{a}}; \quad    w_\textup{DO}(x,a)=\frac{(2a-1)\eta_a(x)}{p_{a,1}}; \quad 
   w_\textup{PD}(x,a)= \frac{(2a-1)(1-\eta_a(x))}{p_{a,0}}
\end{equation}
where $p_a = \P(A=a)$ and $p_{a,y} = \P(A=a,Y=y)$ for $a\in\{0,1\}$ and $y\in\{0,1\}$. 
Thus, we have, for instance
$s_{\textup{DD},a} = 0$ and $b_{\textup{DD},a}= (2a-1)/{p_{a}}$ for all $a$;
 the quantities $s_{\textup{Dis},a}$ and $b_{\textup{Dis},\P,a}$ for the other disparity measures can be similarly read off from \eqref{eq:disparity level expression}.
\end{proposition}
Some of our results apply to linear disparity measures, and some to bilinear ones; thus we use both definitions.
We will use \eqref{eq:exp dis} and Proposition \ref{prop:expression of dl} 
to develop a unified framework for deriving fair Bayes-optimal classifiers. 
For ease of exposition, we discuss a binary protected attribute and
a single fairness constraint in the main text. 
However, 
our theory can also handle other scenarios,
such as multiple fairness constraints and in particular a multi-class protected attribute.
We briefly discuss these extensions in \Cref{sec:ex}, and provide details in \Cref{ext}.

\section{Fair Bayes-Optimal Classifiers and the Fair Pareto Frontier}
\label{fpfsec}
\label{sec:Bayes}
\subsection{Form of Fair Bayes-Optimal Classifiers}
In this section, we derive the explicit form of fair Bayes-optimal classifiers based on our theoretical framework. 
For simplicity, we present the results for a single linear disparity measure as per Definition \ref{ldm}, and discuss cases with multiple constraints in \Cref{ext}. 
Furthermore, 
in the main text, we consider the setting where
both random variables $\eta_a(X)$ and $ w_{\textup{Dis}}(X,a)$, for $a\in\{0,1\}$, have probability density functions over  
$\X$. 
This implies that the boundary sets of randomized classifiers are of measure zero, allowing us to only consider deterministic classifiers. 
We provide the---more involved---results applicable to the general case without this assumption in \Cref{gd}. 

{To build intuition for the explicit form of the $\delta$-fair Bayes optimal classifier, consider a linear disparity measure \( \text{Dis} \). 
For expositional simplicity, consider finding the Bayes-optimal classifier with respect to the one-sided constraint $\text{Dis}\le \delta$; while later we will work with the two-sided constraint.
Let \( \phi_0(x,a) = 2\eta_a(x) - 1 \) and \( \phi_1(x,a) = w_{\text{Dis}}(x,a) \) for all \( x \) and \( a \) in equations ~\eqref{con1} and ~\eqref{NPLEM1}, respectively specify the accuracy and (one-sided) disparity constraint. 
The generalized Neyman-Pearson lemma, together with the implicit characterization of the \( \delta \)-fair Bayes-optimal classifier in~\eqref{eq:exp acc} and the form of the linear disparity constraints in~\eqref{eq:exp dis}, suggest that a fair Bayes-optimal classifier takes the deterministic form:
\begin{equation}\label{eq:partos}
    f_{\textup{Dis},t}(x,a) = I\lsb\eta_a(x) > \frac{1}{2} + \frac{t}{2} \, w_{\textup{Dis}}(x,a)\rsb,
\end{equation}
for all \( x \in \mathcal{X} \), \( a \in \{0,1\} \), and some \( t \in \mathbb{R} \).}

To make this derivation precise, it will be helpful
to define a disparity function ${D}_{\textup{Dis}}:\R\to [-1,1]$ measuring the disparity level of $f_{\textup{Dis},t}$ as a function of $t$, such that for all $t\in \R$:
\begin{align}\label{eq:dt}
\nonumber{D}_{\textup{Dis}}(t) :=& \textup{Dis}(f_{\textup{Dis},t}) = \int_{\A} \int_{\X} \lmb  w_{\textup{Dis}}(x,a) \cdot I\lsb
\eta_a(x)>\frac12 + \frac{t}2   w_{\textup{Dis}}(x,a)\rsb\rmb d\P_{X,A}(x,a)\\
=& \sum_{a\in\{0,1\}}p_a\int_{\X} \lmb  w_{\textup{Dis}}(x,a) \cdot I\lsb
\eta_a(x)>\frac12 + \frac{t}2   w_{\textup{Dis}}(x,a)\rsb\rmb d\Pa(x).
\end{align}
The following proposition characterizes how the accuracy and disparity of $f_{\textup{Dis},t}$ depends on $t$, and will be a crucial stepping stone towards precisely characterizing the form of fair Bayes-optimal classifiers.

\begin{proposition}[Properties of Risk and Disparity]\label{prop:monotonicity}
Let $f_{\textup{Dis},t}$ and $D_{\textup{Dis}}$ be defined in \eqref{eq:partos} and \eqref{eq:dt}, respectively. Then,  as a function of $t$, 
\begin{itemize}[]
    \item (1) the disparity $\Dt$ is monotone non-increasing;
    \item (2) the misclassification error $R(f_{\textup{Dis},t})$ is monotone non-increasing on $(-\infty, 0)$ and monotone non-decreasing on $[0,\infty)$. 
  \end{itemize}
 \end{proposition}
Now, consider finding a classifier $f_{\textup{Dis},t}$ that meets fairness constraints while minimizing 
the misclassification error.
By Proposition \ref{prop:monotonicity}, 
it is enough to minimize $|t|$; indeed this can be seen considering the cases of $|{D}_{\textup{Dis}}(0)|\le \delta$ (in which case the optimal $t=0$), ${D}_{\textup{Dis}}(0)<- \delta$ (in which case the optimal $t$ is negative), and ${D}_{\textup{Dis}}(0)> \delta$ (in which case the optimal $t$ is positive). 
This motivates us
 to define the function $t_{\mathrm{Dis}}:[0,\infty)\to \R$ as an ``inverse function" of $|{D}_{\textup{Dis}}(t)|$ such that for all $\delta\ge 0$,
 \begin{equation}\label{eq:td}
\td = \argmin_t \{|t|: |{D}_{\textup{Dis}}(t)|\le \delta\}.
\end{equation}
With these preparations, we can provide 
the form of fair Bayes-optimal classifiers for linear disparity measures.

\begin{theorem}[Form of Fair Bayes-optimal Classifiers for Linear Disparity Measures]\label{thm-opt-dp}
For a given linear disparity measure $\textup{Dis}$ as per Definition \ref{ldm},
suppose that for $a\in\{0,1\}$, both $\eta_a(X)$ and $ w_{\mathrm{Dis}}(X,a)$ have probability density functions on $\mathcal{X}$.
Recalling $f_{\textup{Dis},t}$ from \eqref{eq:partos} and $t_{\mathrm{Dis}}:[0,\infty)\to \R$ from \eqref{eq:td},
for any $\delta\ge 0$,
$f_{\mathrm{Dis}, \td}$ \emph{is a $\delta$-fair Bayes-optimal classifier} as per Definition \ref{fbo}.
This Bayes-optimal classifier $f^\star_{\textup{Dis},\delta}$ of the form
\begin{align}\label{eq:fblinear}
\begin{split}
&f^\star_{\textup{Dis},\delta}(x,a) :=
f_{\textup{Dis},\td}(x,a)=
I\left(\eta_{a}(x)> \frac12 +\frac{t_{\textup{Dis}}(\delta)}2  w_{\textup{Dis}}(x,a)\right),
\end{split}
\end{align}
for all $x,a$.
Furthermore, when the disparity measure is bilinear as per Definition \ref{bldm},
the above $\delta$-fair Bayes-optimal classifier simplifies to 
a  group-wise thresholding rule, such that for all $x,a$,  
\begin{equation}\label{eq:fbbilinear}
f^\star_{\textup{Dis},\delta}(x,a) = I\lsb\eta_a(x)>\frac{1 +  b_{\textup{Dis},a}\cdot\td }{2- s_{\textup{Dis},a} \cdot\td}\rsb. 
\end{equation}
\end{theorem}

{\cite{chen2023post}
similarly characterize---and develop practical algorithms for---finding $\delta$-fair Bayes optimal classifiers under 
a specific form of composite constraints, which refer to differences of probabilities of a flipped version of  
$\widehat Y$ given specific pairs of values of the protected attribute (which are a bit more restricted than our linear disparity measures).
They characterize fair Bayes optimal classifiers 
in a different but equivalent way, 
using certain \textit{instance-level bias scores} $s_k(x) = f_k(x) / \eta_a(x)$, 
where $f_k$ are weighting functions involving protected class attribute membership probabilities, related to our $w_{\textup{Dis}}$.
They show that the optimal fair classifiers take the form
$\kappa(x) = 1\left(\sum_k z_k s_k(x) > 1\right)$, for some unspecified scalars $z_k$.}

{In contrast, our result above concerns only one constraint, but it provides the explicit form of the optimal classifier, without unspecified scalars.
Moreover, our monotonicity results enable developing efficient binary-search based algorithms, while 
 \cite{chen2023post} are limited to essentially trying all possible values of the hyperparameter settings, which can be computationally expensive.
Our results also inspire pre- and in-processing algorithms, which can not be as easily gleaned from their characterization. 
Specifically,  changing the threshold $t$ equivalently leads to an optimal classifier for a certain cost-sensitive risk; and minimizing a cost-sensitive risk is in turn equivalent to minimizing a re-sampled/re-weighted dataset according to our label shift up-/down-sampling mechanism.
}

Similar to the unconstrained case from Proposition \ref{prop:ba-op}, 
    {we are able to find the explicit form of Bayes-optimal fair classifiers as group-wise thresholding rules with explicitly characterized thresholds.}
Clearly, the accuracy is maximized by predicting the more likely class in each group. 
Moreover, mitigating disparity over protected groups necessitates shifting the thresholds. 
This shift depends delicately on the disparity measure and the population distribution.
By incorporating the expressions from Proposition \ref{prop:expression of dl}, we conclude the following corollary for the common disparity measures.
{In particular, this explicit form appears to be novel when a nonzero user-specified disparity level is given.}

\begin{corollary}[Bayes-optimal Classifiers for Three Common Disparity Measures]\label{cor:fbbinilir}
    For any $\delta>0$, the group-wise thresholding rules such that for all $x,a$,
    \begin{align*}
&f^{\star}_{\textup{DD},\delta}(x,a) = I\lsb \eta_a(x)>\frac{1}{2}+\frac{(2a-1)t_{\textup{DD}}(\delta)}{p_{a}}\rsb;\,\,\,
f^{\star}_{\textup{DO},\delta}(x,a) = I\lsb \eta_a(x)>\frac{p_{a,1}}{2p_{a,1} -(2a-1)t_{\textup{DO}}(\delta)}\rsb;\\
&\qquad\qquad\qquad f^{\star}_{\textup{PD},\delta}(x,a) = I\lsb \eta_a(x)>\frac{p_{a,0} +(2a-1)t_{\textup{PD}}(\delta)}{2p_{a,0} +(2a-1)t_{\textup{PD}}(\delta)}\rsb,
    \end{align*}
    with
        \begin{align*}
&t_{\textup{DD}}(\delta) = \argmin_t\lbb |t|: \lab\sum_{a\in\{0,1\}}\int_X I\lsb\eta_a(x)>\frac12 +\frac{(2a-1)t}{2p_a}\rsb d\Pa(x)\rab\le\delta\rbb;\\
&t_{\textup{DO}}(\delta) = \argmin_t\lbb |t|: \lab\sum_{a\in\{0,1\}}\int_X I\lsb\eta_a(x)>\frac{p_{a,1}}{2p_{a,1}-(2a-1)t}\rsb d\P_{A=a,Y=1}(x)\rab\le\delta\rbb;\\
&t_{\textup{PD}}(\delta) = \argmin_t\lbb |t|: \lab\sum_{a\in\{0,1\}}\int_X I\lsb\eta_a(x)>\frac{p_{a,0}+(2a-1)t}{2p_{a,0}+(2a-1)t}\rsb d\P_{A=a,Y=0}(x)\rab\le\delta\rbb,
    \end{align*}
   are $\delta$-fair Bayes-optimal classifiers   under 
   $\textup{DD}$, $\textup{DO}$ and $\textup{PD}$  from \eqref{eq:disparity level}, respectively. 

\end{corollary}

\subsection{Fairness-Accuracy Tradeoff: the Fair Pareto Frontier}
\label{fro}
Theorem \ref{thm-opt-dp} specifies fair Bayes-optimal classifiers for a given disparity level $\delta$. 
A core challenge is to set a suitable value for $\delta$, 
balancing accuracy and fairness. 
Identifying a suitable balance among multiple objectives can be viewed 
through the Pareto frontier. 
In the context of fair classification with a disparity measure $\textup{Dis}$, 
we consider the following ordering relation 
on classifiers. A classifier $f_1$ is \emph{Pareto dominant} over another classifier $f_2$, if one of the following conditions is met:
$$(1) \quad R(f_1)<R(f_2) \text{ and } \textup{Dis}(f_1)\le \textup{Dis}(f_2), \qquad   \text{ or } \qquad (2) \quad R(f_1)\le R(f_2) \text{ and } \textup{Dis}(f_1)< \textup{Dis}(f_2) .$$
A classifier $f$ is considered \emph{Pareto optimal} if it is not dominated by any other classifier. The collection of all such Pareto optimal classifiers constitutes the \emph{fair Pareto frontier} ($\textup{FPF}$). 
Formally, the FPF is the solution of the following vector objective optimization problem, with respect to the ordering defined above: 
\begin{equation}\label{eq:vec-opt}
    \argmin_{f\in\mF}\, \left(R(f),\textup{Dis}(f)\right).
\end{equation}

The following proposition characterizes the fair Pareto frontier:
\begin{proposition}[Fair Pareto Frontier]\label{prop:FPF}
 Let $t_{\textup{Dis}}$ be defined  in  \eqref{eq:dt}, and let
 $\underline{t}_0 =\min (0, t_{\textup{Dis}} (0))$ and  $\overline{t}_0 =\max (0, t_{\textup{Dis}} (0))$.
Then the fair Pareto frontier $\textup{FPF}$ from \eqref{eq:vec-opt}
includes all classifiers $f_{\textup{Dis},t}$
for $t \in [\underline{t}_0 ,\overline{t}_0]$, i.e.,
we have 
 $$\{ f_{\textup{Dis},t}: t \in [\underline{t}_0 ,\overline{t}_0]\} \subset \textup{FPF}.$$
Moreover, for any $f_{\textup{FPF}}\in \textup{FPF}$, there is $t\in  [\underline{t}_0 ,\overline{t}_0]$ such that $f_{\textup{Dis},t}$ has the same classification error and disparity as $f_{\textup{FPF}}$, i.e.,
  $$R(f_{\textup{Dis},t})  = R(f_\textup{FPF})  \ \ \ \text{ and }  \ \ \ \textup{Dis}(f_{\textup{Dis},t})=\textup{Dis}(f_{\textup{FPF}}).$$
\end{proposition}

Now, consider an equivalence relation ``$\sim$'' between classifiers $f_1,f_2 \in \mF$ such that $f_1\sim f_2$ if and only if $R(f_1)=R(f_2)$ and  $\textup{Dis}(f_1)=\textup{Dis}(f_2).$ 
Then, for each equivalence class determined by ``$\sim$'',  there is a minimizer of \eqref{eq:vec-opt} over $\{ f_{\textup{Dis},t}: t \in [\underline{t}_0 ,\overline{t}_0]\}.$
Hence, the classifiers $f_{\textup{Dis},t}$ fully characterize the Pareto frontier.

We next study the cost of fairness
as a function of the disparity level. 
We define the following trade-off function that measures the best achievable performance for a given disparity level.
\begin{definition}[Tradeoff Function]
 For a disparity measure $\textup{Dis}$ and any $\delta\ge 0$, let
\begin{equation}\label{eq:Td}
    T(\delta) = \inf\{ R(f): \textup{Dis}(f)\le \delta\}.
\end{equation}
\end{definition}
From  Proposition \ref{prop:FPF}, it follows that for all $\delta\ge0$, $T(\delta)=R(f_{\mathrm{Dis}, \td})$.
The tradeoff function plays a crucial role in characterizing the misclassification rate and disparity along the Pareto frontier. 
This function is monotone non-increasing by definition, 
as a more stringent fairness constraint 
cannot increase accuracy.
The following proposition establishes that this tradeoff function is \emph{convex} for linear disparity measures. 
Thus, the lower the disparity level, the larger the additional accuracy cost of fairness. 
\begin{proposition}
\label{prop:tradeoff_convex}
If the disparity measure $\textup{Dis}$ is linear, 
then the tradeoff function $T$ is convex on $[0,\textup{Dis} (0)]$.
\end{proposition}

{A similar FPF has also been recently studied in \cite{wang2023aleatoric} which focuses on characterizing the Pareto frontier between accuracy and a related but different set of fairness constraints. 
While our approach primarily deals with a single linear or bilinear disparity constraint and provides an exact characterization of the fair Pareto frontier in Proposition \ref{prop:FPF}, the work of \cite{wang2023aleatoric} simultaneously considers a set of constraints (i.e., demographic parity, equalized odds, and overall accuracy equality---e.g., \cite{berk2021fairness}) and derives an upper bound on the optimal fairness-accuracy tradeoff using Blackwell’s theory of statistical experiments.}

{A further distinction between our works is that \cite{wang2023aleatoric} provide a numerical procedure to approximate the class of realizable conditional confusion matrices, whereas we exactly characterize a FPF.
However, in practice, we restrict ourselves to a model hypotheses space for the regression functions $\eta_a$ (e.g. logistic regression or neural networks), thereby---as any algorithm does---providing a model-dependent lower bound. 
In particular, our work enables finding these lower bounds in new ways, via new pre- and in-processing algorithms.
}

\subsection{Extensions}
\label{sec:ex}
In this section, we  discuss 
extensions of our theoretical framework to 
four scenarios: 
(1) The protected attribute $A$ is excluded from the predictive features;
(2) Equalized odds;
(3) Demographic parity with a multi-class protected attribute; 
(4) cost-sensitive loss; and 
We provide the insights and main results here and defer details to \Cref{ext}.

\subsubsection{Protected Attribute Not Available at the Prediction Phase} 

In the previous discussion, 
we assumed that the protected feature is available for training and prediction; and that it is allowed to use it in both phases.  
  Although this is a common setting (e.g., \cite{HPS2016,CSFG2017,CHS2020}, etc.), 
  in certain cases, there may be ethical or legal considerations that invalidate it. 
  However, our
  framework can also be applied to fair classification
  when the protected attribute cannot be used at the prediction phase. 
  When the protected attribute is not available for \emph{training}, 
  the problem is very different, and it may require inferring the unobserved protected attribute; this is beyond our scope.

When $A$ is available for training but 
not for prediction/testing, 
the classifier we use at test-time
must be defined on $\X$ rather than on $\X\times \A$. 
 In other words, a classifier is a measurable function $f_X$: $\X\to [0,1]$ with $\widehat{Y}_{f_X}|X\sim \Bern(f_X(X))$. In this section, in addition to writing $\eta^Y_{A=a}(x) = \eta^Y_a(x):=\P(Y=1\mid A=a,X=x)$ 
  for all $x$, $a$, 
 we will further denote $\eta^Y(x)=\P(Y=1\mid X=x)$ and  $\eta^A(x)=\P(A=1\mid X=x)$ for all $x$, 
 the regression functions of $Y$ and $A$ on $X$, respectively. 
 
We call a disparity measure $\textup{Dis}_X$  linear if \eqref{eq:exp dis} holds with $(x,a)$ replaced by $x$, i.e., 
there is a weight function  $  w_{X,\textup{Dis}}:\X\to \R$ such that, for all $f_X:\X\to [0,1]$,
\begin{equation}\label{eq:disnoa}
   \textup{Dis}_X(f_X)= 
   \int_{\X}f_X(x)  w_{X,\textup{Dis}}(x) d\P_X(x).
\end{equation}
We show that, when the protected attribute $A$ is not used for prediction, the three common disparity measures from \eqref{eq:disparity level} are still linear.
\begin{proposition}[Common Disparity Measures are Linear when Not Using $A$]\label{prop:ddpnoa}
   When the classifier depends only on $X$ rather than on $(X,A)$, the disparity measures  $\textup{DD}$, $\textup{DO}$, and $\textup{PD}$ from \eqref{eq:disparity level} are linear with weighting functions defined for all $x$ by
\begin{align*}
&w_{\textup{DD}}(x)=\frac{\eta^A(x)}{p_{1}} -\frac{1-\eta^A(x)}{p_{0}} ;\qquad w_{\textup{DO}}(x)=\frac{\eta^Y_{A=1}(x)\eta^A(x)}{p_{1,1}}-\frac{\eta^Y_{A=0}(x)(1-\eta^A(x))}{p_{0,1}};\\
&w_{\textup{PD}}(x)=\frac{(1-\eta^Y_{A=1}(x))\eta^A(x)}{p_{1,0}}-\frac{(1-\eta^Y_{A=0}(x))(1-\eta^A(x))}{p_{0,0}}.  
\end{align*}
\end{proposition}

For a general disparity measure Dis,
a $\delta$-fair Bayes-optimal classifier is defined as $f^\star_{X,\textup{Dis},\delta}\in\argmin_{f_X:\X\to[0,1]}\{R(f_X): |\textup{Dis}_X(f_X))|\le\delta\}$. 
We observe that the analysis of linear disparity measures also applies to the scenario, leading to the following result.

\begin{corollary}[Bayes-Optimal Classifiers for Linear Disparity Measures, Not Using $A$]\label{cor-opt-dp}
When $A$ is not used for prediction, for a given linear disparity measure $\textup{Dis}_X$ as defined in \eqref{eq:disnoa},
suppose that both $\eta^Y(X)$ and $ w_{X,\textup{Dis}}(X)$ have probability density functions on $\mathcal{X}$.
Let ${D}_{X,\textup{Dis}}: \R\to[-1,1]$ be defined for all $t$ as
$${D}_{X,\textup{Dis}}(t) :=\int_{\X} \lmb  w_{X,\textup{Dis}}(x) \cdot I\lsb
\eta^Y(x)>\frac12 + \frac{t}2   w_{X,\textup{Dis}}(x)\rsb\rmb d\P_X(x).
$$
Let $t_{X,\textup{Dis}}:[0,\infty)\to \R$ be defined for all $\delta\ge 0$ by
\begin{equation}\label{eq:tdnoA}
t_{X,\textup{Dis}}(\delta) = \argmin_t \{|t|: |{D}_{X,\textup{Dis}}(t)|\le \delta\}.
\end{equation}
Then, for any $\delta\ge0$, there exists a $\delta$-fair Bayes-optimal classifier $f^\star_{X,\textup{Dis},\delta}$
taking the form, for all $x\in\X$, 
\begin{equation*}
f^\star_{X,\textup{Dis},\delta}(x)=  
I\left(\eta^Y(x)> \frac12 +\frac{t_{X,\textup{Dis}}(\delta)}2  w_{X,\textup{Dis}}(x)\right).
\end{equation*}
\end{corollary}

{This result can be used to develop similar algorithms to those for the attribute-aware setting.
 Theorems \ref{thm:FUDS1} and \ref{thm:FCSC1} provide the form of the pre- and in-processing algorithms (FUDS and FCSC) for 
 common linear disparity measures (i.e., demographic parity, equality of opportunity and predictive equality).
Pseudocode for the resulting algorithms is provided in Algorithms \ref{alg:FUDS_noA} and \ref{alg:FCSC_noA}.}

{Our algorithms leverage the special form of these metrics and the resulting Bayes-optimal classifiers that we derived, and do not hold for arbitrary linear disparity measures.
In particular, 
it is not clear to us how to leverage/extend the results of \cite{chen2023post} to reach the same conclusion and algorithms.}

\subsubsection{Equalized Odds}

We continue with two examples that handle multiple fairness constraints.
In many applications, achieving fairness is a multifaceted challenge, 
needing simultaneous consideration of different fairness metrics. 
Additionally, there can be multiple protected attributes, which are often multi-class rather than binary.
As an example,
the U.S. Equal Employment Opportunity Commission recognizes 11 discrimination types, mostly with multi-class attributes. 
In such scenarios, managing multiple fairness constraints becomes crucial.  
To illustrate how our method 
can be extended to 
handle multiple fairness constraints, we consider two examples: 
(1) equalized odds, which requires ensuring \emph{both} equality of opportunity \emph{and} predictive equality, and 
(2)  demographic parity with a multi-class protected attribute.

Recalling 
disparity of opportunity $\textup{(DO)}$ and
predictive disparity  $\textup{(PD)}$
from \eqref{eq:disparity level},
a classifier $f$ satisfies the $\delta$-parity constraint with respect to equalized odds from Definition \ref{eod} if
$$\max(|\textup{DO}(f)|,|\textup{PD}(f)|)\le \delta.$$ 
The corresponding
$\delta$-fair Bayes-optimal classifiers are then defined as 
$$f^\star_{\textup{DEO},\delta} \in  {\argmin}_{f\in\mF} \lbb R(f): \max\{|\textup{DO}(f)|,|\textup{PD}(f)|\}\le \delta \rbb. $$
By taking 
$\phi_0, \phi_1, \phi_2$ such that for all $x,a$,
$\phi_0(x,a)=2\eta_a(x)-1$, $\phi_1(x,a)=(2a-1)\eta_a(x)/p_{a,1}$, $\phi_2(x,a)=(2a-1)(1-\eta_a(x))/p_{a,0}$ and applying the generalized Neyman Pearson lemma, 
after some intricate calculations, 
we can derive  the following result.

\begin{theorem}[Bayes-optimal Classifiers for Equalized Odds]\label{thm-fb-eq-odd}
    For any $\delta>0$, there is
    $t_{\textup{DEO},1}(\delta)\in(-p_{0,1},p_{1,1})$ and $t_{\textup{DEO},2}(\delta)\in(-p_{1,0},p_{0,0})$ such that
    a $\delta$-fair Bayes-optimal classifier takes values, for all $x,a$,
    $$f^\star_{\textup{DEO}}(x,a) = I\lsb\
\eta_a(x)>  \frac{p_{a,1}p_{a,0} +(2a-1)t_{\textup{DEO},2}(\delta) p_{a,1}  }
{2p_{a,1}p_{a,0} +(2a-1) \lsb t_{\textup{DEO},2}(\delta) p_{a,1} -t_{\textup{DEO},1}(\delta) p_{a,0}\rsb}\rsb.$$
The values of   $t_{\textup{DEO},1}(\delta)$ and $t_{\textup{DEO},2}(\delta)$ are provided in \eqref{eq:toeod} in  \Cref{sec:equodds}.
\end{theorem}

{Prior work such as \cite{chen2023post}
showed that the optimal classifier belongs to a linear combination of certain scoring functions, but did not find its exact form.}

The proof of Theorem \ref{thm-fb-eq-odd} requires studying intricate properties of 
the DO and PD, as detailed in a series of lemmas (Proposition \ref{prop:mon-D12}, Lemmas \ref{conj}, \ref{lem:fbeqodd}, \ref{lem:existance-eqodd}), as well 
discussing the seven cases that are involved in the definitions of $t_{\textup{DEO},1}(\delta)$ and $t_{\textup{DEO},2}(\delta)$ in \eqref{eq:toeod}.
Briefly, we define a candidate class of potentially Bayes-optimal classifiers, parametrized by two scalars $t_1,t_2$; and show that all Bayes-optimal classifiers belong to this class in Lemma \ref{lem:fbeqodd}.
We show that the DO and PD of these classifiers is continuous and monotone non-increasing as a function of either of these arguments while holding the other one fixed (Proposition \ref{prop:mon-D12}).
This leads to the definitions of $t_{\textup{DEO},1}(\delta)$ and $t_{\textup{DEO},2}(\delta)$ in \eqref{eq:toeod};  where the claim that they are well-defined is proved in Lemma \ref{lem:existance-eqodd}.
Then, we show 
that the most efficient approach, in terms of minimizing changes to PD, to attain a particular level of DO involves solely varying the value of $t_1$. Conversely, when aiming to achieve a specific level of PD, the optimal strategy is to solely adjust the value of $t_2$ (Lemma \ref{conj}).
The proof of  Theorem \ref{thm-fb-eq-odd} follows by careful arguments leveraging all these properties.

{Developing an algorithm here seems to be more challenging than in the basic single constraint setting, 
as the optimal classifier is determined by two non-linear equations, whose efficient solution (beyond grid search) may be nontrivial.}

\subsubsection{Demographic Parity with a Multi-class Protected Attribute}

Next, we illustrate how our approach
can be extended to 
handle demographic parity with a multi-class protected attribute.
For a  multi-class protected attribute $a\in\A=\{1,2,...,|\A|\}$, 
we say that a classifier satisfies \emph{perfect} demographic parity
  if the  $|\A|-1$ linear constraints 
$$\P(\widehat{Y}_f=1\mid A=a)=\P(\widehat{Y}_f=1\mid A=1) \quad \text{ for } \quad  a=2,3,...,|\A|$$
hold.
A fair Bayes-optimal classifier  is defined as
$$f_{\textup{DD},|\A|}^\star\in {\argmin}_{f\in\mF}\,\lbb \P(Y\neq \widehat{Y}_f) : \sum_{a=2}^{|\A|}\lab\P(\widehat{Y}_f=1\mid A=a)=\P(\widehat{Y}_f=1\mid A=1) \rab=0\rbb.$$
By leveraging the generalized Neyman-Pearson lemma, we can show the following result:
\begin{theorem}
[Bayes-optimal Classifier for Demographic Parity with a Multi-class Protected Attribute]
\label{thm-opt-dp-multi}
Let for $a\in\A$, $p_a=\P(A=a)$. Then, there are constants $(t_{\textup{DD},a})_{a=1}^{\A}$ with $\sum_{a=1}^{|\A|} t_{\textup{DD},a}=0$,
such that,
for all $a \in \{2,\ldots,|\A|\}$,
\begin{equation}\label{eq:td-multi}
\P_{X\mid A=a}\left(\eta_a(X)>\frac12+\frac{t_{\textup{DD},a}}{2p_{a}}\right)=
\P_{X\mid A=1}\left(\eta_{1}(x)>\frac12+\frac{t_{\textup{DD},1}}{2p_{1}}\right).
\end{equation}
Moreover, there is a fair Bayes-optimal classifier such that for all $x,a$, 
$$f_{\textup{DD},t_{\textup{DD},1},...,t_{\textup{DD},|\A|}}(x,a) = I\lsb\eta_a(x)>\frac12+\frac{t_{\textup{DD},a}}{2p_a}\rsb. $$
\end{theorem}

Extending our algorithmic procedures to multi-class protected attributes is non-trivial, since there are no analogous monotonicity results.
    To see this, the form of the $\delta$-fair Bayes optimal classifier under demographic parity in the attribute-aware setting is known to be $
    I\lsb \eta_{a_i}(x) > \frac12 + \frac12 \frac{\beta_{a_i}}{p_{a_i}}\rsb
    $
    where $\sum_{i=1}^{|\mathcal{A}|} \beta_{a_i} = 0$ (in the binary protected attribute setting $\mathcal{A}=\{0,1\}$, we have $t=\beta_1=-\beta_0$). For $|\mathcal{A}|>2$, increasing $\beta_1$ may either increase or decrease the demographic disparity between groups $a_i$ and $a_j$ for $i\neq j\neq1$, so a result about the monotonicity of disparity measures does not hold in general. 
    Thus, developing algorithmic analogs of FUDS and FCSC is not straightforward, since those rely on the monotonicity properties of the disparity measure in the hyperparameter $t$.

\subsubsection{Cost-sensitive Risk} \label{sec:ext-csr}
cost-sensitive classification is useful 
when the consequences of false negatives and positives are unequally important, 
such as in certain medical applications. 
For a given cost parameter $c\in[0,1]$, the cost-sensitive 
classification error 
of a classifier $f$ is defined as:
\begin{equation}\label{csr}
    R_c(f)=c\cdot \P(\widehat{Y}_f=1,Y=0)+(1-c)\cdot \P(\widehat{Y}_f=0, Y=1).
\end{equation}
When $c=1/2$, cost-sensitive risk reduces to the usual mis-classification error.

Taking fairness into account, 
let $\mathcal{F}_\delta = \{f\in \mathcal{F}: |\textup{Dis}(f)| \le \delta\} $
be the set of $\delta$-fair  classifiers.
A $\delta$-fair Bayes-optimal classifier
for the cost-sensitive risk is defined as 
$f_{\textup{Dis},\delta,c}^\star\in {\text{argmin}_{f\in\mathcal{F}_\delta}} R_c(f),$ or equivalently
$
f_{\textup{Dis},\delta,C}^\star\in {\argmax}_{f\in\mathcal{F}_{\delta}}\, \int_{\A}\int_{\X}f(x,a) (\eta_a(x)-c)d\P_{X,A}(x,a)
$. 
Since this is a linear objective, our approach
extends seamlessly to find fair Bayes-optimal classifiers.

\begin{corollary}[Fair Bayes-optimal Classifiers for a cost-sensitive Risk]\label{cor-opt-cost-sen}
For a given linear disparity measure $\textup{Dis}$,
suppose that for $a\in\{0,1\}$, both $\eta_a(X)$ and $ w_{\mathrm{Dis}}(X,a)$ have probability density functions on $\mathcal{X}$. Let ${D}_{\textup{Dis},c}: \R\to[-1,1]$  be defined  for all $t$ as
$$
{D}_{\textup{Dis},c}(t) := \sum_{a\in\{0,1\}}p_a\int_{\X} \lmb  w_{\textup{Dis}}(x,a) \cdot I\lsb
\eta_a(x)>c +t w_{\textup{Dis}}(x,a)\rsb\rmb d\Pa(x),
$$
and let $t_{\textup{Dis},c}:[0,\infty)\to \R$ be defined as, for all $\delta\ge 0$,  $t_{\textup{Dis},c}(\delta) = \argmin_t \{|t|: |{D}_{\textup{Dis},c}(t)|\le \delta\}$. Then, for any $\delta>0$, there exists a  $\delta$-fair Bayes-optimal classifier  $f^\star_{\textup{Dis},c}$ under cost-sensitive risk, such that for all $x,a$, 
$$f^\star_{\textup{Dis},c}(x,a) = I \lsb\eta_a(c)>c+ t_{\textup{Dis},c}\cdot  w_{\textup{Dis},c}(x,a)\rsb,$$

\end{corollary}

\section{Bayes-Optimal Fair Classifiers via Pre-, In-, and Post-processing}\label{sec:alg}

In this section, we 
propose a comprehensive set of algorithms for Bayes-optimal fair classification, 
based on pre-, in-, and post-processing.
These algorithms aim to ensure fairness at different stages of the training process.  
Each class of methods
has strengths and limitations. 
Pre- and post-processing methods can be used without significant changes to standard classification methods, 
facilitating the use of widely-available software. 
However, these approaches might 
potentially clash with certain data protection regulations \citep{barocas2016big}.
On the other hand, 
in-processing methods  are applicable whenever standard training workflows are,
but may require modifying the fitting methods or optimization objectives. 

We develop methods of each type aiming to recover the Bayes-optimal classifiers.
Our methods include pre-processing via fair up-/down-sampling, 
in-processing via group-based cost-sensitive classification, 
and post-processing via 
plug-in estimation. 
Each method inherits the strengths mentioned above.

In this section, we consider bilinear disparity measures as per Definition \ref{bldm} for simplicity. 
However, our methods are adaptable to the more general scenario of linear disparity measures,
as shown in case studies in Section \ref{sec:ext:Anotobserve} of the Appendix. 
Given a bilinear disparity measure characterized by $s_{\textup{Dis},a}$ and $b_{\textup{Dis},a}$ for $a\in\mA$,
we denote for $ t\in\R$ the threshold function
\begin{equation}\label{eq:defh}
    H_{\textup{Dis},a}(t)=
\frac{1 + t\cdot b_{\textup{Dis},a} }{2-t\cdot s_{\textup{Dis},a}}.
\end{equation}
 Then, the disparity function $D$ from \eqref{eq:dt} has values, for all  $t\in\R$,
\begin{align}\label{eq:dis_new}
{D}_{\textup{Dis}}(t) = \sum_{a\in\{0,1\}}p_a\int_{\X} \lmb  \lsb s_{\textup{Dis},a}\eta_a(x)+b_{\textup{Dis},a}\rsb \cdot I\lsb
\eta_a(x)>H_{\textup{Dis},a}(t)\rsb\rmb d\Pa(x).
\end{align}
In the following, we let the 
observed data points $\mathcal{S}_n=\{(x_i,a_i,y_i)\}_{i=1}^n$ be 
an independent and identically distributed sample from $\P$ over the domain
$\X \times \A \times \mathcal{Y}$.  
Moreover, we  separate the data 
according to the protected feature, 
letting for $a\in\{0,1\}$, 
$\mathcal{S}_{n,a}= \{(x_i,a_i,y_i)\in \mathcal{S}_{n}, a_i=a \}$.
For $a\in\{0,1\}$, 
the $j$-th element of $\mathcal{S}_{n,a}$ is denoted as $(\xaj,a,y_{a,j})$, for $j\in[n_a]$. We further separate the dataset $S_{n,a}$ into $\mathcal{S}_{n,a,1}$ and $\mathcal{S}_{n,a,0}$, according to the label information. 
Moreover, we denote, for $a,y\in\{0,1\}^2$, $n_a=|\mathcal{S}_{n,a}|$ and $n_{a,y}=|\mathcal{S}_{n,a,y}|$.

\subsection{Pre-processing: Fair Up- and Down-Sampling (FUDS)}
\label{fuds}
Many pre-processing algorithms use heuristics to eliminate the impact of protected attributes on the common features $X$ and the label $Y$ \citep{KJ2016,XYZ2018}. 
Nonetheless, pre-processing a dataset to satisfy a fairness criterion does not guarantee fair  outcomes, see e.g., \cite{eitan2022fair}. 
To illustrate this, consider a transformed distribution $(\widetilde{X},\widetilde{A},\widetilde{Y})$, 
where $(\widetilde{X},\widetilde{Y})$ is independent of ${A}$. 
A classifier $f$ trained on a sample with the distribution $(\widetilde{X},\widetilde{Y})$ 
would be independent of $A$. 
Yet, during the testing phase, the classifier's output $f(X)$ may still exhibit unfairness if $X$ and $A$ are dependent.
This observation leads us to ask:
\emph{On what data distributions $(\widetilde{X},\widetilde{A},\widetilde{Y})\sim \widetilde{\P}$ can we learn 
classifiers without a fairness constraint and recover  a fair Bayes-optimal classifier for the original distribution $({X},{A},{Y})\sim {\P}$?} 

We propose an answer by designing a joint distribution for $(\widetilde A,\widetilde Y)$, while keeping the conditional distribution of $\widetilde X$ given $\widetilde A$ and $\widetilde Y$ unchanged from that of 
$X$ given $A$ and $Y$.


We will consider distributions $\tP^{\,t}$ depending on a scalar $t\in \R$.
 For the distribution $\tP^{\,t}$  of 
$(\tX,\tA,\tY)$, we denote 
for all $x,a,y$, 
$\tp^{\,t}_{a,y}=\tP^{\,t}(\tA=a,\tY=y)$,  denote 
$\teta^{\,t}_a(x) = \tP^{\,t}(\tY=1|\tA=a,\tX=x)$,
and denote by
$\widetilde{\P}^t_{\widetilde{X}|\widetilde{A}=a,\widetilde{Y}=y}(\cdot)$ the marginal distribution of $\tX$ given $\tA=a$ and $\tY=y$.

\begin{theorem}[Bayes-Optimal Fair Up- and Down-Sampling]\label{thm:FUDS} 
Suppose that for $a\in\{0,1\}$, both $\eta_a(X)$ and $ w_{\mathrm{Dis}}(X,a)$ have probability density functions on $\mathcal{X}$.
For $t\in \R$,
let $\tP^{\,t}$ be the distribution
satisfying 
that  for all $(x,a,y) \in \mX \times\{0,1\}^2$,
$\widetilde{\P}^t_{\widetilde{X}|\widetilde{A}=a,\widetilde{Y}=y}(x)=\P_{{X}|{A}=a,{Y}=y}(x)$, 
and, 
with $H_{\textup{Dis},a}$ from \eqref{eq:defh},
\begin{eqnarray}\label{adj}
\tp^{\,t}_{a,y}=\tP^{\,t}\lsb \widetilde{A}=a,\widetilde{Y}=y\rsb =c_a[(1- H_{\textup{Dis},a}(t))y+ H_{\textup{Dis},a}(t)(1-y)]p_{a,y}.
\end{eqnarray}
Here $c_1>0$ and $c_0>0$ are 
such that $\widetilde{p}^{\,t}_{11}+\widetilde{p}^{\,t}_{10}+\widetilde{p}^{\,t}_{01}+\widetilde{p}^{\,t}_{00}=1$. 
Then, the  unconstrained   Bayes-optimal classifier for $\widetilde{\P}^{\,t}$
is  a $|\Dt|$-fair Bayes-optimal classifier for $\P$. 
Thus, for a given $\delta\ge0$,
any Bayes-optimal classifier for $\widetilde{\P}^{\,t_{\textup{Dis}}(\delta)}$ is a $\delta$-fair Bayes-optimal classifier for $\P$.
22\end{theorem}

Theorem \ref{thm:FUDS} 
designs a fairness-inducing distribution by
multiplying the probabilities 
of each group with a group-wise factor, as per \eqref{adj}.  
As a result, 
$(1-H_{\textup{Dis},a}(t))y+H_{\textup{Dis},a}(t)(1-y)$ can be viewed as
a fairness-inducing 
factor for the group with $A=a$ and $Y=y$. 
As can be seen from \eqref{eq:dis_new},
a large 
threshold
$H_{\textup{Dis},a}(t)$ for the group with $A=a$
indicates that the fair Bayes-optimal classifier specifies a higher acceptance standard for group $A=a$. 
In other words, the unconstrained classifier 
has a higher acceptance rate 
on the subgroup with $A=a$.
To avoid this disparity, we need to generate more data with $Y=0$ 
and less data with $Y=1$, 
for the group with $A=a$; 
reducing $\eta_a(x)=\P(Y=1|X=x,A=a)$.  
For this reason, we call our approach Fair Up- and Down-Sampling (FUDS).

{
{\bf An extension.}
It is possible to show more generally that if the $\delta$-fair Bayes-optimal classifier has the form 
$I(\eta_a(x)>H(a,\delta))$, then re-weighting for all $(x,a,y) \in \mX \times\{0,1\}^2$, by
$\widetilde{\P}^t_{\widetilde{X}|\widetilde{A}=a,\widetilde{Y}=y}(x)=\P_{{X}|{A}=a,{Y}=y}(x)$, 
and, 
\begin{eqnarray*}
\tp^{\,t}_{a,y}=\tP^{\,t}\lsb \widetilde{A}=a,\widetilde{Y}=y\rsb =c_a[(1- H(a,\delta))y+ H(a,\delta)(1-y)]p_{a,y}
\end{eqnarray*}
leads to a distribution whose unconstrained Bayes-optimal classifier is a $\delta$-fair Bayes-optimal classifier for original distribution.  
}

{
However, in the more general setting of linear disparity measures, which includes the attribute blind setting (Section \ref{sec:ex}), developing pre- and in-processing algorithms is delicate.
Though we have shown 
how to do this 
for DD, DO, PD, 
this does not extend 
to general classifiers of the form $I(\eta^Y(x)>H(x;a,\delta))$.
Nevertheless, we show that such algorithms do exist for these common fairness criteria in the attribute-blind setting, and we formalize these findings in Theorems \ref{thm:FUDS1} and \ref{thm:FCSC1}.}

{{\bf Resampling as a pre-processing method.}} 
{In the broader fairness literature, both reweighting and resampling are widely used approaches for developing fairness aware pre-processing algorithms (e.g., \cite{kamiran2012data}, \cite{chai2022fairness}, \cite{CFDV2017}, and \cite{hort2024bias}).
Reweighting and resampling
are equivalent in our case: assigning weights to data points implicitly defines a corresponding resampling distribution, and vice versa.
While we presented the re-sampling version here, our methods are equally applicable to develop group-wise re-weighting methods.}

{\textbf{Data augmentation vs. Resampling methods.}
There have been a number of recent studies that compare generative augmentation methods to 
resampling techniques on tabular datasets.
For instance, \cite{panagiotou2024synthetic} find that models trained on synthetic data from CTGAN \citep{xu2019modeling} and TVAE \citep{patki2016synthetic} outperform SMOTE \citep{chawla2002smote} by tracing a more favorable fair Pareto frontier. More recently, \cite{hastings2025data} demonstrate that diffusion-based tabular data augmentation improves upon the work of \cite{kamiran2012data}. 
}

{However, the verdict of which fairness-aware approach is better remains context-dependent: modern resampling methods such as FairBatch \citep{roh2020fairbatch}, Adaptive Priority Reweighing \citep{hu2023adaptive}
can have advantages,
especially when generative models cannot reliably approximate the conditional distribution $\P(X,Y|A=a)$.}

{Ultimately, our method, FUDS, uses group-aware resampling of the original dataset, which differs fundamentally from generative data augmentation. A direct empirical comparison with data augmentation frameworks would require extensive implementation and domain-specific tuning. We view this as an important direction for future work, but it lies beyond the scope of the current work.}

\subsubsection{Algorithmic implementation}

\begin{algorithm} 
   \caption{Fair Up-/Down-Sampling (FUDS)}
   \label{alg:FUDS}

   \KwIn {Step size $\alpha>0$; Disparity Level $\delta\ge0$; Error tolerance level $\ep>0$; Dataset $S=S_1\cup S_0$ with  $S=\{x_{i},a_i,y_{i}\}_{i=1}^{n}$, $S_1=\{x_{1,i},y_{1,i}\}_{i=1}^{n_1}$ and $S_0=\{x_{0,i},y_{0,i}\}_{i=1}^{n_0}$; $n_{a,y}  = \#\{a_i=a,y_i=y\}$.}
   \BlankLine
\hrule
   \BlankLine

\textbf{Disparity Estimation sub-routine: Construct $\widehat{f}^{\textup{FUDS}}_t$ and estimate $\textup{Dis}(\widehat{f}^{\textup{FUDS}}_t)$ for a given threshold parameter $t$:}

 1. For all $a$, estimate $\widehat{w}_{\textup{Dis}}$ and $\widehat{H}_{\textup{Dis},a}$ using the expressions in \eqref{eq:disparity level expression} and \eqref{eq:defh}. 
 
 2. For all $a,y$, let $\widehat{\widetilde{p}}^{\,t}_{a,y}$ be as in \eqref{htp}.

 3. Apply up- and down-sampling to generate a new dataset $S_{t}$ with proportions $\widehat{\widetilde{p}}_{a,y}^{\,t}$ for all $a,y$. 
 
 4. Fit classifier $\widehat{f}^{\textup{FUDS}}_{t}$ on the dataset $S_t$ using any method. 
 
 5. Estimate the disparity level of $\widehat{f}^{\textup{FUDS}}_{t}$ as in \eqref{dise}.
 \BlankLine
 \hrule
 \BlankLine

\textbf{Estimate the $\delta$-fair Bayes-optimal classifier:}

   Run Disparity Estimation sub-routine with $t=0$.
    
  \uIf{$\left|\widehat{\textup{Dis}}^{\textup{FUDS}}({0})\right|\le \delta$} {
 $\widehat{t}_{\textup{Dis}}(\delta)=0$.}
  \Else{
  \uIf{$\widehat{\textup{Dis}}^{\textup{FUDS}}({0})> \delta$}{ $\delta'=\delta$; $t_{\min}=0$, $t_{\max}=1$.}
  \Else {$\delta = -\delta$; $t_{\min}=-1$, $t_{\max}=0$.}
\While {$t_{\max}-t_{\min}>\ep$}
        {$t =  (t_{\max}-t_{\min})/2;$ 
        
     Run Disparity Estimation sub-routine with current $t$.
     
 \uIf {$\widehat{\textup{Dis}}^{\textup{FUDS}}(t) > \delta$} {$t_{\max} = t$}
 \Else {$t_{\min} = t$}}
$\widehat{t}_{\textup{Dis}}(\delta)=t.$}

 {\bfseries Output:}  
$\widehat{f}_{\textup{Dis},\delta}=\widehat{f}^{\textup{FUDS}}_{\widehat{t}_{\textup{Dis}}(\delta)}.$
 \BlankLine

\end{algorithm}
While the sampling probabilities from \eqref{adj}
depend on unknown population parameters, 
we can estimate them from the training data.
Then, 
we can generate approximately fair datasets using up- and down-sampling; see \Cref{alg:FUDS}.
According to Theorem \ref{thm:FUDS}, for each value of $t$, there is a fair Bayes-optimal classifier for \emph{some} level---$|\Dt|$---of disparity. 
Therefore, by adjusting $t$, we can create a range of data distributions, each leading to a different point on the fair Pareto frontier as defined in \eqref{eq:vec-opt}.  
When we need a dataset with 
a predetermined disparity level $\delta$,
we can estimate the threshold parameter $t_{\textup{Dis}}(\delta)$ using an iterative method for updating $t$, see Algorithm \ref{alg:FUDS}. 

Specifically, for a given $t$, we estimate $f^{\textup{FUDS}}_t$ and its disparity level as follows: 
\begin{itemize}[]
    \item  (1) 
    For all $x,a,y,t$,
    estimate $p_{a,y}$ by $n_{a,y}/n$ and 
    estimate $\wD(x,a)$ and $\Hta$ from \eqref{eq:disparity level expression} and \eqref{eq:defh} by plug-in estimation. 
    Taking demographic disparity as an example, we let for all $x,a$ and for all $t\in \R$,
$$
\widehat{w}_{\textup{DD}}(x,a)=\frac{n}{n_{a,1}+n_{a,0}}, \ \ \text{ and } \ \ 
\widehat{H}_{\text{DD},a}(t)=\frac12 + \frac{(2a-1)nt}{n_{a,1}+n_{a,0}}.
$$
\item (2) Estimate the group-wise proportions, for $(a,y)\in\{0,1\}^2$,
\begin{equation}\label{htp}
    \widehat{\widetilde{p}}^{\,t}_{a,y}=\widehat c_a[(1- \widehat{H}_{\textup{Dis},a}(t))y+ \widehat{H}_{\textup{Dis},a}(t)(1-y)]\cdot n_{a,y}/n.
\end{equation}
Here $\widehat c_1$ and $\widehat c_2$ are 
such that $\widehat{\widetilde{p}}^{\,t}_{11}+\widehat{\widetilde{p}}^{\,t}_{10}+\widehat{\widetilde{p}}^{\,t}_{01}+\widehat{\widetilde{p}}^{\,t}_{00}=1$. 
\item (3) Apply up- and down- sampling to generate a new dataset $S_t
=\{\widetilde{x}_i,\widetilde{a}_i,\widetilde{y}_i\}_{i=1}^n$ with 
$\widetilde{n}^{\,t}_{a,y}=\#\{\widetilde{a}_i=a,\widetilde{y}_i=y\}
=\lfloor n\cdot \widehat{\tpt}_{a,y} \rfloor $.
The precise sampling method is presented in the paragraphs below.
\item (4) Fit a classifier $\widehat{f}^{\textup{FUDS}}_{t}$ to the dataset $S_t$ using any method, such as logistic regression or deep neural nets.
\item (5) Estimate the disparity level of $\widehat{f}^{\textup{FUDS}}_{t}$ for the given $t$:
\begin{equation}\label{dise}
  \widehat{\textup{Dis}}^{\textup{FUDS}}(t)=\frac1{n_1}\sum_{j=1}^{n_1} \widehat{f}^{\textup{FUDS}}_{t}(\xoj,1) \widehat{w}_{\textup{Dis}}(\xoj,1)-\frac1{n_0}\sum_{j=1}^{n_0} \widehat{f}^{\textup{FUDS}}_{t}(\xzj,0) \widehat{w}_{\textup{Dis}}(\xzj,0).  
\end{equation} 
\end{itemize}
Following this, $t$ can be adjusted using the bisection method for solving ${D}_{\textup{Dis}}(t)=0$, based on the estimating equation $\widehat{\textup{Dis}}^{\textup{FUDS}}(t)=0$. 
This is grounded in our observation that, as per Proposition \ref{prop:monotonicity}, at a population level, the disparity function $\textup{Dis}(\widehat{f}^{\textup{FUDS}}_t)$ is monotone non-increasing in $t$.

Performing independent random resampling of the original dataset 
$S_n$ 
  in each iteration can lead to significant variability in the datapoints---and also in the induced classifiers. 
  To address this issue, we propose a sampling strategy aimed at reducing variability.
   
   Consider the generated dataset $\tmStm$ for a given threshold 
$t'$, and write it  as $\tmStm=\cup_{a,y\in\{0,1\}^2}\tmStmay$, where for all $a,y$, $\tmStmay$
  represents the set of points with $\widetilde{a}_i=a$ and $\widetilde{y}_i=y$.  
  Let $t$ be the next threshold, and let
  $\widetilde{n}^{\,t'}_{a,y}=\lfloor n\cdot \tptpr_{a,y} \rfloor $ and
  $\widetilde{n}^{\,t}_{a,y}=\lfloor n\cdot \tpt_{a,y} \rfloor $ for all $a,y$.
  Our strategy considers the following two cases, for all $a,y$:
\begin{itemize}
\item If $\widetilde{n}_{a,y}^{\,t}\ge\widetilde{n}_{a,y}^{\,t'}$, we randomly select $\widetilde{n}_{a,y}^{\,t}-\widetilde{n}_{a,y}^{\,t'}$ data points from $S_{n,a}$ to add to  $\tmStmay$, forming $\tmStay$.

\item If $\widetilde{n}_{a,y}^{\,t}<\widetilde{n}_{a,y}^{\,t'}$, we randomly select $\widetilde{n}_{a,y}^{\,t}$ data points from  $\tmStmay$ to form $\tmStay$.
\end{itemize}
Our simulation studies indicate that this strategy offers greater stability than independent random sampling in each iteration, leading to a more favorable fairness-accuracy tradeoff and reduced variance in the empirical performance.

Another challenge is posed by computational complexity. 
As discussed above, we fit a classifier at each iteration and then estimate its disparity level. 
For larger datasets and more complex models, this procedure can become time-consuming.
A faster method is to adjust $t$ while fitting each classifier, 
fine-tuning the optimal sampling ratio concurrently with model training. 
We only need to revise step (4): 
for instance, when training via an iterative optimization-based method,
rather than finishing the entire training procedure, we only 
 train for a small number of steps---or, epochs---in each iteration to obtain a classifier $\widehat{f}_t^{\textup{FUDS}}$. 
 The classifier obtained in the each iteration is  used as the starting point in the next iteration.

\subsection{In-Processing: Fair cost-sensitive Classification (FCSC) }
\label{sec:fcsc}
In this section, we propose an 
in-processing fair classification method. 
Our objective is to identify a risk function whose 
unconstrained 
minimizer is the fair Bayes-optimal classifier. 
The key observation is that 
fair Bayes-optimal classifiers
adjust thresholds for each protected group, 
which is also known to occur in cost-sensitive classification.

Recall the cost-sensitive risk $R_c$ from \eqref{csr}.
An unconstrained Bayes-optimal classifier for this cost-sensitive risk, denoted as $f^\star_c$,
 is a classifier that minimizes $R_c(f)$, i.e., $f_c^\star\in \argmin_{f\in\mF} R_c(f).$  
 All such classifiers can be represented as
$f_c^\star(x,a)=I(\eta_a(x)>c)$ for all $x,a$, \citep[see e.g.,][]{C2001CS}. 
Drawing inspiration from this connection, 
we propose a \emph{group-wise} cost-sensitive classification approach. 
We aim to shift the thresholds used for the protected groups in opposite directions 
to reduce disparity. 
\begin{definition}[Fair cost-sensitive Risk]\label{fcsr}
Recall $H_{\textup{Dis},a}$, $a\in \mA$ from \eqref{eq:defh}     
and let $c_{a,y}(t)=(1-2y)H_{\textup{Dis},a}(t)+y$ for all $t,a,y$. 
For $t\in \R$, 
   define the following \emph{fair cost-sensitive risk} of a classifier $f$:
$$R^{\textup{FCSC}}_t(f)=\sum_{a\in\{0,1\}}\sum_{y\in\{0,1\}}c_{a,y}(t) \P(\widehat{Y}_f=1-y,Y=y,A=a).$$
\end{definition}
\emph{Fair cost-sensitive classification} (FCSC) minimizes the 
fair cost-sensitive risk from Definition \ref{fcsr}:
\begin{equation}\label{ffcsc}
f^{\textup{FCSC}}_{t} = \argmin_{f\in{\mF}}\, R^{\textup{FCSC}}_t(f).
\end{equation}
Our key result here is that fair cost-sensitive classification is Bayes-optimal.
\begin{theorem}[Fair cost-sensitive Classification is Bayes-Optimal]\label{thm:FCSC}
Letting $t^\star_0 = t_{\textup{Dis}} (0) =\argmin_{t}\{|t|: |{D}_{\textup{Dis}}(t)|=0\}$,  
we have, for $t\in [\min(t^\star_0,0),\max(t^\star_0,0)]$,  with $\Dt$ defined in \eqref{eq:dt},   
that $f^{\textup{FCSC}}_{t}$ from \eqref{ffcsc} is a $|\Dt|$-fair Bayes-optimal classifier for  $\P$. 
  In particular,   $f^{\textup{FCSC}}_{ t_{\textup{Dis}}(\delta)}$ is a $\delta$-fair Bayes-optimal classifier for $\P$.
\end{theorem}
Theorem \ref{thm:FCSC} shows that carefully designed group-wise cost-sensitive classification
can lead to a fair Bayes-optimal classifier. 
Similar to the fair sampling method from \Cref{fuds}, one can generate a fairness-accuracy tradeoff curve by using FCSC and varying $t$. 
To achieve a specific disparity level  $\delta$, 
an updating method similar to that from \Cref{alg:FUDS}
can be used to estimate $ t_{\textup{Dis}}(\delta)$; see Algorithm \ref{alg:FCSC}.

\begin{algorithm}
   \caption{Fair cost-sensitive Classification (FCSC)}
   \label{alg:FCSC}

   \KwIn {Step size $\alpha>0$; Disparity Level $\delta\ge0$, Error tolerance level $\ep>0$; Dataset $S=S_1\cup S_0$ with  $S=\{x_{i},a_i,y_{i}\}_{i=1}^{n}$, $S_1=\{x_{1,i},y_{1,i}\}_{i=1}^{n_1}$ and $S_0=\{x_{0,i},y_{0,i}\}_{i=1}^{n_0}$; $n_{a,y}  = \#\{a_i=a,y_i=y\}$; Step size $\alpha\ge 0$.}
   \BlankLine

\textbf{Disparity Estimation sub-routine: Construct $\widehat{f}^{\textup{FCSC}}_t$ and estimate $\textup{Dis}(\widehat{f}^{\textup{FCSC}}_t)$ for a given threshold parameter $t$:}

 1. For all $a$, estimate  $\widehat{w}_{\textup{Dis}}$  and $\widehat{H}_{\textup{Dis},a}$ by plug-in estimation using the expressions in \eqref{eq:disparity level expression} and \eqref{eq:defh}. 
 
 2. Denote $\widehat c_{a,y}(t) = (1-2y)\widehat{H}_{\textup{Dis},a}(t)+y$, for all $t,a,y$.

 3. Use any cost-sensitive classification method to fit $\widehat{f}^{\textup{FCSC}}_{t}(\cdot,1)$ on $S_1$ and $\widehat{f}^{\textup{FCSC}}_{t}(\cdot,0)$ on $S_0$.
 
4. Estimate the disparity level of $\widehat{f}^{\textup{FCSC}}_{t}$ as in \eqref{dise} with $\widehat{f}^{\textup{FCSC}}_{t}$ instead of $\widehat{f}^{\textup{FUDS}}_{t}$.
 \BlankLine
 \hrule
 \BlankLine

  \begin{multicols}{2}
  {
\textbf{Estimate the fair Pareto frontier:}
 \BlankLine
\hrule
   \BlankLine
   Run Disparity Estimation sub-routine with $t=0$.
    
\uIf { $\widehat{\textup{Dis}}^{\textup{FCSC}} (0)>0$} {$\alpha' = \alpha$}
 \Else { $\alpha' = -\alpha$}
 \While{ $\widehat{\textup{Dis}}^{\textup{FCSC}}(t)\cdot \widehat{\textup{Dis}}^{\textup{FCSC}} (0) > 0$}
 { $t = t+\alpha'$. 
 
Run Disparity Estimation sub-routine with current $t$.
 }
 {\bfseries Output:}  $\{\widehat{f}^{\textup{FCSC}}_{t}\}_{t=0,1,2,...}$.}
\columnbreak

\textbf{Estimate the $\delta$-fair Bayes-optimal classifier:}
 \BlankLine
\hrule
   \BlankLine
   Run Disparity Estimation sub-routine with $t=0$.
  \uIf{$\left|\widehat{\textup{Dis}}^{\textup{FCSC}}({0})\right|\le \delta$} {
 $\widehat{t}_{\textup{Dis}}(\delta)=0$.}
  \Else{
  \uIf{$\widehat{\textup{Dis}}^{\textup{FCSC}}({0})> \delta$}{ $\delta'=\delta$; $t_{\min}=0$, $t_{\max}=1$.}
  \Else {$\delta = -\delta$; $t_{\min}=-1$, $t_{\max}=0$.}
\While {$t_{\max}-t_{\min}>\ep$}
        {$t =  (t_{\max}-t_{\min})/2;$ 
        
     Run Disparity Estimation sub-routine with current $t$.
     
 \uIf {$\widehat{\textup{Dis}}^{\textup{FCSC}}(t) > \delta$} {$t_{\max} = t$}
 \Else {$t_{\min} = t$}}
$\widehat{t}_{\textup{Dis}}(\delta)=t.$}

 {\bfseries Output:}  
$\widehat{f}_{\textup{Dis},\delta}=\widehat{f}^{\textup{FUDS}}_{\widehat{t}_{\textup{Dis}}(\delta)}.$
  \end{multicols}
 \BlankLine

\end{algorithm}

Similar to the FUDS algorithm,
the computational efficiency of FCSC can be improved
by adjusting the threshold  parameter 
$t$ and the associated loss function after a few training epochs (rather than after the whole training process in step 3).

\subsection{Post-Processing: Fair Plug-in Thresholding Rule (FPIR)}
\label{sec:fpir}

Finally, we consider a post-processing algorithm that modifies the model output to control the disparity. 
Given the near-explicit form of Bayes-optimal classifiers from \Cref{thm-opt-dp}, 
after fitting any classifier $\widehat{f}$, 
we can estimate the group-wise thresholds. 
Such an approach has been proposed in a more limited case for demographic parity and equality of opportunity in \cite{MW2018}.
In contrast, our framework is applicable to \emph{any} bilinear disparity measure.

\begin{algorithm}
   \caption{Fair Plug-in Rule (FPIR)}
   \label{alg:FPIR}

   \KwIn {Step size $\alpha>0$; Disparity level $\delta\ge0$; Error tolerance level $\ep>0$; Dataset $S=S_1\cup S_0$ with  $S=\{x_{i},a_i,y_{i}\}_{i=1}^{n}$, $S_1=\{x_{1,i},y_{1,i}\}_{i=1}^{n_1}$ and $S_0=\{x_{0,i},y_{0,i}\}_{i=1}^{n_0}$; $n_{a,y}  = \#\{a_i=a,y_i=y\}$.}
   \BlankLine

   \BlankLine   {\textbf{Step 1}:} Construct estimate $\widehat{\eta}_a$ of $\eta_a$ using any approach, for all $a$.
   \BlankLine

  {\textbf{Step 2}:  Estimate the the fair Pareto frontier or the $\delta$-fair Bayes-optimal classifier: 
}

\textbf{Disparity Estimation sub-routine: Construct $\widehat{f}^{\textup{FPIR}}_t$ and estimate $\textup{Dis}(\widehat{f}^{\textup{FPIR}}_t)$ with threshold parameter  $t$:}

 1. For all $a$, estimate  $\widehat{w}_{\textup{Dis}}$  and $\widehat{H}_{\textup{Dis},a}$ by plug-in estimation using the expressions in \eqref{eq:disparity level expression} and \eqref{eq:defh}. 

2. Let $\widehat{f}^{\mathrm{FPIR}}_t$ be defined by 
$\widehat{f}^{\mathrm{FPIR}}_t(x,a) = I(\widehat{\eta}_a(x)> \widehat{H}_{\textup{Dis},a}(t))$ for all $x,a$.

3. Evaluate the disparity level of $\widehat{f}^{\textup{FPIR}}_{t}$ as in \eqref{dise} with $\widehat{f}^{\textup{FPIR}}_{t}$ instead of $\widehat{f}^{\textup{FUDS}}_{t}$.
 
 \BlankLine
 \hrule
 \BlankLine

  \begin{multicols}{2}
  {
\textbf{Estimate the fair Pareto frontier:}
 \BlankLine
\hrule
   \BlankLine
   Run Disparity Estimation sub-routine with $t=0$.
    
\uIf { $\widehat{\textup{Dis}}^{\textup{FPIR}} (0)>0$} {$\alpha' = \alpha$}
 \Else { $\alpha' = -\alpha$}
 \While{ $\widehat{\textup{Dis}}^{\textup{FPIR}}(t)\cdot \widehat{\textup{Dis}}^{\textup{FPIR}} (0) > 0$}
 { $t = t+\alpha'$. 
 
Run Disparity Estimation sub-routine with current $t$.
 }
 {\bfseries Output:}  $\{\widehat{f}^{\textup{FPIR}}_{t}\}_{t=0,1,2,...}$.}
\columnbreak

\textbf{Estimate the $\delta$-fair Bayes-optimal classifier:}
 \BlankLine
\hrule
   \BlankLine
   Run Disparity Estimation sub-routine with $t=0$.
    
  \uIf{$\left|\widehat{\textup{Dis}}^{\textup{FPIR}}({0})\right|\le \delta$} {
 $\widehat{t}_{\textup{Dis}}(\delta)=0$.}
  \Else{
  \uIf{$\widehat{\textup{Dis}}^{\textup{FPIR}}({0})> \delta$}{ $\delta'=\delta$; $t_{\min}=0$, $t_{\max}=1$.}
  \Else {$\delta = -\delta$; $t_{\min}=-1$, $t_{\max}=0$.}
\While {$t_{\max}-t_{\min}>\ep$}
        {$t =  (t_{\max}-t_{\min})/2;$ 
        
Run Disparity Estimation sub-routine with current $t$.
     
 \uIf {$\widehat{\textup{Dis}}^{\textup{FPIR}}(t) > \delta$} {$t_{\max} = t$}
 \Else {$t_{\min} = t$}}
$\widehat{t}_{\textup{Dis}}(\delta)=t.$}

 {\bfseries Output:}  
$\widehat{f}_{\textup{Dis},\delta}=\widehat{f}^{\textup{FPIR}}_{\widehat{t}_{\textup{Dis}}(\delta)}.$
  \end{multicols}
 \BlankLine
\end{algorithm}

First, 
we can use any method to estimate the feature-conditional probability of $Y=1$ for each protected group. 
Second, we estimate the threshold for each protected group.
One can either vary the value of $t$ to generate fairness-accuracy tradeoff curves or solve the 
estimated one-dimensional fairness equation for a specific disparity level. 

Specifically, let $\widehat{\eta}_a$ be an estimator of $\eta_a$, and consider the following plug-in rule for all $x,a$:
\begin{align*}
\widehat{f}^{\textup{FPIR}}_t(x,a)&=
I\left(\widehat\eta_{a}(x)>\widehat{H}_{\textup{Dis},a}\right),
\end{align*}
where $\widehat{H}_{\textup{Dis},a}$ is estimated using the expression in  \eqref{eq:defh}.  
Now, our goal is to construct an estimate $\widehat{t}_\delta$ such that 
$\widehat{f}_{\widehat{t}_\delta}$ 
approximately satisfies the fairness constraint. 
{We
define $\widehat{\textup{Dis}}(t)$, an estimator of the disparity function $\Dt$, as in \eqref{dise} with $\widehat{f}_{t}^{\textup{FPIR}}$ instead of $\widehat{f}^{\textup{FUDS}}_{t}$,
where $\widehat{w}_{\textup{Dis}}$ 
is estimated using the plug-in estimators $\widehat p_{a,y} =n_{a,y}/n$ in \eqref{eq:disparity level expression}.
 We then follow the definition of $\td$ in \eqref{eq:td} to estimate it by
$$\widehat{t}_{\textup{Dis}}(\delta) = \argmin_t \{|t|: |\widehat{\textup{Dis}}(t)|\le \delta\}.$$
Since $\widehat{\textup{Dis}}(t)$ is monotone non-increasing as a function of $t$,
$\widehat{t}_{\textup{Dis}}(\delta)$ can be estimated via the bisection method, as summarized in Algorithm \ref{alg:FPIR}.
Our final FPIR estimator of the fair Bayes-optimal classifier is $\widehat{f}^{\textup{FPIR}}_{\widehat{t}_{\textup{Dis}}(\delta)}$.

{
We also note the recent findings of \cite{cruz2024unprocessing}, which suggest that post-processing methods can achieve better fair Pareto frontiers compared to pre- and in-processing approaches. While this might imply that post-processing methods alone are sufficient, pre- and in-processing approaches play distinct and valuable roles, making relevant techniques such as resampling and reweighting
essential components of the broader fairness toolbox. 
One of the main contributions of our work is to provide such tools via a unified framework for characterizing the $\delta$-fair Bayes optimal classifier and for tracing the fair Pareto frontier.}



\section{Empirical Experiments}
\label{sim}
\subsection{Synthetic Datasets}
\label{synth}
In this section, we conduct simulation studies to illustrate the numerical performance of our proposed methods. We consider 
a label-conditional normal distribution, for which the fair Bayes-optimal classifier has a simple closed form.

\textbf{Data-generating process.}
For a positive integer dimension $p$, let $X=(X_1,X_2,...,X_p)^\top\in \mathbb{R}^p$ be the usual features, 
$A\in\{0,1\}$ be the protected attribute and $Y\in\{0,1\}$ be the label. 
Recalling the notations from the end of Section \ref{Pre_and _Not}, we generate $A$ and $Y$ according to the probabilities $p_{1,1}=0.49$, $p_{1,0}=0.21$ $p_{0,1}=0.12$ and $p_{0,0}=0.18$. 
Conditional on $A=a$ and $Y=y$, we generate $X$ from a multivariate
Gaussian distribution $\N(\mu_{a,y},\sigma^2I_p)$, where $I_p$ is the $p$-dimensional identity covariance matrix, and $\sigma^2$ controls the variability  of the feature entries. 
The entries of $\mu_{a,y}$ are sampled from $\mu_{ay,j}\sim \mathrm{Unif} (0,1)$, $j=1,\ldots,p$, where $\mathrm{Unif} (0,1)$ is the uniform distribution over $[0,1]$.  In this model, $\eta_a$ has a closed form, 
and we can use it to find the fair Bayes-optimal classifier, see \Cref{exp_det}.
Intriguingly, under this model, group-wise thresholding rules are linear in  $x$.

\begin{table}[b]
\caption{Numerical simulation results showing classification accuracies and levels of disparity of the true fair Bayes-optimal classifier and the three proposed methods (FUDS, FCSC, FPIR) using logistic regression. 
The reported results are averages over 5,000 test data points, with standard deviations shown in parentheses.
\vspace{-0.5cm}}
\label{table_ddp}
\vskip -0.2in
\begin{center}
\setlength{\tabcolsep}{5.1pt}
\renewcommand{\arraystretch}{1.15}
\begin{small}
\begin{sc}
\resizebox{\textwidth}{!}{%
\begin{tabular}{cc|cc|cc|cc}
\hline
\multicolumn{8}{c}{Demographic Parity}   \\\hline
\multicolumn{2}{c|}{Theoretical}  & \multicolumn{2}{c|}{FUDS} &
\multicolumn{2}{c|}{FCSC}& \multicolumn{2}{c}{FPIR} \\ \hline
$\delta$& $\text{ACC}$ & DD & $\text{ACC}$& DD & $\text{ACC}$ & DD & $\text{ACC}$ \\ 
\hline
0.00 & 0.727 & 0.013 (0.010) & 0.722 (0.007) & 0.013 (0.010) & 0.722 (0.007) & 0.013  (0.010) & 0.722  (0.007) \\
0.10 & 0.736 & 0.100 (0.017) & 0.735 (0.007) & 0.100 (0.017) & 0.736 (0.007) & 0.100  (0.017) & 0.736  (0.007) \\ 
0.20 & 0.746 & 0.200 (0.018) & 0.746 (0.007) & 0.200 (0.018) & 0.746 (0.006) & 0.200  (0.018) & 0.746  (0.006)\\
0.30 & 0.753 & 0.300 (0.017) & 0.753 (0.006) & 0.300 (0.017) & 0.753 (0.006) & 0.300  (0.017) & 0.753  (0.006)\\\hline
\multicolumn{8}{c}{Equality of Opportunity}   \\\hline
\multicolumn{2}{c|}{Theoretical}  & \multicolumn{2}{c|}{FUDS} &
\multicolumn{2}{c|}{FCSC}& \multicolumn{2}{c}{FPIR} \\ \hline
$\delta$& $\text{ACC}$ & DD & $\text{ACC}$& DD & $\text{ACC}$ & DD & $\text{ACC}$ \\ 
\hline
0.00 & 0.731 & 0.017 (0.013) & 0.730 (0.007) & 0.017 (0.013) & 0.731 (0.007) & 0.017  (0.013) & 0.731  (0.007) \\
0.10 & 0.745 & 0.101 (0.022) & 0.744 (0.007) & 0.101 (0.022) & 0.744 (0.007) & 0.101  (0.023) & 0.744  (0.007) \\ 
0.20 & 0.753 & 0.201 (0.024) & 0.752 (0.006) & 0.202 (0.024) & 0.753 (0.006) & 0.202  (0.024) & 0.753  (0.006)\\
0.30 & 0.757 & 0.302 (0.025) & 0.756 (0.006) & 0.302 (0.025) & 0.757 (0.006) & 0.302  (0.025) & 0.757  (0.006)\\\hline
\multicolumn{8}{c}{Predictive Equality}   \\\hline
\multicolumn{2}{c|}{Theoretical}  & \multicolumn{2}{c|}{FUDS} &
\multicolumn{2}{c|}{FCSC}& \multicolumn{2}{c}{FPIR} \\ \hline
$\delta$& $\text{ACC}$ & DD & $\text{ACC}$& DD & $\text{ACC}$ & DD & $\text{ACC}$ \\ 
\hline
0.00 & 0.743 & 0.020 (0.015) & 0.743 (0.006) & 0.020 (0.015) & 0.743 (0.006) & 0.020  (0.015) & 0.743  (0.007) \\
0.10 & 0.751 & 0.099 (0.025) & 0.750 (0.006) & 0.099 (0.025) & 0.750 (0.006) & 0.099  (0.025) & 0.750  (0.006) \\ 
0.20 & 0.756 & 0.199 (0.024) & 0.755 (0.006) & 0.199 (0.025) & 0.755 (0.006) & 0.199  (0.025) & 0.755  (0.006)\\
0.30 & 0.758 & 0.297 (0.023) & 0.757 (0.006) & 0.296 (0.024) & 0.757 (0.006) & 0.297  (0.024) & 0.757  (0.006)\\\hline
\end{tabular}
}
\end{sc}
\end{small}
\end{center}
\vskip -0.2in
\end{table}

 \textbf{Experimental setting.} To evaluate our FUDS, FCSC and FPIR algorithms, we randomly generate $10,000$ training data points  and $5,000$ test data points.
Since the group-wise thresholding rules are linear in $x$, we use logistic regression in our methods.
All experiments are conducted in python and we use the scikit-learn package to train the logistic regression model.
This package allows us to specify label weights for cost-sensitive classification, and thus works seamlessly for FCSC.
 In our experiments, we consider four fairness metrics: 
 demographic parity, equality of opportunity, predictive equality and overall accuracy equality.

We  first evaluate the three algorithms with various pre-determined levels of disparity. We present the simulation results in Table \ref{tab1}. We observe that FPIR controls the disparity level at the pre-determined value, as desired. 
 We further present the 
 fairness-accuracy tradeoff curve of fair Bayes-optimal classifiers (the fair Pareto frontier from \Cref{fro}) and 
 our three methods in Figure \ref{plt-synthetic}. Our classifiers closely track the fair Pareto frontier.

\begin{figure}[t]
\begin{center}
\centerline{\includegraphics[width=1\columnwidth]{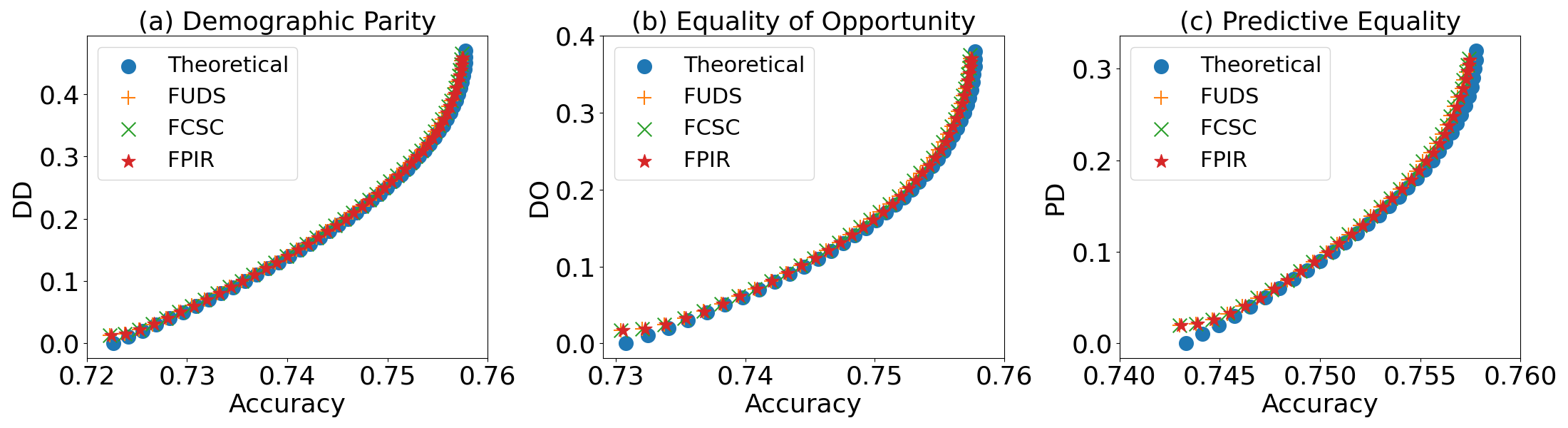}} \caption{Numerical simulation results showing the fairness-accuracy tradeoff the fair Bayes-optimal classifier (fair Pareto frontier) and the three proposed methods (FUDS, FCSC, FPIR) using logistic regression.}
 \label{plt-synthetic}
 \end{center}
  \vskip -0.4in
 \end{figure}

\subsection{{Empirical Data Analysis}}
To further illustrate our proposed methods, 
we compare our method with strong baseline methods on standard datasets.

{\bf Data Description.} 
We consider four benchmark datasets: AdultCensus, COMPAS, LawSchool, and ACSIncome. For the first three datasets, we follow the pre-processing pipeline introduced in \cite{CHS2020}, while for ACSIncome, we adopt the pipeline from \cite{xian2024unified}.

\begin{compactitem}[]
\item {{(1) AdultCensus}}: In the AdultCensus dataset, the target variable $Y$ is whether the income of an individual is more than \$50,000.  Age, marriage status, education level, and other related variables are included in $X$, and the protected attribute $A$ refers to gender. 

\item {{(2) COMPAS}}: In  the COMPAS dataset, $Y$  indicates whether or not a criminal will re-offend. Here $X$ includes  prior criminal records, age, and an indicator of misdemeanor.  The protected attribute $A$  is the race of an individual, ``white-vs-non-white''.

\item {{(3) LawSchool}}: The target variable $Y$ in the LawSchool dataset is whether a student gets admitted to law school. 
Here $X$ includes the LSAT scores, undergraduate GPA and more. The protected attribute $A$ we consider is the race, ``white-vs-non-white''.

\item {{{(4) ACSIncome}}}: {This dataset extends the AdultCensus dataset by incorporating more recent and comprehensive demographic and economic information from the American Community Survey (\cite{ding2021retiring}). The target variable \( Y \) indicates whether an individual's total income exceeds \$50{,}000. The covariates \( X \) include age, education level, employment status, hours worked per week, occupation, and marital status. The protected attribute \( A \) is gender.}

\end{compactitem}
\vspace{0.10in}
{\bf Algorithms Considered.}
We select ten methods from the literature, comprising of 3 pre-processing, 4 in-processing, and 3 post-processing methods.

\begin{itemize}[]

\item (1) Disparate Impact Remover (DIR, \cite{FFMS2015}):

DIR is a pre-processing method that aims to modify the value of $Y$ to $\widetilde{Y}$,  ensuring that the probability of  $\widetilde{Y}=1$ is equal across different protected groups, specifically, $\P(\widetilde{Y}=1|A=1)=\P(\widetilde{Y}=1|A=0)$. The key idea  of DIR is to align the cumulative distribution functions of 
$\eta^Y_{A=1}(X)$ and $\eta^Y_{A=0}(X)$  while 
maximizing similarity between the transformed and original datasets.

\item (2) Fairness-Aware Oversampling (FAWOS, \cite{salazar2021fawos}):

FAWOS is a pre-processing method that uses data augmentation to achieve fairness. Unlike our FUDS method that adjusts the sizes of all groups, FAWOS applies SMOTE \citep{chawla2002smote}--—a popular synthetic data augmentation method for unbalanced classification problems---to increase the number of data points in the positive unprivileged group ($A=0, Y=1$). The number of data points generated is
$N = \alpha_F \times \left(
n_{11} n_{00}/n_{10} - n_{01}\right)$, with $\alpha_F$ tuned to control the accuracy-fairness tradeoff.

\item (3) {Adaptive Reweighing Algorithm (ARA, \cite{chai2022fairness})}

{Adaptive reweighing
is a pre-processing technique that aims to induce fairness 
via finding weights for
a cost-sensitive classification problem.
It proceeds by iteratively (1) finding individual-level
weights by solving a $\ell_2$-regularized quadratic programming (QP) problem for each protected subgroup, and (2) fitting a classifier with the resulting weights via cost-sensitive classification.
Although fairness constraints are not directly encoded in the iterative algorithm, fairness is encouraged by choosing an appropriate level of normalization through a hyperparameter in the QP step. }

\item (4) KDE-based constrained optimization (KDE, \cite{CHS2020}):

KDE-based optimization is an in-processing technique that uses kernel density estimation (KDE) to approximate fairness constraints. The proposed estimator of disparity is a differentiable function with respect to the model parameters. 
KDE-based optimization uses this estimated disparity as a regularizer for empirical risk minimization.

\item (5) Adversarial training (ADV, \cite{ZLM2018}):

Adversarial training is an in-processing method that involves simultaneously training two models: (1) a classifier to predict the label and (2) an adversary that attempts to predict protected attributes from the classifier's outputs. 
This dual-training process encourages the classifier to learn predictable representations that are devoid of biases related to protected attributes, thereby promoting fairness in its predictions.

\item (6) {Reductions (RED, \cite{agarwal2018reductions})}

{Reductions is an iterative in-processing algorithm.
At each round, it trains a cost-sensitive classifier that balances error and fairness-violation penalties (using current Lagrange multipliers), and updates those multipliers 
using the observed constraint slack.
Ultimately, it returns a randomized ensemble of the resulting classifiers.}

\item (7) {MinDiff (MD, \cite{prost2019toward})}

{MinDiff is an in-processing algorithm that adds a fairness regularization term to the original loss function. The regularizer minimizes the Maximum Mean Discrepancy (\cite{gretton2006kernel}) between the distributions of score functions of the binary protected attribute subpopulations using an appropriate universal kernel (e.g. Gaussian or Laplace).}

\item (8) Post-processing through flipping (PPF, \cite{chen2023post}):

\cite{chen2023post} showed
that a Bayes-optimal fair classifier can be obtained by flipping the output of the unconstrained classifier and proved that the flipping probability satisfies an estimation equation. Based on these observations, they designed a post-processing method that first estimates the flipping probability of each output and then estimates the fair Bayes-optimal classifier.

\item (9) Post-processing through optimal transport (PPOT, \cite{xian2023fair}):

\cite{xian2023fair} proved that the fair classification problem with bounded demographic parity is equivalent to a  Wasserstein-barycenter problem, and the fair Bayes-optimal classifier is given by the composition of the Bayes-optimal score function 
$\eta_a$ 
and the optimal transport  map from the Wasserstein-barycenter problem. 
They proposed a post-processing algorithm 
to estimate 
the fair Bayes-optimal classifier from the estimated score functions $\widehat{\eta}_a$.

\item (10) {Group Fairness Framework for Post-processing Everything (FRAPPÉ, \cite{tifrea2023frappe})}

{FRAPPÉ is a post-processing method that converts any regularized in-processing fairness approach into a plug-in post-processing adjustment by learning a lightweight additive correction module. 
It is based on an optimization problem that minimizes a divergence between the base and adjusted model outputs on data without group attributes, while also incorporating the original in-processing fairness regularizer on protected-attribute data.}

\end{itemize}

 \textbf{{Experimental Setting:}}
 We follow the training settings from \cite{CHS2020} and \cite{xian2024unified}. With the exception of the pre-processing method ARA \citep{chai2022fairness} across all datasets, and our in-processing method FCSC and the in-processing method RED \citep{agarwal2018reductions} on the AdultCensus dataset, all methods across all datasets use a three-layer fully connected neural network and are trained with the Adam optimizer using $(\beta_1, \beta_2) = (0.9, 0.999)$, the default hyperparameters. For the AdultCensus, COMPAS, and LawSchool datasets, each hidden layer contains 32 neurons; for the ACSIncome dataset, we use the same optimizer configuration but adopt a larger architecture of with hidden layers of sizes 500, 200, and 100. The ARA method is trained using a simple logistic regression model. The FCSC and RED method on the AdultCensus dataset is trained using a histogram-based gradient boosting tree configured with 300 trees, a maximum depth of 4, a learning rate of $0.05$, and an $\ell_2$ regularization strength of $1.0$.\footnote{The other in-processing baseline methods-- ADV \citep{ZLM2018}, MinDiff \citep{prost2019toward}, and KDE \citep{CHS2020} -- are not easily compatible with gradient boosting trees as base classifiers, since these approaches rely on adding a fairness regularization term to the loss function and require backpropagation through the classifier.}

For adversarial training \citep{ZLM2018}, we use a linear classifier as the discriminator. For FRAPPÉ \citep{tifrea2023frappe}, the post-hoc module builds off the KDE-based fairness regularizer of \cite{CHS2020}, using a single-layer perceptron with 128 hidden nodes for all datasets. The training batch size, training epochs, starting learning rate, and learning rate decay factor for the four datasets are summarized in Table \ref{detail}. 

Let $n_{\textup{epoch}}$ represent the total number of training epochs. For the FUDS and FCSC methods, we initially pre-train the model by training the neural network for $n_{\textup{epoch}}/4$ epochs. Subsequently, we apply the bisection method as described in Algorithms \ref{alg:FUDS} and \ref{alg:FCSC} to update $t$ after every $n_{\textup{epoch}}/20$ epochs. Through this approach, we update $t$ a total of 15 times, resulting in a tolerance level of $2^{-15}\sim 3\times 10^{-5}$.
For adversarial training, we first pre-train the classifier over $n_{\textup{epoch}}/4$ epochs and then update the weights of both the classifier and discriminator in the later $3n_{\textup{epoch}}/4$ epochs.

We repeat each experiment 50 times with different random seeds for the AdultCensus, COMPAS, and LawSchool datasets, and 10 times for ACSIncome.
For each dataset and random seed, we randomly split the 
data into training and test sets (with a 70\%-30\% split). 
As a result, the randomness of the experiments arises from the stochasticity of the train-test set split, the neural network initialization, and the minibatch selection during the optimization.

\begin{table}[t]
\caption{Training details for the empirical datasets.}
\label{detail}
\begin{center}
\begin{sc}
\begin{tabular}{c|cccc}
\hline
  Dataset &Batch  & Training   &Starting & Learning Rate \\
   & size &  Epochs  & Learning Rate & Decay Factor\\\hline
 AdultCensus &512 & 200    &1e-1 & 0.98\\
 COMPAS &2048 & 500    &5e-4 & None\\
LawSchool &2048 & 200 &   2e-4  & None \\
ACSIncome& 128 & 20 & 1e-3 & None \\ \hline
\end{tabular}
\end{sc}
\end{center}
\end{table}

\textbf{{Computational complexity:}} As outlined above, for our pre- and in-processing algorithms, we iteratively update the classifier rather than train an entirely new classifier at each iteration, so the total 
number of (stochastic) gradient steps used for estimating the $\delta$-fair Bayes optimal classifier is \[
    \mathcal{O} \lsb \left(n_{\textup{epoch}}^{\textup{train}}\cdot \frac{N}{B}\right) +\lsb \left(n_\textup{epoch}^{\textup{iter}}\cdot \frac{N}{B}\right) \cdot \log(1/\epsilon) + N\log(1/\epsilon)\rsb \rsb 
\]
where $n_{\textup{epoch}}^{\textup{train}}$ and $n_{\textup{epoch}}^{\textup{iter}}$ are the number of epochs used to train the unconstrained classifier and iteratively update the classifier, respectively. Here, $N$ is the size of the dataset, $B$ is the batch size and $\epsilon$ is the tolerance level. The computational complexity of the post-processing algorithm consists of only the latter term $\mathcal{O}(N\log(1/\epsilon))$, since we assume access to a pre-trained classifier. This term specifies the complexity associated with repeatedly evaluating the disparity level for the trained classifier at each iteration of the bisection search. In particular, across all algorithms, we set $\epsilon=2^{-15}$, resulting in at most $15$ updates to the parameter $t$ via the bisection search. 

These computational complexities are similar to 
those of
state of the art methods.
For example, the reductions method of \cite{agarwal2018reductions} has a runtime of $\mathcal{O}\lsb (n_{\textup{epoch}}^{\textup{iter}} \cdot {N}/{B})\cdot 1/\nu^2\rsb$ (where $\nu$ is the $\nu$-approximate saddle point per their Theorem 1, $\nu$ typically $\approx0.01$). 
For post-processing, FPIR requires a computational runtime of $\mathcal{O}(N\log(1/\epsilon))$ where for $\epsilon = 2^{-15}$, $1/\epsilon = 32768$.
This is comparable
to our sample sizes
$N$, which
are typically on the order of thousands to tens of thousands, so this complexity matches that of \cite{chen2023post}.

 \textbf{Simulation Results for Our Methods:}
We first evaluate the FUDS, FCSC, and FPIR algorithms with various pre-determined levels of disparity. We present the simulation results in Table \ref{tab1}. We observe that the proposed three methods control the disparity level at the pre-determined values, as desired.
We then compare the fairness-accuracy tradeoff of the proposed methods to the baseline.

\begin{table}[b]
 \caption{Empirical data results showing the DD of FUDS, FCSC and FPIR with predetermined unfairness levels.}
 \label{tab1}
 \vskip -0.1in
 \begin{center}
 \begin{small}
 \begin{sc}
 \begin{tabular}{l|c|ccccc}\hline
 
\multicolumn{7}{c}{AdultCensus}  \\\hline
Methods & $\delta$ & 0.00 & 0.04 &0.08 & 0.12 & 0.16\\\hline
FUDS &    &  0.009 (0.007) & 0.037 (0.020) & 0.080 (0.012) & 0.123 (0.016) & 0.152 (0.016) \\
 FCSC & DD &  0.017 (0.019) & 0.037 (0.011) & 0.081 (0.012) & 0.123 (0.011) & 0.159 (0.012) \\
 FPIR &    &  0.007 (0.006) & 0.040 (0.010) & 0.080 (0.019) & 0.121 (0.020) & 0.159 (0.015) \\\hline

\multicolumn{7}{c}{COMPAS}  \\\hline
Methods & $\delta$ & 0.00 & 0.06 &0.12 & 0.18 & 0.24\\\hline
FUDS &    &  0.039 (0.026) & 0.053 (0.031) & 0.106 (0.039) & 0.166 (0.039) & 0.218 (0.033) \\
 FCSC & DD &  0.033 (0.025) & 0.060 (0.032) & 0.115 (0.039) & 0.169 (0.039) & 0.223 (0.036) \\
 FPIR &    &  0.028 (0.023) & 0.061 (0.028) & 0.117 (0.038) & 0.173 (0.037) & 0.227 (0.033) \\\hline
 
\multicolumn{7}{c}{Lawschool}  \\\hline
Methods & $\delta$ & 0.000 & 0.0016 &0.032 & 0.048 & 0.064\\\hline
FUDS &    &  0.006 (0.006) & 0.015 (0.007) & 0.031 (0.007) & 0.045 (0.008) & 0.061 (0.006) \\
 FCSC & DD &  0.007 (0.005) & 0.017 (0.007) & 0.032 (0.008) & 0.048 (0.007) & 0.062 (0.007) \\
 FPIR &    &  0.004 (0.003) & 0.016 (0.006) & 0.032 (0.005) & 0.047 (0.005) & 0.063 (0.005) \\\hline
 
\multicolumn{7}{c}{ACSIncome}  \\\hline
Methods & $\delta$ & 0.000 & 0.050 &0.100 & 0.150 & 0.200\\\hline
FUDS &    &  0.007 (0.005) & 0.047 (0.012) & 0.102 (0.015) & 0.143 (0.007) & 0.179 (0.010) \\
 FCSC & DD &  0.012 (0.010) & 0.050 (0.008) & 0.100 (0.009) & 0.148 (0.012) & 0.185 (0.011) \\
 FPIR &    &  0.001 (0.001) & 0.050 (0.001) & 0.100 (0.001) & 0.150 (0.001) & 0.186 (0.010) \\\hline
\end{tabular}
 \end{sc}
\end{small}
 \end{center}
 \vskip -0.1in
 \end{table}

 \textbf{Experimental Setting for Controlling Disparity:}
In our proposed methods—FUDS, FCSC, and FPIR—along with the two other post-processing methods, the level of disparity is  directly controlled. We set the range from  zero  to the empirical DD of the unconstrained classifier.

For the DIR method, as suggested by \cite{FFMS2015}, a ``partial remover'' tool is employed. 
This tool features a parameter $\lambda_{\textup{DIR}}$, which varies from zero  (indicating the original dataset) to $1$ (indicating perfect distribution alignment between $\eta^Y_{A=1}(X)$ and $\eta^Y_{A=0}(X)$). We adjust $\lambda_{\textup{DIR}}$  from $0.05$ to $0.95$ to explore the fairness-accuracy tradeoff.

In FAWOS, we vary $\alpha_{\textup{FAWOS}}$ from zero to two. Higher values of $\alpha_{F}$ lead to an excess of data points in the positive unprivileged group, potentially causing reverse discrimination. {We omit the result of this method on the ACSIncome dataset, as FAWOS requires computing nearest neighbors for each data point which 
was extremely slow in our
experiments
for the ACSIncome dataset ($\approx$1.66 million data points).}

{For ARA, we vary $\alpha_{\textup{ARA}}$ from 0.40 to 0.50 on the COMPAS dataset. Higher values of $\alpha_{\textup{ARA}}$ correspond to more uniform reweighting, which results in reduced fairness. 
For the other datasets, we were unable to meaningfully decrease demographic disparity.\footnote{
To address this, we contacted the authors and explored several potential remedies, including alternative classifier model spaces and adjustments to hyperparameters within the QP subproblem, but still could not achieve a significant reduction in the disparity.}}

{Both KDE-based constrained optimization and MinDiff balance fairness and accuracy by adjusting a tuning parameter that trades off empirical loss against fairness regularization. We vary the tuning parameter $\lambda_{\textup{KDE}}$ from $0.05$ to $0.95$, and $\lambda_{\textup{MinDiff}}$ from $0.0$ to $10.0$.}

In adversarial training, the tradeoff is controlled by altering the parameter $\alpha_{\textup{ADV}}$ that controls the gradient of the discriminator. We vary this parameter from zero  to three, as we empirically find that, similar to FAWOS, a larger $\alpha_{\textup{ADV}}$ would result in reverse discrimination. 

\begin{figure}[h!]
\centering
\resizebox{0.8\textwidth}{!}{%
    \begin{minipage}{\textwidth}
        \centering
        Panel (A): AdultCensus Dataset\\
        \includegraphics[width=1\textwidth]{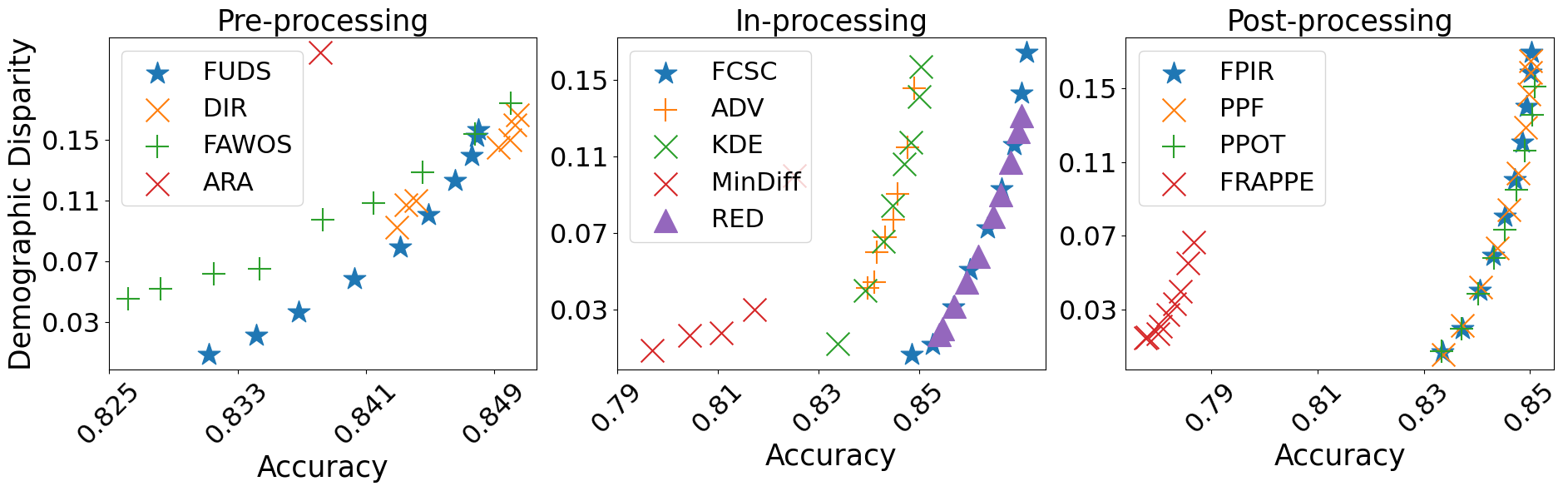}\\[1em]

        Panel (B): COMPAS Dataset\\
        \includegraphics[width=1\textwidth]{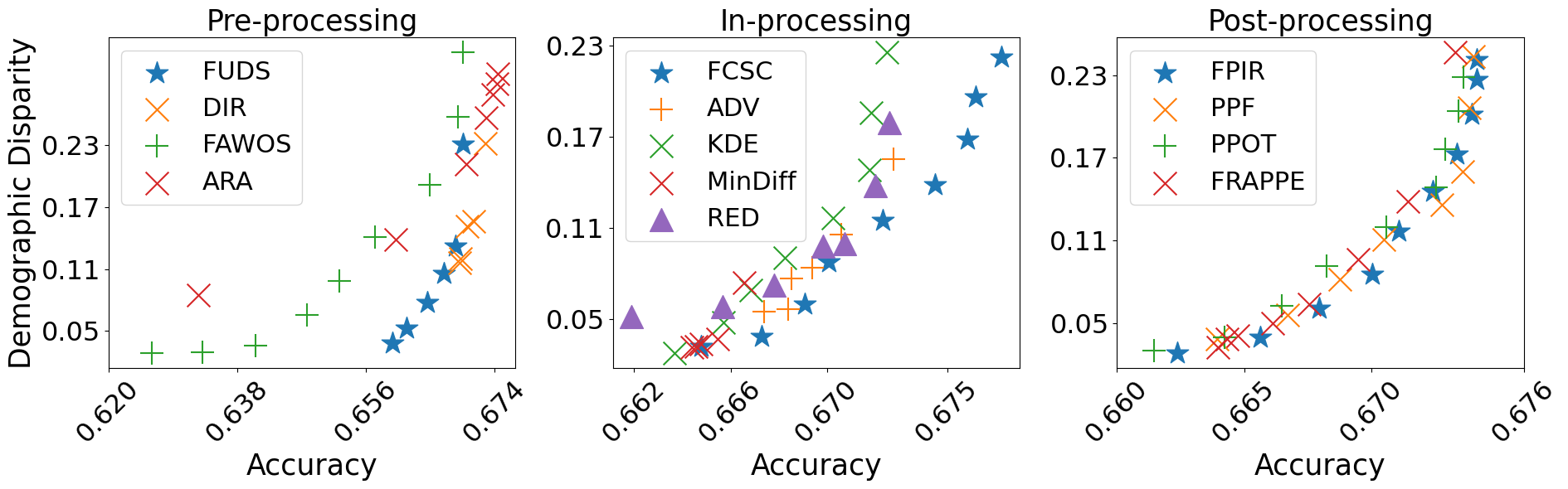}\\[1em]

        Panel (C): LawSchool Dataset\\
        \includegraphics[width=1\textwidth]{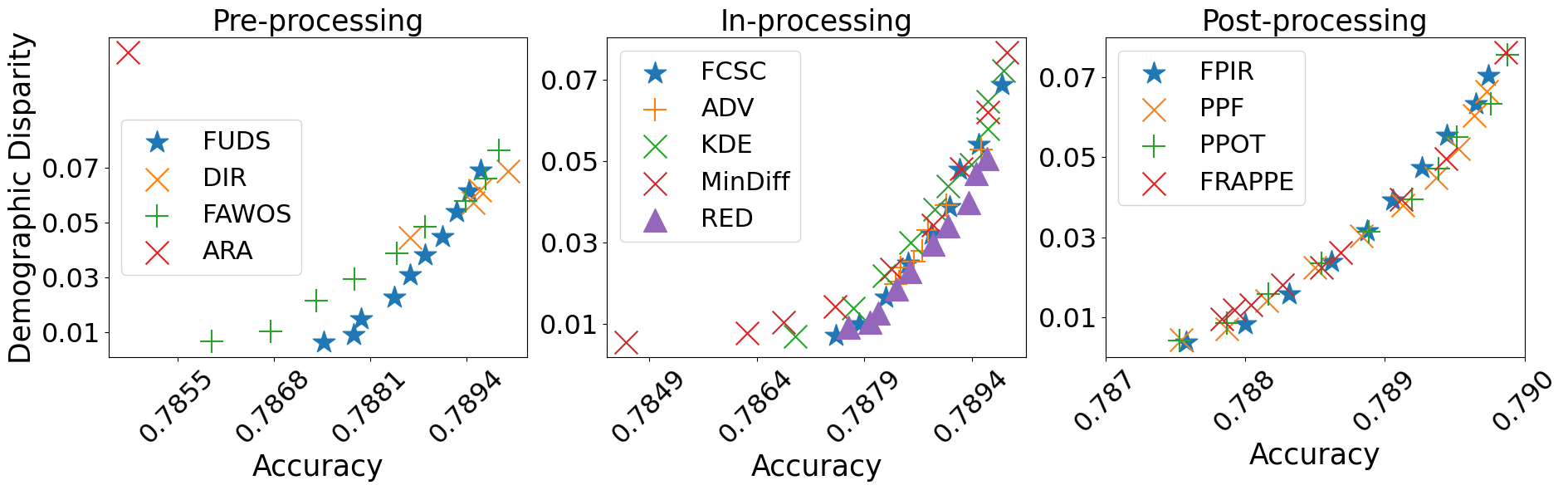}\\[1em]

        Panel (D): ACSIncome Dataset\\
        \includegraphics[width=1\textwidth]{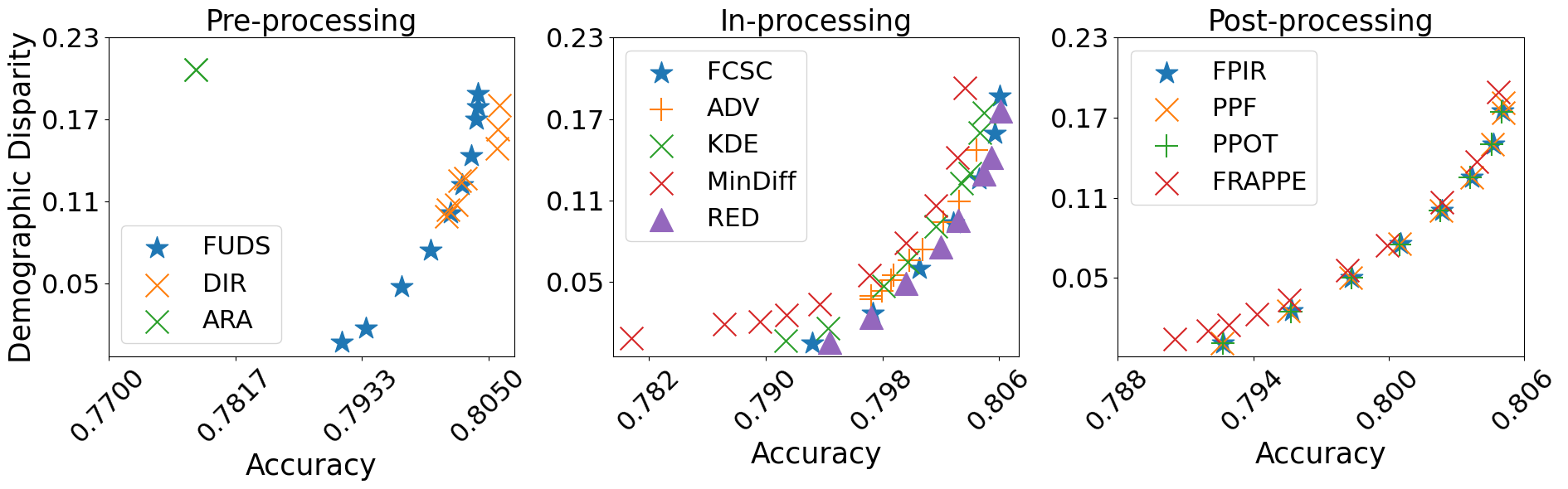}
    \end{minipage}
}
\caption{Fairness-accuracy tradeoff on the AdultCensus (Panel (A)), COMPAS (Panel (B)), LawSchool (Panel (C)), and ACSIncome (Panel (D)) datasets.}
\label{plt-tradeoff}
\vskip -0.2in
\end{figure}

  \textbf{Simulation Results for Controlling Disparity:}
Figure \ref{plt-tradeoff} presents the fairness-accuracy tradeoff with respect to Demographic Disparity (DD) evaluated across four datasets. In the plot, each point corresponds to a specific tuning parameter. 

{{\bf Pre-processing.} Among the three pre-processing methods, our FUDS algorithm exhibits the most favorable fairness-accuracy tradeoff and allows for direct control of disparity.} 

{
In contrast, while the DIR method shows satisfactory accuracy, it falls short in achieving perfect fairness. This is because, although DIR ensures equal probabilities of positive outcomes across different groups in the transformed dataset---i.e., $\P(\tilde{Y}=1|A=1)=\P(\tilde{Y}=1|A=0)$---this does not automatically ensure that a fair model is learned. {This shortcoming is readily apparent in the results of DIR on the ACSIncome dataset, where it fails to substantially reduce the disparity level.} 
The third method, FAWOS, effectively manages disparity, but its heuristic-based sampling strategy strays from the ideal proportion we identified. As a result, compared to FUDS, FAWOS compromises more on accuracy at all levels of disparity.}

{\bf In-processing.} {For the three in-processing methods, FCSC outperforms all  baselines on the COMPAS dataset and demonstrates comparable or better results on the other three datasets. The MinDiff method of \cite{prost2019toward} performs equally as well as FCSC and the other baselines (KDE and ADV) on the Lawschool and ACSIncome datasets, but has trouble tracing a reasonable Pareto frontier on AdultCensus.}

{The KDE method experiences a sudden drop in accuracy when the disparity level approaches zero. 
This issue may caused by its use of a Huber surrogate loss to address the non-differentiability of the absolute value function at $\delta = 0$. Compared with FCSC and KDE-based optimization, adversarial training exhibits inferior performance and fails to achieve near-perfect fairness. 
This shortcoming could be attributed to the challenges associated with minimax optimization, which often leads to unstable training processes.}

{Finally, the RED method traces a Pareto optimal frontier that is comparable with our FCSC method across all datasets, with the exception of COMPAS, for which it performs slightly worse across all levels of disparity. Similar to FCSC, RED can leverage a gradient boosting tree as its base classifier, thereby enabling strong performance on the AdultCensus dataset where such a model is more accurate.}

{{\bf Post-processing.} 
{Finally, with the exception of FRAPPÉ on the AdultCensus dataset, we observe that the four post-processing methods present comparable performance in balancing fairness and accuracy.} This similarity arises because they 
all aim to estimate the fair Bayes-optimal classifier. 
However, FPIR can be viewed as a more direct approach.}

{\textbf{Computational Runtimes.}} 
{We also report runtimes measuring the average execution time of a single run per seed for the COMPAS and ACSIncome datasets (see Figure \ref{computational-runtime}). All experiments are conducted on a computing cluster equipped with Intel Xeon Platinum 8375C CPUs $@$ 2.90GHz processors. For the COMPAS dataset, each run is allowed to use up to 10 CPU cores with 0.8 GB of RAM, while for ACSIncome, each run uses a single core with access of up to 32 GB of RAM.}

{The pre-processing method termed ARA of \cite{chai2022fairness} runs much quicker than its pre-processing counterparts since the classifier is an ordinary logistic regression estimator. However, ARA achieves a sub-optimal Pareto frontier on COMPAS and fails to trace any meaningful frontier for the other datasets we considered.}

{The post-processing methods do not include the runtime of the pre-trained classifiers which the post-hoc fairness corrections rely upon. 
The time required to obtain
this unconstrained classifier---that does not impose any fairness constraints or fairness algorithms---is denoted as \textit{BASE}.}

{Our method's runtime is either comparable or slightly better than that of competing methods with Pareto-optimal performance across all types of algorithms (pre-, in-, and post-processing).}


\begin{figure}[b]
\centering
\resizebox{0.66\textwidth}{!}{%
    \begin{minipage}{\textwidth}
        \centering
        Panel (A): COMPAS Dataset\\
        \includegraphics[width=\textwidth]{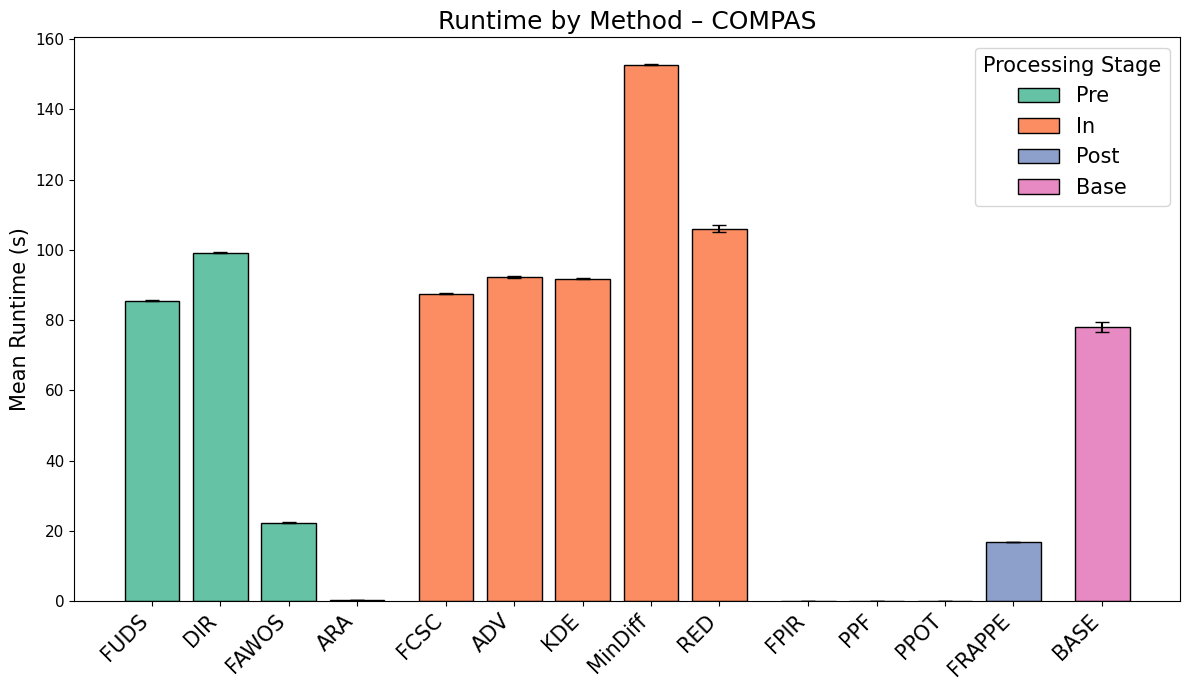}\\[1em]
        Panel (B): ACSIncome Dataset\\
        \includegraphics[width=\textwidth]{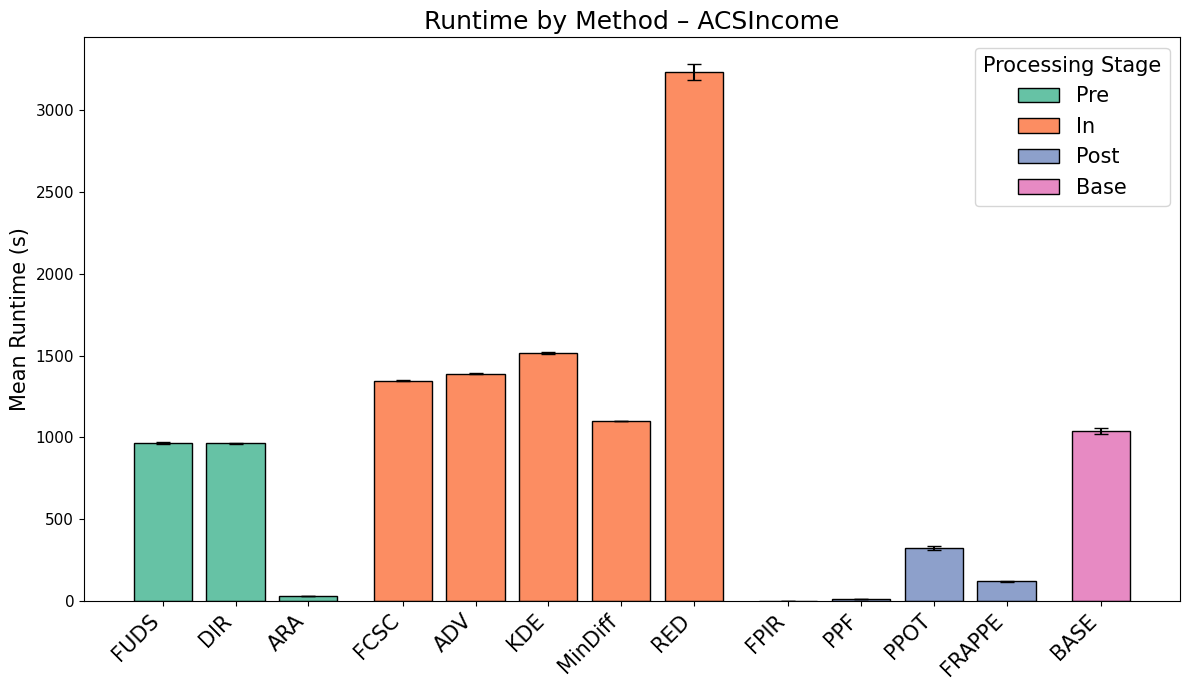}
    \end{minipage}
}
\caption{Computational runtimes on the COMPAS (Panel (A)) and ACSIncome (Panel (B)) datasets. \textit{BASE} records the runtime of training an unconstrained classifier, i.e.~one with no fairness constraints imposed or fairness algorithm implemented. This unconstrained classifier is used for post-processing.}
\label{computational-runtime}
\vskip -0.2in
\end{figure}

{
In summary, our methods have state-of-the-art performance compared to existing algorithms.
In particular, our pre-processing method can a reach higher accuracy than prior pre-processing methods at low disparity levels.}

\section{Summary and Discussion}
\label{sec:dis}
In this paper, we study the statistical and methodological underpinnings of fair classification problems.
We establish a unified framework for deriving fair Bayes-optimal classifiers. 
We introduce the notion of linear and bilinear disparity measures, 
and find Bayes-optimal classifiers under these constraints.
Our framework can be applied in a wide range of scenarios, including various disparity measures, cost-sensitive loss functions, and multiple fairness constraints while allowing for direct control of the disparity level.
Additionally, we study the Pareto frontier for 
fair classification and prove that the tradeoff function is convex for linear group fairness measures.

Based on our theoretical results, we further design pre-, in-, and post-processing methods that handle algorithmic bias. 
The proposed methods aim to recover the optimal fairness-accuracy tradeoff. Moreover, they allow for direct control of the disparity level.  

Our work suggests several promising directions for further theoretical and algorithmic development. From the theoretical
perspective, the fair Bayes-optimal classifier only characterizes the best solution at the population level, and the finite-sample properties need to be characterized. 
From the algorithmic perspective, it is of interest to improve our  algorithms. 
For example, for our pre-processing method, we only considered fair sampling. However, with the desired proportions in hand, we can generate fair synthetic data using more advanced conditional generative models, such as conditional GANs and conditional diffusion models.

{Recent work by \cite{xian2024unified} presents results for multi-class optimal fair classifiers via post-processing.
It would be of interest, but it does not seem completely straightforward, to adapt
our in- and pre-processing algorithms to this setting.
}

Furthermore, our analysis 
supposes the availability of protected attributes during training. 
However, the collection and use of protected information is often restricted. For instance, under the General Data Protection Regulations in the European Union, acquiring protected personal information requires consent from individuals.
This motivates studying methods for protected attributes only partially available during training. 
Recently, \cite{chai2022fairness} proposed addressing this challenge through knowledge distillation, which generates soft labels for protected attributes. Exploring the integration of this technique into our framework presents an intriguing avenue for future research.

\section*{Acknowledgements}

This work was supported in part by ARO W911NF-23-1-0296, NSF 2031895, NSF 2046874, NSF 2247795, ONR N00014-21-1-2843, ONR N00014-22-1-2680, and the Sloan Foundation. {We are grateful to Junyi Chai and Xiaoqian Wang for insightful discussions regarding their work on the ARA method, and to Ruicheng Xian for his assistance with the implementation of the reductions approach.}

{\small
\setlength{\bibsep}{0.2pt plus 0.3ex}
\bibliographystyle{plainnat-abbrev}
\bibliography{main}
}



\appendix

\section*{Appendix}

{\bf Additional notation and conventions.}
In this appendix, we use some additional notation.
For a real-valued function $f$ defined on $[a,b)$ for some $a<b$, we denote by $\lim_{x\to a^+} f(x)$ the limit from the right of $f$ at $a$, if it exists. Similarly, if $f$ is defined on $(b,a]$ for $b<a$,  we denote by $\lim_{x\to a^-} f(x)$ the limit from the left of $f$ at $a$, if it exists. 
For an interval $[a,b]$, and scalars $c\in \R$, $d>0$, we denote $c+ d[a,b] =[c+ da,c+ db] $.
For a classifier $f$, we denote
$\textup{Acc}(f)=1-R(f)$.
When needed, we define $0/0:=0$.

\section{Proofs}

In Sections \ref{sec:lemmas} to \ref{pfalg}, we present the proofs of our theoretical results from the main text, except for \Cref{thm-fb-eq-odd} and \Cref{thm-opt-dp-multi}. We provide an independent \Cref{ext} to discuss the interesting extension of our theoretical and methodological framework.
We first introduce several technical lemmas that are essential for proving our theoretical results (Section \ref{sec:lemmas}).

\subsection{Additional Lemmas}\label{sec:lemmas}

\begin{lemma}[Risk and Accuracy as Linear Functionals of Classifiers]
\label{lem:misclassification}
For any classifier $f: \X\times \A\to[0,1]$, 
we have 
\begin{equation}\label{eq:expr}
    R(f)=\sum_{a\in\{0,1\}}p_{a}\int  \lsb 1-2\eta_a(x)\rsb f(x,a) d\Pa(x)+\sum_{a\in\{0,1\}}p_{a}\int   \eta_a(x)  d\Pa(x),
\end{equation}
and
\begin{equation}\label{eq:expacc}
\textup{Acc}(f)=\sum_{a\in\{0,1\}}p_{a}\int  \lsb 2\eta_a(x)-1\rsb f(x,a) d\Pa(x)+\sum_{a\in\{0,1\}}p_{a}\int  \lsb 1- \eta_a(x)\rsb  d\Pa(x),\end{equation}
\end{lemma}
\begin{proof}

By definition, $\widehat{Y}$ is conditionally independent of $Y$ given $X$ and $A$. Thus,
 \begin{align*}
 &\P(\widehat{Y}_f=1,Y=0\mid X=x,A=a)=f(x,a)(1-\eta_a(x)),\\
 &\P(\widehat{Y}_f=1,Y=0\mid X=x,A=a)=\eta_a(x)(1-f(x,a)).
 \end{align*}
 This implies that
 \begin{align*}
&R(f)=\P(Y\neq \widehat{Y}_{f})=\sum_{a\in\A}p_a \P(Y\neq \widehat{Y}_{f}|A=a)\\
&=\sum_{a\in\{0,1\}}p_a\int_\X \lsb \P(\widehat{Y}_f=1,Y=0\mid X=x,A=a)+\P(\widehat{Y}_f=1,Y=0\mid X=x,A=a)\rsb d\Pa(x)\\ 
 &=\sum_{a\in\{0,1\}}p_a\int_\X \lsb f(x,a)(1-\eta_a(x))+\eta_a(x)(1-f(x,a)) \rsb d\Pa(x)\\
 &=\sum_{a\in\{0,1\}}p_{a}\int  \lsb 1-2\eta_a(x)\rsb f(x,a) d\Pa(x)+\sum_{a\in\{0,1\}}p_{a}\int  \lsb  \eta_a(x)\rsb  d\Pa(x).
\end{align*}
Moreover,
 \begin{align*}
\textup{Acc}(f)&=1-R(f) =1-\sum_{a\in\{0,1\}}p_a\int_\X \lsb f(x,a)(1-\eta_a(x))+\eta_a(x)(1-f(x,a)) \rsb d\Pa(x)\\
 &=\sum_{a\in\{0,1\}}p_a\int_\X \lsb 1- \eta_a(x) -(1-2\eta_a(x))f(x,a) \rsb d\Pa(x)\\
 &=\sum_{a\in\{0,1\}}p_{a}\int  \lsb 2\eta_a(x)-1\rsb f(x,a) d\Pa(x)+\sum_{a\in\{0,1\}}p_{a}\int  \lsb 1- \eta_a(x)\rsb  d\Pa(x).
\end{align*}
\end{proof}

\begin{lemma}[Characterizing Bayes-Optimal Classifiers for Linear Disparity Measures]\label{lem:fb}
For $t\in\R$ and $\delta\ge0$,
let $f_{\textup{Dis},t}$, $\Dt$ and $\td$ be defined in \eqref{eq:partos}, \eqref{eq:dt} and \eqref{eq:td}, respectively.
Then, using the convention that $0/0=0$, for any fixed $t$,  we have
$$ f_{\textup{Dis},t} = \argmin_{f\in\mF}\lbb R(f):  \frac{t\cdot\textup{Dis}(f)}{|t|} \le \frac{t\cdot \Dt}{|t|}
\rbb.   $$
Moreover,  for all classifiers $f'\in\argmin_{f\in\mF}\lbb R(f): { t\cdot\textup{Dis}(f)}/|t| \le {t\cdot \Dt}/{|t|}
\rbb$, $f'=f_{\textup{Dis},t}$ almost surely with respect to $\P_{X,A}$.
In addition, if $t\in [\min (0,t_{\textup{Dis}} (0)),\max (0,t_{\textup{Dis}} (0))]$,
$$ f_{\textup{Dis},t} = \argmin_{f\in\mF}\lbb R(f): { {|\textup{Dis}(f)}|} \le|\Dt|\rbb.   $$
\end{lemma}
\begin{proof}

This result is a consequence of the generalized Neyman-Pearson lemma. If $t=0$, the result follows since  $f_{\textup{Dis},0}$ is the unconstrained Bayes-optimal classifier.
When $t\neq 0$,
let, for all $x,a$, 
$\phi_0(x,a) = 2\eta^Y_{A=1}(x,a)-1$, $\phi_1(x,a)={ w_{\textup{Dis}}(x)}$ 
and for all classifiers $f$, and $t\in\R$,
$\overline{\textup{Dis}}_t(f) = t\cdot \textup{Dis}(f) /|t|$.
 We have, for all $x,a$, 
\begin{equation*} f_{\textup{Dis},t}(x,a) =  I\lsb\phi_0(x,a)> t\phi_1(x,a)\rsb =
I\lsb \phi_0(x,a)> |t| \frac{t\phi_1(x,a)}{|t|}\rsb  
.\end{equation*}
Moreover, by Lemma \ref{lem:misclassification} and \eqref{eq:exp dis}, we can write $\textup{Acc}(f)$ and $\overline{D}_t(f)$ as
\begin{align*} 
&   \textup{Acc}(f) = \int_\A\int_\X f(x,a) \phi_0(x,a) d\P_{X,A}(x,a) + \int_\A\int_\X (1-\eta_a(x))  d\P_{X,A}(x,a)  ;\\
&\overline{\textup{Dis}}_t(f) =\frac{t}{|t|} \textup{Dis}(f)=  \int_{\X\times\A}f(x,a)\frac{t\phi_1(x,a)}{|t|}d\P_{X,A}(x,a).
\end{align*}
Define:
\begin{equation*}
\begin{array}{l}
\mathcal{F}_{t,=} = \lbb f:  \overline{\textup{Dis}}_t(f) = \frac{t\Dt}{|t|}\rbb;  
\mathcal{F}_{t,|\cdot|,\le} = \lbb f:  \lab\overline{\textup{Dis}}_t(f)\rab \le \frac{t\Dt}{|t|}\rbb;  \text{ and } 
\mathcal{F}_{t,\le} = \lbb f:\overline{\textup{Dis}}_t(f) \le \frac{t\Dt}{|t|}\rbb.
\end{array}\end{equation*}
It is clear that $f_{\textup{Dis},t}\in\mathcal{F}_{t,=}\subset\mathcal{F}_{t,\le}$. Since $|t|\ge0$, by the generalized Neyman-Pearson lemma (Lemma \ref{NP_lemma}),
$$f_{\textup{Dis},t} \in \underset{f\in\mathcal{F}_{t,\le}}{\argmax}\,  \textup{Acc}(f),$$
and, since $\Pa(\eta_a(X)=1/2+\td\wD(X,a)/2)=0$ when both $\eta_a(X)$ and $\wD(X,a)$ are continuous random variables with respect to $\Pa$, for all $f'\in\argmax_{f\in\mathcal{F}_{t,\le}}\,  \textup{Acc}(f), $
$f'= f_{\textup{Dis},t}$ almost surely with respect to $\P_{X,A}$.

Now, suppose further that $t\in [\min (0,t_{\textup{Dis}} (0)),\max (0,t_{\textup{Dis}} (0))]$.
By result (1) of 
Proposition \ref{prop:monotonicity}, $\Dt$ is monotone non-increasing with respect to $t$.
By the definition of $\td$ in \eqref{eq:td}, we have $\tz\ge  0$ when $\Dz\ge 0$ and   $\tz\le 0$ when $\Dz\le 0$. In both cases, we have $\tz\cdot\Dz\ge 0$. 
This further implies $t\cdot \Dz\ge0$.
Moreover, suppose now that $\td \ge 0$
so that 
 $t\in [0,t_{\textup{Dis}} (0))]$.
Then, since $D_\mathrm{Dis}$ non-increasing, we have $D_\mathrm{Dis}(t) \ge D_\mathrm{Dis}(t_{\textup{Dis}} (0)) \ge 0$.
Thus $t\cdot D_\mathrm{Dis}(t)\ge0$
and 
$t\cdot D_\mathrm{Dis}(t)/|t|\ge0$.
The same claim holds when $t\le 0$.
Therefore, 
$\overline{\textup{Dis}}_t(f) = \frac{t\Dt}{|t|}$
implies
$\lab\overline{\textup{Dis}}_t(f)\rab \le \frac{t\Dt}{|t|}$.
Consequently, $f_{\textup{Dis},t}\in\mathcal{F}_{t,=}\subset\mathcal{F}_{t,|\cdot|,\le}$. 
Then,
$$\underset{f\in\mathcal{F}_{t,\le}}{\max}\, \textup{Acc}(f) = \textup{Acc}(f_{\textup{Dis},t})\le  \underset{f\in\mathcal{F}_{t,|\cdot|,\le}}{\max}\, \textup{Acc}(f) \le  \underset{f\in\mathcal{F}_{t,\le}}{\max}\, \textup{Acc}(f).$$
Thus, we can conclude that 
$$f_{t} = \underset{f\in\mathcal{F}_{t,|\cdot|,\le}}{\argmax}\,  \textup{Acc}(f)
= \argmin_{f\in\mF}\lbb R(f): { |\textup{Dis}(f)|} \le \frac{t\cdot \Dt}{|t|}
\rbb. $$
\end{proof}

\begin{lemma}[Properties of the Tradeoff function]\label{lem:boundtrade}
For any $\delta\ge0$,
let 
$\td$ and $T(\delta)$ be defined in  \eqref{eq:dt} and \eqref{eq:Td}, respectively. 
Let for $a\in\{0,1\}$, $\eta_a(X)$ and $\wD(X,a)$ be continuous random variables when $X\sim \P_{X\mid A=a}$,  for $a\in\{0,1\}$. 
Then, for $\min (0,\Dz)\le \delta_1<\delta_2\le \max (0,\Dz)$,
\begin{equation} \label{b3ineq}
\tdt(\delta_2-\delta_1) \le T(\delta_1)-T(\delta_2) \le \tdo(\delta_2-\delta_1).
\end{equation}
\end{lemma}
\begin{proof}
As $\Dt$ is monotone non-increasing and continuous in $t\in\R$, it follows that
$\td$ is strictly decreasing in $\delta$ 
on $[0,|\Dz|]$. Then, we have $\tdo> \tdt$. For
$a\in \{0,1\}$, we define
$$\mathcal{T}_{a,+}=\lbb x\in\X,\,  \wD(x,a)> 0, \,  \tdt \le \frac{2\eta_a(x)-1}{ \wD(x,a)} < \tdo  \rbb,$$
 and 
$$\mathcal{T}_{a,-}=\lbb x\in\X,\,  \wD(x,a)< 0, \,  \tdt \le \frac{2\eta_a(x)-1}{ \wD(x,a)} < \tdo \rbb.$$
By definition,
\begin{align*}
& f_{\tdo}(x,a)-f_{\tdt}(x,a)\\
&= I\lsb \eta_a(x) >\frac12+\frac{\tdo}2 \wD(x,a) \rsb -I\lsb \eta_a(x)
 >\frac12+\frac{\tdo}2 \wD(x,a) \rsb\\
&=   \left\{ \begin{array}{lcl}
-I\lsb  \tdt \le \frac{2\eta_a(x)-1}{\wD(x,a)} < \tdo\rsb,    && \wD(x,a) > 0;\\
 I\lsb  \tdt \le \frac{2\eta_a(x)-1}{\wD(x,a)} < \tdo \rsb,    && \wD(x,a) < 0;\\
   0,   && \wD(x,a) =0,
    \end{array}\right.
\end{align*}
In addition, we have, on $\mathcal{T}_{a,+}$,
$${\tdt}\cdot\wD(x,a)\le 2\eta_a(x)-1 \le  {\tdo}\cdot\wD(x,a),$$
and
on  $\mathcal{T}_{a,-}$,
$${\tdo}\cdot\wD(x,a)\le 2\eta_a(x)-1 \le {\tdt}\cdot\wD(x,a).$$
Then, by \eqref{eq:expr} from Lemma \ref{lem:misclassification},
\begin{align*}
&T(\delta_1) -T(\delta_2)=R\lsb f_{\tdo}\rsb-R\lsb f_{\tdt}\rsb\\
&=\sum_{a\in\A}p_{a}\int_\X \lsb 1-2\eta_a(x)\rsb \lsb f_{\tdo}(x,a)-f_{\tdt}(x,a)\rsb d\Pa(x)\\
&=\sum_{a\in\A}p_{a}\int_{\{\wD(x,a)>0\}} \lsb 2\eta_a(x)-1\rsb I\lsb  \tdt\le \frac{2\eta_a(x)-1}{\wD(x,a)} < \tdo\rsb d\Pa(x)\\
&-\sum_{a\in\A}p_{a}\int_{\{\wD(x,a)<0\}} \lsb 2\eta_a(x)-1\rsb I\lsb  \tdt\le \frac{2\eta_a(x)-1}{\wD(x,a)} < \tdo\rsb d\Pa(x)\\
&=\sum_{a\in\A}p_{a}\int_{\mathcal{T}_{a,+}} \lsb 2\eta_a(x)-1\rsb d\Pa(x)-\sum_{a\in\A}p_{a}\int_{\mathcal{T}_{a,+}} \lsb 2\eta_a(x)-1\rsb  d\Pa(x)\\
\le&\tdo\sum_{a\in\A}p_{a}\int_{\mathcal{T}_{a,+}}  \wD(x,a) d\Pa(x)-\tdo\sum_{a\in\A}p_{a}\int_{\mathcal{T}_{a,-}}  \wD(x,a)d\Pa(x)\\
&=-\tdo\sum_{a\in\A}p_{a}\int_{\X\cap\{\wD(x,a)>0\}}  \wD(x,a)\cdot \lsb f_{\tdo}(x,a)-f_{\tdt}(x,a) \rsb d\Pa(x)\\
&-\tdo\sum_{a\in\A}p_{a}\int_{\X\cap\{\wD(x,a)<0\}}  \wD(x,a)\cdot \lsb f_{\tdo}(x,a)-f_{\tdt}(x,a)\rsb d\Pa(x)\\
&=-\tdo\sum_{a\in\A}p_{a}\int_{\X}  \wD(x,a)\cdot \lsb f_{\tdo}(x,a)-f_{\tdt}(x,a)\rsb d\Pa(x)\\
&=-\tdo\lsb\textup D\lsb f_{\tdo} \rsb -D\lsb f_{\tdt} \rsb \rsb=\tdo(\delta_2-\delta_1).
\end{align*}
On the other hand,
\begin{align*}
&T(\delta_1) -T(\delta_2)=\sum_{a\in\A}p_{a}\int_{\mathcal{T}_{a,+}} \lsb 2\eta_a(x)-1\rsb d\Pa(x)-\sum_{a\in\A}p_{a}\int_{\mathcal{T}_{a,+}} \lsb 2\eta_a(x)-1\rsb  d\Pa(x)\\
\ge&\tdo\sum_{a\in\A}p_{a}\int_{\mathcal{T}_{a,+}}  \wD(x,a) d\Pa(x)-\tdt\sum_{a\in\A}p_{a}\int_{\mathcal{T}_{a,-}}  \wD(x,a)d\Pa(x)\\
&=-\tdt\sum_{a\in\A}p_{a}\int_{\X\cap\{\wD(x,a)>0\}}  \wD(x,a)\cdot \lsb f_{\tdo}(x,a)-f_{\tdt}(x,a) \rsb d\Pa(x)\\
&-\tdt\sum_{a\in\A}p_{a}\int_{\X\cap\{\wD(x,a)<0\}}  \wD(x,a)\cdot \lsb f_{\tdo}(x,a)-f_{\tdt}(x,a)\rsb d\Pa(x)\\
&=-\tdt\sum_{a\in\A}p_{a}\int_{\X}  \wD(x,a)\cdot \lsb f_{\tdo}(x,a)-f_{\tdt}(x,a)\rsb d\Pa(x)\\
&=-\tdt\lsb\textup D\lsb f^\star_{\delta_1} \rsb -D\lsb f^\star_{\delta_2} \rsb\rsb =\tdt(\delta_2-\delta_1).
\end{align*}
We conclude \eqref{b3ineq}.

\end{proof}

\subsection{Proofs of Results from Section \ref{gnpfc}}
\subsubsection{Proof of Proposition \ref{prop:expression of dl}}

Let   $f$ be any  classifier. 
Its DD can be expressed as
\begin{align}
 \label{exp:DD}\nonumber\textup{DD}(f)&=\P(\widehat{Y}_f  = 1|A=1)-\P(\widehat{Y}_f=1|A=0)\\
\nonumber&=\int_{\mathcal{X}}f(x,1)d\P_{X\mid A=1}(x)-\int_{\mathcal{X}}f(x,0)d\P_{X\mid A=0}(x)\\
\nonumber&=\sum_{a\in\{0,1\}}p_a\left(\int_{\mathcal{X}}\left[\frac{I(a=1)}{p_{1}}f(x,a)-\frac{I(a=0)}{p_{0}}f(x,a)\right]d\P_{X\mid A=a}(x)\right)\\
\nonumber&=\int_{\mathcal{A}}\int_{\mathcal{X}}f(x,a)\left(\frac{2a-1}{p_{a}}\right){d\P_{X\mid A=a}}d\P_A(a)\\
&=\int_{\mathcal{A}}\int_{\mathcal{X}}f(x,a)(2a-1)\left(\frac{1}{p_{a}}\right)d\P_{X,A}(x,a).
\end{align}

By Bayes' theorem, we can express the regression function $\eta_a$ in terms of the conditional distributions of $X$ given $A,Y$ and $A$, as follows
\begin{equation*}
\eta_a(x)=\P\lsb Y=1|X=x,A=a\rsb=
\frac{\P(A=a,Y=1) d\P_{X\mid A=a,Y=1}(x)}{
\P(A=a) d\P_{X\mid A=a}(x)}
=\frac{ p_{a,1} d\P_{X\mid A=a,Y=1}(x)}{p_a  d\P_{X\mid A=a}(x)}.
\end{equation*}

Next, $\text{DO}$ can be expressed as

\begin{align}
 \label{exp:DO}\nonumber\textup{DO}(f)&=\P(\widehat{Y}_f  = 1|A=1,Y=1)-\P(\widehat{Y}_f=1|A=0,  = 1)\\
\nonumber&=\int_{\mathcal{X}}f(x,1)d\P_{X\mid A=1,Y=1}(x)-\int_{\mathcal{X}}f(x,0)d\P_{X\mid A=0,Y=1}(x)\\
\nonumber&=\sum_{a\in\{0,1\}}p_a\left(\int_{\mathcal{X}}\left[\frac{I(a=1)}{p_{1}}f(x,a)-\frac{I(a=0)}{p_{0}}f(x,a)\right]d\P_{X\mid A=a,Y=1}(x)\right)\\
\nonumber&=\int_{\mathcal{A}}\int_{\mathcal{X}}f(x,a)\left(\frac{2a-1}{p_{a}}\right)\frac{d\P_{X\mid A=a,Y=1}(x)}{d\P_{X\mid A=a}}{d\P_{X\mid A=a}}d\P_A(a)\\
&=\int_{\mathcal{X}\times\mathcal{A}}f(x,a)(2a-1)\left(\frac{\eta_a(x)}{p_{a,1}}\right)d\P_{X,A}(x,a).
\end{align}

Similarly to above
\begin{equation*}
1-\eta_a(x)=\P\lsb Y=0|X=x,A=a\rsb=
\frac{\P(A=a,Y=0) d\P_{X\mid A=a,Y=0}(x)}{
\P(A=a) d\P_{X\mid A=a}(x)}
=\frac{ p_{a,0} d\P_{X\mid A=a,Y=1}(x)}{p_a  d\P_{X\mid A=a}(x)}.
\end{equation*}
Thus, $\text{PD}$ can be expressed as



\begin{align}
 \label{exp:PD}\nonumber\textup{PD}(f)&=\P(\widehat{Y}_f  = 1|A=1,Y=0)-\P(\widehat{Y}_f  = 1|A=0,Y=0)\\
\nonumber&=\int_{\mathcal{X}}f(x,1)d\P_{X\mid A=1,Y=0}(x)-\int_{\mathcal{X}}f(x,0)d\P_{X\mid A=0,Y=0}(x)\\
\nonumber&=\sum_{a\in\{0,1\}}p_a\left(\int_{\mathcal{X}}\left[\frac{I(a=1)}{p_{1}}f(x,a)-\frac{I(a=0)}{p_{0}}f(x,a)\right]d\P_{X\mid A=a,Y=0}(x)\right)\\
\nonumber&=\int_{\mathcal{A}}\int_{\mathcal{X}}f(x,a)\left(\frac{2a-1}{p_{a}}\right)\frac{d\P_{X\mid A=a,Y=0}(x)}{d\P_{X\mid A=a}}{d\P_{X\mid A=a}}d\P_A(a)\\
&=\int_{\mathcal{X}\times\mathcal{A}}f(x,a)(2a-1)\left(\frac{1-\eta_a(x)}{p_{a,0}}\right)d\P_{X,A}(x,a). 
\end{align}
Thus, Proposition \ref{prop:expression of dl} follows.

\subsection{Proofs of Results from Section \ref{sec:Bayes}}
\subsubsection{Proof of Proposition \ref{prop:monotonicity}}
For (1), since, for $a\in\{0,1\}$, both $\eta_a(X)$ and $\wD(X,a)$ are continuous random variable given $A=a$, we have that for $a\in\{0,1\}$, $t\mapsto \Pa\lsb \eta_a(X)>1/2+{t}\wD(X,a)/2\rsb$ is a continuous function. 
Thus, $t\mapsto \Dt$ is continuous.
Now, define, for $a\in\{0,1\}$, 
\begin{equation*}
   {G}_{a,+} = \{x \in\mathcal{X},  w_{\textup{Dis}}(x,a) > 0\};  \ \ \ \, \quad {G}_{a,0} = \{x\in\mathcal{X},  w_{\textup{Dis}}(x,a) = 0\} \ \ \ \, \quad {G}_{a,-} = \{x\in\mathcal{X}, w_{\textup{Dis}}(x,a) < 0\}.
\end{equation*}
Letting $t_1<t_2$, we have  for $a\in\{0,1\}$ and $x\in \mX$, 
\begin{equation}\label{eq:diffinf}
 f_{\textup{Dis},t_1}(x,a)-f_{\textup{Dis},t_2}(x,a)
= \left\{
\begin{array}{ll}
  I\lsb t_1< \frac{2\eta_a(x)-1}{  \wD(x,a)}
  \le  t_2\rsb  ,  & x\in G_{a,+};\\
  - I\lsb t_1< \frac{2\eta_a(x)-1}{  \wD(x,a)}
  \le  t_2 \rsb ,  & x\in G_{a,-};\\
0,   & \text{otherwise}.
\end{array}
\right.
\end{equation}
It thus follows that,
\begin{align*}
&\Dto-\Dtt = \textup{Dis}(f_{\textup{Dis},t_1})-\textup{Dis}(f_{\textup{Dis},t_2})
= \int_{\A}\int_{\X}\lmb
f_{\textup{Dis},t_1}(x,a)-f_{\textup{Dis},t_2}(x,a)
\rmb w_{\textup{Dis}}(x,a) d\P_{X,A}({x,a})\\
&=\sum_{a\in\{0,1\}} p_{a}\int_{\X}\lmb
f_{\textup{Dis},t_1}(x,a)-f_{\textup{Dis},t_2}(x,a)
\rmb w_{\textup{Dis}}(x,a) d\Pa\\
&=\sum_{a\in\{0,1\}} p_{a}\int_{G_{a,+}}   I\lsb t_1< \frac{2\eta_a(x)-1}{  \wD(x,a)}
  \le  t_2\rsb   w_{\textup{Dis}}(x,a) d\P_{X\mid A=a}\\
  &-\sum_{a\in\{0,1\}} p_{a}\int_{G_{a,-}}   I\lsb t_1< \frac{2\eta_a(x)-1}{  \wD(x,a)}
  \le  t_2\rsb   w_{\textup{Dis}}(x,a) d\P_{X\mid A=a}\ge 0.
\end{align*}
The last inequality holds since the indicator function is non-negative, and
$  w_{\textup{Dis}}(x,a)$ is positive on $G_{a,+}$ and negative on $G_{a,-}$.
Thus, $t\mapsto\Dt$ is a monotone non-increasing function.

For (2), when $t_1<t_2<0$, it follows that on $G_{a,+}$,
$$ (1-2\eta_a(x)) I\lsb t_1< \frac{2\eta_a(x)-1}{  \wD(x,a)}  \le  t_2\rsb \ge -t_2 \wD(x,a) I\lsb t_1< \frac{2\eta_a(x)-1}{  \wD(x,a)}  \le  t_2\rsb \ge 0,
$$
and, on $G_{a,-}$,
$$ (2\eta_a(x)-1) I\lsb t_1< \frac{2\eta_a(x)-1}{  \wD(x,a)}  \le  t_2\rsb \ge t_2  \wD(x,a) I\lsb t_1< \frac{2\eta_a(x)-1}{  \wD(x,a)}  \le  t_2\rsb \ge 0.
$$
Then, by Lemma \ref{lem:misclassification} and \eqref{eq:diffinf},
\begin{align*}
&R(f_{\textup{Dis},t_1})-R(f_{\textup{Dis},t_2}) = \int_{\A}\int_{\X} (2\eta_a(x)-1)\lmb
f_{\textup{Dis},t_1}(x,a)-f_{\textup{Dis},t_2}(x,a)\rmb   d\P_{X,A}({x,a})\\
&=\sum_{a\in\{0,1\}} p_{a}\int_{\X}(2\eta_a(x)-1)\lmb
f_{\textup{Dis},t_1}(x,a)-f_{\textup{Dis},t_2}(x,a)
\rmb  d\Pa(x,a)\\
&=\sum_{a\in\{0,1\}} p_{a}\int_{G_{a,+}}  (1-2\eta_a(x)) I\lsb t_1< \frac{2\eta_a(x)-1}{  \wD(x,a)}  \le  t_2\rsb   d\P_{X\mid A=a}(x,a)\\
  &+\sum_{a\in\{0,1\}} p_{a}\int_{G_{a,-}} (2\eta_a(x)-1)  I\lsb t_1< \frac{2\eta_a(x)-1}{  \wD(x,a)}  \le  t_2\rsb    d\P_{X\mid A=a}(x,a)\ge 0.
\end{align*}
Thus, $t\mapsto R(f_{\textup{Dis},t})$ is monotone non-increasing on $(-\infty,0)$.

On the other hand, when $0\le t_1<t_2$, we have on $G_{a,+}$ that
$$ (1-2\eta_a(x)) I\lsb t_1< \frac{2\eta_a(x)-1}{  \wD(x,a)}  \le  t_2\rsb \le -t_1 \wD(x,a) I\lsb t_1< \frac{2\eta_a(x)-1}{  \wD(x,a)}  \le  t_2\rsb \le 0,
$$
and, on $G_{a,-}$ that
$$ (2\eta_a(x)-1) I\lsb t_1< \frac{2\eta_a(x)-1}{  \wD(x,a)}  \le  t_2\rsb \le t_1  \wD(x,a) I\lsb t_1< \frac{2\eta_a(x)-1}{  \wD(x,a)}  \le  t_2\rsb \le 0.
$$
Then, by Lemma \ref{lem:misclassification} and \eqref{eq:diffinf},
\begin{align*}
&R(f_{\textup{Dis},t_1})-R(f_{\textup{Dis},t_2}) = \int_{\A}\int_{\X} (2\eta_a(x)-1)\lmb
f_{\textup{Dis},t_1}(x,a)-f_{\textup{Dis},t_2}(x,a)\rmb   d\P_{X,A}({x,a})\\
&=\sum_{a\in\{0,1\}} p_{a}\int_{\X}(2\eta_a(x)-1)\lmb
f_{\textup{Dis},t_1}(x,a)-f_{\textup{Dis},t_2}(x,a)
\rmb  d\Pa(x,a)\\
&=\sum_{a\in\{0,1\}} p_{a}\int_{G_{a,+}}  (1-2\eta_a(x)) I\lsb t_1< \frac{2\eta_a(x)-1}{  \wD(x,a)}  \le  t_2\rsb   d\P_{X\mid A=a}(x,a)\\
  &+\sum_{a\in\{0,1\}} p_{a}\int_{G_{a,-}} (2\eta_a(x)-1)  I\lsb t_1< \frac{2\eta_a(x)-1}{  \wD(x,a)}  \le  t_2\rsb    d\P_{X\mid A=a}(x,a)\le 0.
\end{align*}
Thus, $t\mapsto R(f_{\textup{Dis},t})$ is monotone non-decreasing on $[0,\infty)$.
This shows that $t\mapsto R(f_{\textup{Dis},t})$ is non-increasing in $|t|$.

\subsubsection{Proof of Theorem \ref{thm-opt-dp}}

We analyze the following three cases separately: (1) $|\textup{Dis} (0)|\le \delta$, (2) $\textup{Dis} (0)>\delta$ and (3) $\textup{Dis} (0)<-\delta$. 
Since the proof for case (3) 
is analogous to case (2), we omit the discussion of case (3).

Case (1): $|\textup{Dis} (0)|\le \delta$.
 In this case,  the unconstrained Bayes-optimal classifier satisfies the fairness constraint. 
 As a result, we have $\td=0$ and a $\delta$-fair Bayes-optimal classifier is given by $f^\star_{\textup{Dis},\delta}(x,a) = I(\eta_a(x)>1/2).$ for all $x,a$.

Case (2): $\textup{Dis} (0)>\delta$.
When, for $a\in\{0,1\}$, $\eta_a(X)$ has probability density function on $\mathcal{X}$,  ${D}_{\textup{Dis}}(t)$ is a continuous non-increasing function on $\R$,
and thus we have
${D}_{\textup{Dis}}(\td)=\delta$. 
Moreover, $\textup{Dis} (0)>\delta$ indicates $\td>0$. Then, by Lemma \ref{lem:fb},
with $t=\td$,
\begin{align*}
f_{\mathrm{Dis}, \td}  =\argmin_{f\in\mF}\lbb R(f): |\textup{Dis}(f)|\le \frac{t\cdot {D}_{\textup{Dis}}(\td)}{|t|}\rbb = \argmin_{f\in\mF}\{ R(f): |\textup{Dis}(f)|\le\delta\}. \end{align*}
This finishes the proof.

\subsubsection{Proof of Proposition \ref{prop:FPF}}
Without loss of generality, we assume $\Dz\ge 0$. 
In this case, we have $\tl=0$ and $\tu=\tz$. 
By Lemma \ref{lem:fb},  we have 
$$\{f_{\textup{Dis},t}: t\in[0,\tz]\} \subset \textup{FPF}\}.$$
On the other hand, since $f_{\textup{Dis},0} \in \argmin_{f\in\mF}\{ R(f)\},$
and by Lemma \ref{lem:fb},
$f_{\textup{Dis},\tz} \in \argmin_{f\in\mF}\{ R(f):  \textup{Dis}(f) \le 0\},$
we have, for $f_1$ with $\textup{Dis}(f)<0$ and $f_2$ with $\textup{Dis}(f)>\delta$,
$$R(f_{\textup{Dis},\tz}) \le R(f_1), \, |\textup{Dis}(f_{\textup{Dis},\tz})|=0< |\textup{Dis}(f_{1})|,$$
and
$$R(f_{0}) \le R(f_2), \, |\textup{Dis}(f_{0})|=\Dz< |\textup{Dis}(f_{2})|.$$
Thus, for any  $f_{\textup{FPF}}\in \textup{FPF}$, we have
$$0\le \textup{Dis}(f_{\textup{FPF}})\le \textup{Dis} (0).$$
Moreover, by the definition of the fair Pareto frontier, we have
$$R(f_{\textup{FPF}}) = \min_{f\in\mF}\{R(f):  \textup{Dis}(f)\le \textup{Dis}(f_{\textup{FPF}})\}.$$
 Since $t\mapsto\Dt$ is a continuous monotone non-increasing function on $[0,\tz]$ with ${D}_{\textup{Dis}}(\tz)=0$, there exists a $t\in [0,\tz]$ such that $\textup{Dis}(f_{\textup{Dis},t})=\Dt=\textup{Dis}(f_{\textup{FPF}})$. 
 By Lemma \ref{lem:fb},
 $$R(f_{\textup{Dis},t}) = \min_{f\in\mF}\{R(f):  \textup{Dis}(f)= \textup{Dis}(f_{\textup{FPF}})\}. $$
 In conclusion, there exists a $t\in[0,\tz]$ such that, 
$$ R(f_{\textup{Dis},t}) = R(f_{\textup{FPF}})  \ \ \text{ and } \ \  \textup{Dis}(f_{\textup{Dis},t}) = \textup{Dis}(f_{\textup{FPF}}). $$

\subsubsection{Proof of Theorem \ref{prop:tradeoff_convex}}
Let $ 0\le \delta_1<\delta_2\le |\Dz|$ and $\lambda\in [0,1]$, 
and denote $\delta_\lambda = \lambda {\delta}_1 + (1-\lambda)\delta_2$.
As $\td$ is strictly decreasing on $[0,|\Dz|]$, we have $\tdo> \tdl>\tdt$. 
By Lemma \ref{lem:boundtrade},
\begin{align*}
&T(\delta_\lambda) = \lambda T(\delta_\lambda) +(1-\lambda)T(\delta_\lambda)
= \lambda T(\delta_1) +(1-\lambda)T(\delta_2) + \lambda \lsb T(\delta_\lambda)-T(\delta_1)\rsb  +(1-\lambda)\lsb T(\delta_\lambda)-T(\delta_2)\rsb  \\
&\le  \lambda T(\delta_1) +(1-\lambda)T(\delta_2) -\lambda \tdl (\delta_\lambda -\delta_1)  +(1-\lambda)\tdl (\delta_2-\delta_\lambda)
= \lambda T(\delta_1) +(1-\lambda)T(\delta_2). 
\end{align*}
The proof is thus completed.

\subsubsection{Proof of Proposition \ref{prop:ddpnoa}}
By Bayes' theorem, we have for all $x,y,a$ that
\begin{equation*}
\P\lsb A=a\mid X=x\rsb=
\frac{p_a d\P_{X\mid A=a}(x)}{d\P_{X}(x)},
\end{equation*}
and
\begin{equation*}
\P(Y=y\mid A=a,X=x)\P(A=a\mid X=x)=\P(A=a,Y=y\mid X=x)=\frac{p_{a,y}d\P_{X\mid A=a,Y=y}(x)}{d\P_X(x)}.
\end{equation*}
It then follows that
\begin{align*}
&\frac{d\P_{X\mid A=1}(x)}{d\P_X(x)} = \frac{\P(A=1\mid X=x)}{\P(A=1)}=\frac{\eta^A(x)}{p_1};\\
&\frac{d\P_{X\mid A=0}(x)}{d\P_X(x)} = \frac{\P(A=0\mid X=x)}{\P(A=0)}=\frac{1-\eta^A(x)}{p_0};\\
&\frac{d\P_{X\mid A=1,Y=1}(x)}{d\P_X(x)} =\frac{\P(Y=1\mid A=1,X=x)\P(A=1|X=x)}{p_{1,1}} = \frac{\eta^Y_{A=1}(x)\eta^A(x)}{p_{1,1}};\\
&\frac{d\P_{X\mid A=1,Y=0}(x)}{d\P_X(x)} =\frac{\P(Y=0\mid A=1,X=x)\P(A=1|X=x)}{p_{1,0}} = \frac{(1-\eta^Y_{A=1}(x))\eta^A(x)}{p_{1,0}};\\
&\frac{d\P_{X\mid A=0,Y=1}(x)}{d\P_X(x)} =\frac{\P(Y=1\mid A=0,X=x)\P(A=0|X=x)}{p_{0,1}} = \frac{\eta^Y_{A=0}(x)(1-\eta^A(x)}{p_{0,1}};\\
&\frac{d\P_{X\mid A=0,Y=0}(x)}{d\P_X(x)} =\frac{\P(Y=0\mid A=0,X=x)\P(A=0|X=x)}{p_{0,0}} = \frac{(1-\eta^Y_{A=0}(x))(1-\eta^A(x))}{p_{0,0}}.
\end{align*}
Let   $f$ be any  classifier. 
Its DD, $\text{DO}$ and $\text{PD}$ can be expressed in turn as
\begin{align*}
\textup{DD}(f)&=\P(\widehat{Y}_f  = 1\mid A=1)-\P(\widehat{Y}_f=1\mid A=0)\\
&=\int_{\mathcal{X}}f(x)d\P_{X\mid A=1 , Y=1}(x)-\int_{\mathcal{X}}f(x)d\P_{X\mid A=0 , Y=1}(x)\\
&=\int_{\mathcal{X}}\lmb f(x)\lsb \frac{d\P_{X\mid A=1}(x)}{d\P_X(x)}-\frac{d\P_{X\mid A=0}(x)}{d\P_X(x)}\rsb\rmb d\P_X(x)\\
&=\int_{\mathcal{X}}f(x,a)\left(\frac{\eta^A(x)}{p_{1}} - \frac{1-\eta^A(x)}{p_0}\right)d\P_{X}(x);\\
\textup{DO}(f)&=\P(\widehat{Y}_f  = 1\mid A=1,Y=1)-\P(\widehat{Y}_f=1\mid A=0,Y=1)\\
&=\int_{\mathcal{X}}f(x)d\P_{X\mid A=1}(x)-\int_{\mathcal{X}}f(x)d\P_{X\mid A=0}(x)\\
&=\int_{\mathcal{X}}\lmb f(x)\lsb \frac{d\P_{X\mid A=1, Y=1}(x)}{d\P_X(x)}-\frac{d\P_{X\mid A=0, Y=1}(x)}{d\P_X(x)}\rsb\rmb d\P_X(x)\\
&=\int_{\mathcal{X}}f(x,a)\left(\frac{\eta^Y_{A=1}(x)\eta^A(x)}{p_{1,1}} - \frac{\eta^Y_{A=0}(x)(1-\eta^A(x))}{p_{0,1}}\right)d\P_{X}(x);\\
\textup{PD}(f)&=\P(\widehat{Y}_f  = 1\mid A=1,Y=0)-\P(\widehat{Y}_f=1\mid A=0,Y=0)\\
&=\int_{\mathcal{X}}f(x)d\P_{X\mid A=1, Y=0}(x)-\int_{\mathcal{X}}f(x)d\P_{X\mid A=0, Y =0}(x)\\
&=\int_{\mathcal{X}}\lmb f(x)\lsb \frac{d\P_{X\mid A=1, Y=0}(x)}{d\P_X(x)}-\frac{d\P_{X\mid A=0, Y=0}(x)}{d\P_X(x)}\rsb\rmb d\P_X(x)\\
&=\int_{\mathcal{X}}f(x,a)\left(\frac{(1-\eta^Y_{A=1}(x))\eta^A(x)}{p_{1,0}} - \frac{(1-\eta^Y_{A=0}(x))(1-\eta^A(x))}{p_{0,0}}\right)d\P_{X}(x).
\end{align*}
Thus, Proposition \ref{prop:ddpnoa} follows.

\subsection{Proofs of Results from Section \ref{sec:alg}}
\label{pfalg}
\subsubsection{Proof of Theorem \ref{thm:FUDS}}
By Lemma \ref{lem:fb}, we only need to prove that, for $t\in[\tl,\tu]$, the  unconstrained Bayes-optimal classifier for $\tPt$ is the same as $f_{\textup{Dis},t}$.
By Bayes' theorem, for all $x,a$,
\begin{align*}
\tetat_a(x) = \frac{ \tp^{\,t}_{a,1} d\tPt_{\tX|\tA=a,\tY=1}(x)}{\tpt_{a1}  d\tP_{\tX|\tA=a,\tY=1}(x) + \tpt_{a0}  d\tP_{\tX|\tA=a,\tY=0}(x)}.
\end{align*}
Then,  the unconstrained Bayes-optimal classifier for $\tPt$ is 
\begin{align*}
&{f}^{\textup{FUDS}}_t(x,a) = I\lsb \tetat_a(x) >\frac12\rsb
= I\lsb\frac{ \tpt_{a,1} d\tPt_{\tX|\tA=a,\tY=1}(x)}{\tpt_{a1}  d\tPt_{\tX|\tA=a,\tY=1}(x) + \tpt_{a0}  d\tPt_{\tX|\tA=a,\tY=0}(x)}>\frac12\rsb\\
= &I\lsb\frac{  c_a(1- H_{\textup{Dis},a}(t)) p_{a,y} d\P_{X\mid A=a,Y=1}(x)}{c_a(1- H_{\textup{Dis},a}(t)) p_{a,y}  d\tP_{X\mid A=a,Y=1}(x) + c_a  H_{\textup{Dis},a}(t) p_{a,y}  d\P_{X\mid A=a,Y=0}(x)}>\frac12\rsb\\
= &I\lsb (1- H_{\textup{Dis},a}(t)) p_{a,1} d\P_{X\mid A=a,Y=1}(x)>  H_{\textup{Dis},a}(t) p_{a,0}  d\P_{X\mid A=a,Y=0}(x)\rsb\\
= &I\lsb  p_{a,y} d\P_{X\mid A=a,Y=1}(x)>  H_{\textup{Dis},a}(t) \lmb   p_{a,1} d\P_{X\mid A=a,Y=1}(x) + p_{a,0}  d\P_{X\mid A=a,Y=0}(x)\rsb\rmb\\
&= I\lsb \eta_a(x) > \Hta\rsb = I\lsb \eta_a(x)>\frac12 + \frac{t}{2}\wD(x,a)\rsb = f_{\textup{Dis},t}(x,a).
\end{align*}
This finishes the proof.

\subsubsection{Proof of Theorem \ref{thm:FCSC}}
Again, we only need to prove that for $t\in[\tl,\tu]$, $f_t^{\textup{FCSC}}=f_{\textup{Dis},t}$. By definition,
\begin{align*}
&R^{\textup{FCSC}}_t(f)=\sum_{a\in\{0,1\}}\lmb
 H_{\textup{Dis},a}(t)\cdot \P(\widehat{Y}_f=1,Y=0,A=a)+(1- H_{\textup{Dis},a}(t))\cdot \P(\widehat{Y}_f=0, Y=1,A=a)\rmb\\
 &= \sum_{a\in\{0,1\}}p_a\lmb
   H_{\textup{Dis},a}(t)\cdot \P(\widehat{Y}_f=1,Y=0|A=a)+(1- H_{\textup{Dis},a}(t))\cdot \P(\widehat{Y}_f=0, Y=1|A=a)\rmb\\
   &=\sum_{a\in\{0,1\}}p_a 
   H_{\textup{Dis},a}(t)\cdot \int_\X\P (\widehat{Y}_f=1,Y=0|A=a,X=x)d\Pa(x)\\
   &+\sum_{a\in\{0,1\}}p_a (1- H_{\textup{Dis},a}(t))\cdot \int_\X\P(\widehat{Y}_f=0, Y=1|X=x,A=a)  d\Pa(x)\\
   &=\sum_{a\in\{0,1\}}p_a 
   \int_\X  \lmb H_{\textup{Dis},a}(t)f(x,a)(1-\eta_a(x))  + (1- H_{\textup{Dis},a}(t)) (1-f(x,a))\eta_a(x)  \rmb d\Pa(x)\\
   &=\sum_{a\in\{0,1\}}p_a \int_\X\lmb  
    H_{\textup{Dis},a}(t)f(x,a)
    -\eta_a(x)- H_{\textup{Dis},a}(t)\eta_a(x) -\eta_a(x)f(x,a)  \rmb d\Pa\\
    &=\sum_{a\in\{0,1\}}p_a \int_\X  
    \lsb H_{\textup{Dis},a}(t) -\eta_a\rsb f(x,a)
     d\Pa  - \sum_{a\in\{0,1\}}p_a \int_\X  
    \lsb 1+H_{\textup{Dis},a}(t) \rsb \eta_a(x)
     d\Pa.
    \end{align*}
Note that the second term does not depend on $f$
and that the first term is minimized by taking $f(x,a)=1$ if $\eta_a(x)>H_{\textup{Dis},a}(t)$ and $f(x,a)=0$ if $\eta_a\le H_{\textup{Dis},a}(t)$.
  We can thus conclude that for all $x,a$,
$$f^\textup{FCSC}_t(x,a) = I\lsb\eta_a(x)>H_{\textup{Dis},a}(t)\rsb = f_{\textup{Dis},t}(x,a).$$
This finishes the proof.

\section{Extensions}
\label{ext}
In this section, we discuss some interesting extensions of our theoretical and methodological framework, including (1) fair Bayes-optimal classifier with equalized odds; (2) fair Bayes-optimal classifier with multi-class protected attribute; (3) fair Bayes-optimal classifier with no distributional assumptions; (4) FUDS and FCSC algorithms with linear disparities (protected attribute $A$ is excluded from predictive attribute.)

\subsection{Fair Bayes-optimal Classifier with Equalized Odds}\label{sec:equodds}
In this section, we provide the detailed characterization of the fair Bayes-optimal classifier under equalized odds (EO), as defined in Theorem \ref{thm-fb-eq-odd}.
Recall that, for any $\delta>0$, a $\delta$-fair Bayes-optimal classifier is  defined as 
$$
f^\star_{\textup{DEO},\delta} =  {\argmin} \lbb R(f): \max\{|\textup{DO}(f)|,|\textup{PD}(f)|\}\le \delta \rbb,
$$
with DO and PD defined in \eqref{eq:disparity level} and written more explicitly as
\begin{align}\label{eq:dopd}
& \textup{DO}(f)=\int_{\X\times\A}f(x,a)\frac{(2a-1)\eta_a(x)}{p_{a,1}}d\P_{X,A}(x,a); \\
&\textup{PD}(f)=\int_{\X\times\A}f(x,a)\frac{(2a-1)(1-\eta_a(x))}{p_{a,0}}d\P_{X,A}(x,a).
\end{align}
The fairness constraint $\max\{|\textup{DO}(f)|,|\textup{PD}(f)|\}\le \delta$
consists of four linear inequalities on $f$:
$$  {\textup{DO}}(f)\le  \delta, \quad  -\textup{DO}(f)\le\delta, \quad \textup{PD}(f)\le\delta \, \text{ and } - \textup{PD}(f)\le\delta.$$
or
\begin{align*}
&\int_{\X\times\A}f(x,a)\frac{(2a-1)\eta_a(x)}{p_{a,1}}d\P_{X,A}(x,a)\le \delta;
&\int_{\X\times\A}f(x,a)\frac{(1-2a)\eta_a(x)}{p_{a,1}}d\P_{X,A}(x,a)\le \delta;\\
&\int_{\X\times\A}f(x,a)\frac{(2a-1)(1-\eta_a(x))}{p_{a,0}}d\P_{X,A}(x,a)\le \delta;
&\int_{\X\times\A}f(x,a)\frac{(1-2a)(1-\eta_a(x))}{p_{a,0}}d\P_{X,A}(x,a)\le \delta.
\end{align*}
Then, the generalized Neyman-Pearson lemma motivates us to consider classifiers of the form, for $(c_1,c_2,c_3,c_4)\in \R^4$, $x\in\X$ and $a\in\{0,1\}$,
\begin{align}\label{eq:feqodd0}
  \nonumber f_{c_1,c_2,c_3,c_4}(x,a) = &I\left(2\eta_a(x)-1 >  c_1 \frac{(2a-1)\eta_a(x)}{p_{a,1}} -c_2\frac{(2a-1)\eta_a(x)}{p_{a,1}}\right.\\
   & \qquad\qquad\qquad\left.+c_3 \frac{(2a-1)(1-\eta_a(x))}{p_{a,0}} -c_4\frac{(2a-1)(1-\eta_a(x))}{p_{a,0}}\right). 
\end{align}
Let $\mathcal{T}= [-p_{0,1},p_{1,1}]\times [-p_{1,0},p_{0,0}]  \setminus \{ (-p_{0,1},p_{0,0}),(-p_{1,0},p_{1,1})\})$.
Take $t_1= c_1-c_2\in \R$ and $t_2=c_3-c_4\in \R$ and define, for $a\in\{0,1\}$, $T_a: \mathcal{T}\to [0,1]$
for all $(t_1,t_2)\in\mathcal{T}$
by
\begin{equation}\label{ta}
T_a(t_1,t_2)= 
  \frac{p_{a,1}p_{a,0} +(2a-1)t_2 p_{a,1}  }{2p_{a,1}p_{a,0} +(2a-1) \lsb t_2 p_{a,1} -t_1 p_{a,0}\rsb}.    
\end{equation} 
Then,
\eqref{eq:feqodd0} can be simplified to, for all $x,a$,
\begin{equation}\label{eq:feqodd1}
   f_{\textup{Dis},t_1,t_2}(x,a) = I\left(2\eta_a(x)-1 >  t_1 \frac{(2a-1)\eta_a(x)}{p_{a,1}} +t_2 \frac{(2a-1)(1-\eta_a(x))}{p_{a,0}}\right)= I\lsb\eta_a(x)>T_a(t_1,t_2)\rsb.
\end{equation}
We then define the disparity functions $ D_\textup{DO}$: $\mathcal{T}\to [-1,1]$ and $ D_\textup{PD}$: $\mathcal{T}\to [-1,1]$ to measure the DO and PD of $f_{\textup{Dis},t_1,t_2}$, for all $t_1,t_2\in\mathcal{T}$,
as
\begin{align}
&    D_{\textup{DO}}(t_1,t_2) =\textup{DO}(f_{\textup{Dis},t_1,t_2})
=\int_{\X\times\A}I\lsb \eta_a(x)>T_a(t_1,t_2)\rsb\frac{(1-2a)\eta_a(x)}{p_{a,1}} d\P_{X,A}(x,a); \label{eq:eodeo} \\
 &{D}_{\textup{PD}}(t_1,t_2) 
 =\textup{PD}(f_{\textup{Dis},t_1,t_2})
 =\int_{\X\times\A}I\lsb \eta_a(x)>T_a(t_1,t_2)\rsb\frac{(2a-1)(1-\eta_a(x))}{p_{a,0}}d\P_{X,A}(x,a).\label{eq:eodpe}
\end{align}
\begin{proposition}[Continuity and Monotonicity of $\textup{DO}$ and $\textup{PD}$]\label{prop:mon-D12}

~
\begin{itemize}[]
    \item (1) For fixed $t_2\in(-p_{1,0},p_{0,0})$, both $t \mapsto {D}_{\textup{DO}}(t,t_2)$ and $t \mapsto {D}_{\textup{PD}}(t,t_2)$ are continuous and monotone non-increasing. Moreover,   $\textup{DO}(-p_{0,1},t_2)\ge 0$, 
 $ \textup{DO}(p_{1,1},t_2)\le 0$, $\textup{PD}(-p_{0,1},t_2)\ge 0$ and 
 $ \textup{PD}(p_{1,1},t_2)\le 0$.
    \item (2) For fixed $t_1\in(-p_{0,1},p_{1,1})$, both $t \mapsto {D}_{\textup{DO}}(t_1,t)$ and $t \mapsto {D}_{\textup{PD}}(t_1,t)$ are continuous and monotone non-increasing. Moreover  ${D}_{\textup{DO}}(t_1,-p_{1,0})\ge 0$, 
 $ {D}_{\textup{DO}}(t_1,p_{0,0})\le 0$, ${D}_{\textup{PD}}(t_1,-p_{1,0})\ge 0$ and 
 $ {D}_{\textup{PD}}(t_1,p_{0,0})\le 0$.
\end{itemize}
\end{proposition}

Next, we define the following quantities,
which are useful in determining the optimal thresholds $t_1,t_2$ for the $\delta$-fair Bayes-optimal classifier. 
For any $\delta>0$, define: $\acute{t}_{\textup{DO}}: \R_+ \to [-p_{0,1},p_{1,1}]$ and $\acute{t}_{\textup{PD}}: \R_+ \to[-p_{1,0},p_{0,0}]$ by 
\begin{align}\label{eq:atd}
\acute{t}_{\textup{DO}}(\delta) = \argmin_{t}\{|t|, |{D}_{\textup{DO}}(t,0)|\le\delta \}, \qquad \acute{t}_{\textup{PD}}(\delta) = \argmin_t\{ |t|, |\textup{PD} (0,t)|\le \delta\}.
\end{align}
By Proposition \ref{prop:mon-D12} with $t_1 = t_2=0$, these quantities are well-defined.

Now, we define $t_{\textup{DEO},1}$: $[0,\infty)\to \R$ and $t_{\textup{DEO},2}$: $[0,\infty)\to \R$, 
the thresholds of the $\delta$-fair Bayes-optimal classifiers in Theorem \ref{thm-fb-eq-odd}, for any $\delta\ge0$, as follows; where the claim that they are well-defined is proved in Lemma \ref{lem:existance-eqodd}:
\begin{align}\label{eq:toeod}
&    (t_{\textup{DEO},1}(\delta),t_{\textup{DEO},2}(\delta)) 
\\
\nonumber&=\left\{\begin{array}{ll}
       (0,0),   & (1)\, |{D}_{\textup{DO}}(0,\atdt)|\le \delta \text{ and } |{D}_{\textup{PD}}(\atdo,0)|\le \delta;\\
      (\atdo,0),   & (2)\, |{D}_{\textup{DO}}(0,\atdt)|> \delta \text{ and } |{D}_{\textup{PD}}(\atdo,0)|\le \delta;\\
       (0,\acute{t}_{\textup{PD}}(\delta)),   & (3)\, |{D}_{\textup{DO}}(0,\atdt)|\le \delta \text{ and } |{D}_{\textup{PD}}(\atdo,0)|> \delta;\\
      \textup{solve}\{ ({D}_{\textup{DO}}(t_1,t_2),{D}_{\textup{PD}}(t_1,t_2))=(\delta,\delta)\},   
      &(4)\, {D}_{\textup{DO}}(0,\atdt)> \delta \text{ and } {D}_{\textup{PD}}(\atdo,0)> \delta;\\
    \textup{solve}\{ ({D}_{\textup{DO}}(t_1,t_2),{D}_{\textup{PD}}(t_1,t_2))=(\delta,-\delta)\},   
      &(5)\, {D}_{\textup{DO}}(0,\atdt)> \delta \text{ and } {D}_{\textup{PD}}(\atdo,0)< -\delta;\\
    \textup{solve}\{ ({D}_{\textup{DO}}(t_1,t_2),{D}_{\textup{PD}}(t_1,t_2))=(-\delta,\delta)\},   
      &(6)\, {D}_{\textup{DO}}(0,\atdt)< -\delta \text{ and } {D}_{\textup{PD}}(\atdo,0)> \delta;\\
      \textup{solve}\{ ({D}_{\textup{DO}}(t_1,t_2),{D}_{\textup{PD}}(t_1,t_2))=(-\delta,-\delta)\},   
      &(7)\, {D}_{\textup{DO}}(0,\atdt)< -\delta \text{ and } {D}_{\textup{PD}}(\atdo,0)< -\delta.
    \end{array}\right.
\end{align}
The seven cases above are mutually exclusive and collectively exhaustive. 
The first three cases, where at least one of 
$t_{\textup{DEO},1}$  and $t_{\textup{DEO},2}$ equals zero, correspond to 
(1) the unconstrained Bayes-optimal classifier, 
(2) the $\delta$-fair Bayes-optimal classifier under equality of opportunity, 
and 
(3) the $\delta$-fair Bayes-optimal classifier under predictive equality.
These satisfy the $\delta$-parity constraint for equalized odds. 
For the remaining four cases, both $t_{\textup{DEO},1}$  and $t_{\textup{DEO},2}$ 
are carefully selected to satisfy the hard constraint that for
$a\in\{\textup{DO},\textup{PD}\}$, $D_a(t_{\textup{DEO},1},t_{\textup{DEO},2})$ equals either $\delta$ or $-\delta$.

\subsubsection{Proof of Proposition \ref{prop:mon-D12}}

We prove claim (1), and 
claim (2) follows by similar arguments. 
For any
$0\le c_1\le c_2$, we have that $x\mapsto c_1/(c_2+x)$ is monotone non-increasing on $(-c_2,\infty)$ and  $x\mapsto c_1/(c_2-x)$ is monotone non-decreasing on $(-\infty,c_2)$. 
Let,
 for $a\in\{0,1\}$, $0\le c_{1,a} =p_{a,1}p_{a,0} +t_2 p_{a,1} < 2p_{a,1}p_{a,0} +t_2 p_{a,1} =c_{2,a}$.
 Fixing $t_2\in(-p_{1,0},p_{0,0})$, on one hand, 
 we have
 for
 $-p_{0,1}\le t_{1,1}\le t_{1,2}\le p_{1,1}$ that
    $t_{1,1}p_{1,0}\le t_{1,2}p_{1,0}\le  p_{1,1}p_{1,0}<2p_{1,1}p_{1,0} +t_2 p_{1,1}=c_{2,1} $. 
    Thus, since $ t_{1,1}\le t_{1,2}$,
   $$ T_1(t_{1,1},t_2)=
\frac{p_{1,1}p_{1,0} + t_2 p_{1,1}  }{2p_{1,1}p_{1,0} +  t_2 p_{1,1} -t_{1,1} p_{1,0}} \le \frac{p_{1,1}p_{1,0} + t_2 p_{1,1}  }{2p_{1,1}p_{1,0} +  t_2 p_{1,1} -t_{1,2} p_{1,0}} \le T_1(t_{1,2},t_2).$$
On the other hand, 
    $t_{1,2}p_{0,0}\ge t_{1,1}p_{0,0}\ge  -p_{0,1}p_{0,0}>t_2 p_{0,1}-  2p_{0,1}p_{0,0} =-c_{2,0}$. Thus
    $$ T_0(t_{1,1},t_2)=
\frac{p_{0,1}p_{0,0} - t_2 p_{0,1}  }{2p_{0,1}p_{0,0} - t_2 p_{0,1}  + t_{1,1} p_{0,0}} \ge \frac{p_{0,1}p_{0,0} - t_2 p_{0,1}  }{2p_{0,1}p_{0,0} - t_2 p_{0,1}  + t_{1,2} p_{0,0}} = T_0(t_{1,2},t_2).$$
Then, by definition, we have
\begin{align*}
& {D}_{\textup{DO}}(t_{1,1},t_2) = \P_{A=1,Y=1}\lsb\eta^Y_{A=1}(X)>T_1(t_{1,1},t_2)\rsb  -\P_{A=0,Y=1}\lsb\eta^Y_{A=0}(X)>T_0(t_{1,1},t_2)\rsb  \\
\ge &\P_{A=1,Y=1}\lsb\eta^Y_{A=1}(X)>T_1(t_{1,2},t_2)\rsb  -\P_{A=0,Y=1}\lsb\eta^Y_{A=0}(X)>T_0(t_{1,2},t_2)\rsb = {D}_{\textup{DO}}(t_{1,1},t_2) (t_{1,2},t_2);
\end{align*}
and
\begin{align*}& {D}_{\textup{PD}}(t_{1,1},t_2) = \P_{A=1,Y=0}\lsb\eta^Y_{A=1}(X)>T_1(t_{1,1},t_2)\rsb  -\P_{A=0,Y=0}\lsb\eta^Y_{A=0}(X)>T_0(t_{1,1},t_2)\rsb  \\
\ge &\P_{A=1,Y=0}\lsb\eta^Y_{A=1}(X)>T_1(t_{1,2},t_2)\rsb  -\P_{A=0,Y=0}\lsb\eta^Y_{A=0}(X)>T_0(t_{1,2},t_2)\rsb = {D}_{\textup{PD}}(t_{1,1},t_2) (t_{1,2},t_2).
\end{align*}
Thus, for fixed $t_2\in(-p_{1,0},p_{0,0})$, both $t \mapsto {D}_{\textup{DO}}(t,t_2)$ and 
$t \mapsto {D}_{\textup{PD}}(t,t_2)$ are monotone non-increasing on $[-p_{0,1},p_{1,1}]$.  
Continuity follows since, for $a\in\{0,1\}$,  $t \mapsto T_a(t,t_2)$ is continuous and $\eta_a(X)$ has density function on $\X$.

Moreover, since for any fixed $t_2\in (-p_{1,0},p_{0,0})$,
$ T_1(p_{1,1},t_2) = T_0(-p_{0,1})= 1,$
we have,
\begin{align*}
& \textup{DO}(-p_{0,1},t_2) = \P_{A=1,Y=1}\lsb\eta^Y_{A=1}(X)>T_1(-p_{0,1},t_2)\rsb  -\P_{A=0,Y=1}\lsb\eta^Y_{A=0}(X)>1)\rsb  \\
= &\P_{A=1,Y=1}\lsb\eta^Y_{A=1}(X)>T_1(-p_{0,1},t_2)\rsb  \ge 0;\\
& \textup{DO}(p_{1,1},t_2) = \P_{A=1,Y=1}\lsb\eta^Y_{A=1}(X)>1\rsb  -\P_{A=0,Y=1}\lsb\eta^Y_{A=0}(X)> T_0(p_{1,1},t_2)\rsb  \\
= &-\P_{A=0,Y=1}\lsb\eta^Y_{A=0}(X)> T_0(p_{1,1},t_2)\rsb  \le 0;\\
& \textup{PD}(-p_{0,1},t_2) = \P_{A=1,Y=0}\lsb\eta^Y_{A=1}(X)>T_1(-p_{0,1},t_2)\rsb  -\P_{A=0,Y=0}\lsb\eta^Y_{A=0}(X)>1)\rsb  \\
= &\P_{A=1,Y=2-j}\lsb\eta^Y_{A=1}(X)>T_1(-p_{0,1},t_2)\rsb  \ge 0;\\
& \textup{PD}(p_{1,1},t_2) = \P_{A=1,Y=0}\lsb\eta^Y_{A=1}(X)>1\rsb  -\P_{A=0,Y=0}\lsb\eta^Y_{A=0}(X)> T_0(p_{1,1},t_2)\rsb  \\
= &-\P_{A=0,Y=2-j}\lsb\eta^Y_{A=0}(X)> T_0(p_{1,1},t_2)\rsb  \le 0.
\end{align*}

\subsubsection{Proof of Theorem  \ref{thm-fb-eq-odd}}

The following lemmas are useful in proving Theorem \ref{thm-fb-eq-odd}.
The first result characterizes 
the relation between ${D}_{\textup{DO}}$ and ${D}_{\textup{PD}}$.

\begin{lemma}[Relation between ${D}_{\textup{DO}}$ and ${D}_{\textup{PD}}$]\label{conj}
Let $t$, $t_1$ and $t_2$ be three real numbers.
\begin{itemize}[]
\item (1) If $t_2\ge 0$ and 
${D}_{\textup{DO}}(t_1,t_2)={D}_{\textup{DO}}(t,0)$, we have
${D}_{\textup{PD}}(t_1,t_2)\le{D}_{\textup{PD}}(t,0).$

\item (2) If $t_2\le 0$ and 
${D}_{\textup{DO}}(t_1,t_2)={D}_{\textup{DO}}(t,0)$, we have
${D}_{\textup{PD}}(t_1,t_2)\ge{D}_{\textup{PD}}(t,0).$

\item (3)  If $t_1\ge 0$ and 
${D}_{\textup{PD}}(t_1,t_2)={D}_{\textup{PD}}(0,t)$, we have
${D}_{\textup{DO}}(t_1,t_2)\le{D}_{\textup{DO}}(0,t).$

\item (4) If $t_1\le 0$ and 
${D}_{\textup{PD}}(t_1,t_2)={D}_{\textup{PD}}(0,t)$, we have
${D}_{\textup{DO}}(t_1,t_2)\ge{D}_{\textup{DO}}(0,t).$

\end{itemize}
\end{lemma}

\begin{proof}[Proof of Lemma \ref{conj}]
For conciseness, we show the proofs for claim (1) 
Claims (2) through (4) can be verified using the same approach as for claim (1).

Recall  
$f_{\textup{Dis},t_1,t_2}$ and
$f_{\textup{Dis},t,0}$ from \eqref{eq:feqodd1}.
By Lemma \ref{lem:fb} adapted to equality of opportunity,
$f_{\textup{Dis},t,0}$ is Bayes-optimal with respect to DO, and
as $f_{\textup{Dis},t_1,t_2}$ is another classifier,
$$R(f_{\textup{Dis},{t},0}) =\min_{f\in\mF}\lbb R_f: \frac{t\cdot\textup{DO}(f)}{|t|} \le \frac{t\cdot\textup{DO}({t},0)}{|t|}\rbb 
\le R(f_{\textup{Dis},t_1,t_2}).$$
Moreover, since $\Pa(\eta_a(X)=T_a(t_1,t_2))=0$ when $\eta_a(X)$ is a continuous random variable with respect to $\Pa$,
for all $f'\in\argmin_{f\in\mF}\{ R_f: {t\cdot\textup{DO}(f)}/{|t|} \le {t\cdot\textup{DO}({t},0)}/{|t|}\} $, $f'(x,a) = f_{\textup{Dis},{t},0}(x,a)$ almost surely with respect to $\P_{X,A}$.

We have two cases: 
(i) $R(f_{\textup{Dis},{t},0}) = R(f_{\textup{Dis},t_1,t_2})$ and 
(ii) $R(f_{\textup{Dis},{t},0}) < R(f_{\textup{Dis},t_1,t_2})$.
\begin{itemize}[]
\item (i) When $R(f_{\textup{Dis},{t},0}) = R(f_{\textup{Dis},t_1,t_2})$,  we have $f_{\textup{Dis},t_1,t_2}(x,a) = f_{\textup{Dis},t,0}(x,a)$ almost surely with respect to $\P_{X, A}$. As a result,
$${D}_{\textup{PD}}(t_1,t_2)-{D}_{\textup{PD}}({t},0) =\int_{\X\times\A}\lsb f_{\textup{Dis},t_1,t_2}(x,a)- f_{\textup{Dis},{t},0}(x,a) \rsb\frac{(2a-1)(1-\eta_a(x))}{p_{a,0}}d\P_{X,A}(x,a)=0.
$$
\item (ii) When $R(f_{\textup{Dis},t,0}) < R(f_{\textup{Dis},t_1,t_2})$,
we first notice that 
since $t_2\ge 0$, by Lemma \ref{lem:fbeqodd},
\begin{equation}\label{eq:t1t2optimal}
 f_{\textup{Dis},t_1,t_2} \in  \argmin_{f\in\mF}\lbb R(f): \frac{t_1\cdot\textup{DO}(f)}{|t_1|}\le \frac{t_1 \cdot {D}_{\textup{DO}}(t_1,t_2)}{|t_1|},\,\,   {\textup{PD}(f)}\le {{D}_{\textup{PD}}(t_1,t_2)}
\rbb.  
\end{equation}
This implies ${D}_{\textup{PD}}(t_1,t_2)<{D}_{\textup{PD}}(t,0)$. 
Otherwise, $f_{\textup{Dis},{t},0}$ satisfies
$\textup{DO}(f_{\textup{Dis},{t},0},0)={D}_{\textup{DO}}(t_1,t_2)$, 
$\textup{PD}(f_{\textup{Dis},{t},0})\le 
{D}_{\textup{PD}}(t,0)
\le {{D}_{\textup{PD}}(t_1,t_2)}$ 
and $R(f_{\textup{Dis},{t},0})<R(f_{\textup{Dis},t_1,t_2})$, which contradicts \eqref{eq:t1t2optimal}.
\end{itemize}
In both cases, we have ${D}_{\textup{PD}}(t_1,t_2)\le {D}_{\textup{PD}}(t,0)$, finishing the proof. 
\end{proof}

The next result characterizes Bayes-optimal classifiers under equalized odds.

\begin{lemma}[Characterizing $f_{\textup{Dis},t_1,t_2}$ as Bayes-optimal Classifiers under Equalized Odds]\label{lem:fbeqodd}
For $(t_1,t_2)\in\R^2$ and $\delta\ge0$,
let $f_{\textup{Dis},t_1,t_2}$, ${D}_{\textup{DO}}(t)$, ${D}_{\textup{PD}}(t)$  be defined in \eqref{eq:feqodd1}, \eqref{eq:eodeo} and \eqref{eq:eodpe}, respectively. 
Then, recalling we use the convention that $0/0=0$,
$$ f_{\textup{Dis},t_1,t_2} \in \argmin_{f\in\mF}\lbb R(f):  \frac{t_1\cdot\textup{DO}(f)}{|t_1|}\le \frac{t_1\cdot{D}_{\textup{DO}}(t_1,t_2)}{|t_1|},\,\,   \frac{t_2\cdot\textup{PD}(f)}{|t_2|}\le \frac{t_2\cdot{D}_{\textup{PD}}(t_1,t_2)}{|t_2|} 
\rbb.   $$
Moreover, for all 
$$f'\in \argmin_{f\in\mF}\lbb R(f):  {t_1\cdot\textup{DO}(f)}/{|t_1|}\le {t_1\cdot{D}_{\textup{DO}}(t_1,t_2)}/{|t_1|},\,\,   {t_2\cdot\textup{PD}(f)}/{|t_2|}\le {t_2\cdot{D}_{\textup{PD}}(t_1,t_2)}/{|t_2|} 
\rbb,  $$ 
we have $f'=f_{\textup{Dis},t_1,t_2} $ almost surely with respect to $\P_{X,A}$.
In particular, 
if $t_1\cdot{D}_{\textup{DO}}(t_1,t_2)\ge 0$ and $t_2\cdot{D}_{\textup{PD}}(t_1,t_2)\ge 0$,
$$ f_{\textup{Dis},t_1,t_2} = \argmin_{f\in\mF}\lbb R(f):  |\textup{DO}(f)|\le |{D}_{\textup{DO}}(t_1,t_2)|,   |\textup{PD}(f)|\le |{D}_{\textup{PD}}(t_1,t_2)| 
\rbb.   $$\end{lemma}
\begin{proof}[Proof of Lemma \ref{lem:fbeqodd}]
Similar to Lemma \ref{lem:fb}, this result is a consequence of the generalized Neyman-Pearson lemma. If either $t_1=0$ or $t_2=0$, the result is an application of Lemma \ref{lem:fb}. Now, we assume $t_1\neq 0$ and $t_2\neq 0$.
Let, for all $x,a$, 
$\phi_0(x,a) = 2\eta^Y_{A=1}(x,a)-1$,  $\phi_1(x,a)={(2a-1)\eta_a(x)}/{p_{a,1}}$ and 
$\phi_2(x,a)={(2a-1)(1-\eta_a(x))}/{p_{a,0}}$.
Also, 
let, for all classifiers $f$, and $t\in\R$,
$\overline{\textup{DO}}_{t}(f) = t\cdot \textup{DO}(f) /|t|$ and
$\overline{\textup{PD}}_{t}(f) = t\cdot \textup{PD}(f) /|t|$.
 We have, for $(t_1,t_2)\in\R^2$ and all $x,a$, 
that with $f_{\textup{Dis},t_1,t_2}$ from \eqref{eq:feqodd1},
\begin{equation*} f_{\textup{Dis},t_1,t_2}(x,a) =  I\lsb\phi_0(x,a)> t_1\phi_1(x,a) + t_2\phi_2(x,a)\rsb =
I\lsb \phi_0(x,a)> |t_1| \frac{t_1\phi_1(x,a)}{|t_1|}+|t_2| \frac{t_2\phi_2(x,a)}{|t_2|}\rsb  
.\end{equation*}
Moreover, by Lemma \ref{lem:misclassification} and \eqref{eq:dopd}, we can write $\textup{Acc}(f)$, $\overline{\textup{DO}}_t(f)$ and $\overline{\textup{PD}}_t(f)$
as
\begin{align*} 
&   \textup{Acc}(f) = \int_\A\int_\X f(x,a) \phi_0(x,a) d\P_{X,A}(x,a) + \int_\A\int_\X (1-\eta_a(x))  d\P_{X,A}(x,a)  ;\\
&\overline{\textup{DO}}_{t_1}(f) =\frac{t_1}{|t_1|} \textup{DO}(f)=  \int_{\X\times\A}f(x,a)\frac{t_1\phi_1(x,a)}{|t_1|}d\P_{X,A}(x,a);\\
&\overline{\textup{PD}}_{t_2}(f) =\frac{t_2}{|t_2|} \textup{PD}(f)=  \int_{\X\times\A}f(x,a)\frac{t_2\phi_2(x,a)}{|t_2|}d\P_{X,A}(x,a).
\end{align*}
Define the sets of classifiers
\begin{align*}
&\mathcal{F}_{t_1,t_2,=} = \lbb f:   \overline{\textup{DO}}_{t_1}(f)=\frac{t_1\cdot{D}_{\textup{DO}}(t_1,t_2)}{|t_1|},\,\,   \overline{\textup{PD}}_{t_2}(f)=\frac{t_2\cdot{D}_{\textup{PD}}(t_1,t_2)}{|t_2|}\rbb;\\
&\mathcal{F}_{t_1,t_2,|\cdot|,\le} = \lbb f:  | \overline{\textup{DO}}_{t_1}(f)|\le |{D}_{\textup{DO}}(t_1,t_2)|,   |\overline{\textup{PD}}_{t_2}(f)|\le |{D}_{\textup{PD}}(t_1,t_2)|\rbb;\\
&\mathcal{F}_{t_1,t_2,\le} = \lbb f:  \overline{\textup{DO}}_{t_1}(f)\le \frac{t_1\cdot{D}_{\textup{DO}}(t_1,t_2)}{|t_1|},\,\,   \overline{\textup{PD}}_{t_2}(f)\le \frac{t_2\cdot{D}_{\textup{PD}}(t_1,t_2)}{|t_2|}\rbb. 
\end{align*}\
As $f_{\textup{Dis},t_1,t_2}\in \mathcal{F}_{t_1,t_2,=}$, by the generalized Neyman-Pearson lemma (Lemma \ref{NP_lemma}),
$$f_{\textup{Dis},t_1,t_2} \in \argmax_{f\in\mathcal{F}_{t_1,t_2,=}} \mathrm{Acc}(f).$$
Moreover, since $|t_1|\ge0$ and $|t_2|  \ge 0$, we have
$$f_{\textup{Dis},t_1,t_2} \in \argmax_{f\in\mathcal{F}_{t_1,t_2,\le}} \mathrm{Acc}(f)= \argmin_{f\in\mF}\lbb R(f): \frac{t_1\textup{DO}(f)}{|t_1|}\le \frac{t_1{D}_{\textup{DO}}(t_1,t_2)}{|t_1|},\,\,   \frac{t_2\textup{PD}(f)}{|t_2|}\le \frac{t_2{D}_{\textup{DO}}(t_1,t_2)}{|t_2|} 
\rbb. $$
In addition, since $\Pa(\eta_a(X)=T_a(t_1,t_2))=0$ when $\eta_a(X)$ is a continuous random variable with respect to $\Pa$, for all $f'\in \argmax_{f\in\mathcal{F}_{t_1,t_2,\le}} \mathrm{Acc}(f)$,  $f'=f_{\textup{Dis},t_1,t_2}$ almost surely with respect to $\P_{X,A}$.
Furthermore, if $t_1\cdot{D}_{\textup{DO}}(t_1,t_2)\ge 0$ and $t_2\cdot{D}_{\textup{PD}}(t_1,t_2)\ge 0$, we have
$f_{\textup{Dis},t_1,t_2}\in \mathcal{F}_{t_1,t_2,|\cdot|,\le} \subset \mathcal{F}_{t_1,t_2,=}.$
Thus
$$ f_{\textup{Dis},t_1,t_2} = \argmin_{f\in\mF}\lbb R(f):  |\textup{DO}(f)|\le |{D}_{\textup{DO}}(t_1,t_2)|,   |\textup{PD}(f)|\le |{D}_{\textup{PD}}(t_1,t_2)| 
\rbb.   $$
This finishes the proof.
\end{proof}

Next, we show that $\tiodod$ and $\tiodtd$ are well defined in cases (4) to (7). 
For conciseness, we
consider case (4), and the arguments for the other cases are similar. 
\begin{lemma}[The Quantities $\tiodod$ and $\tiodtd$ Are Well-Defined]\label{lem:existance-eqodd}
Let $\atdo$ and $\atdt$ be defined in \eqref{eq:atd}, Suppose that, for $a\in\{0,1\}$, $\eta_a(X)$ has a density function on $\X$. Then, for fixed $\delta\ge 0$, when ${D}_{\textup{DO}}(0,\atdt)>\delta$ and ${D}_{\textup{PD}}(0,\atdo)>\delta$, we have $\atdo>0$, $\atdt>0$, and there exist $(\tiodod,\tiodtd)\in [0,\atdo]\times [0,\atdt]$ such that $({D}_{\textup{DO}}(t_1,t_2),{D}_{\textup{PD}}(t_1,t_2))=(\delta,\delta)$.
\end{lemma}
\begin{proof}[Proof of Lemma \ref{lem:existance-eqodd}]
If $\atdo\le 0$, then by the definition of $\atdo$ in \eqref{eq:atd}, we have ${D}_{\textup{DO}}(0,0)\le \delta$ and $\textup{DO}(\atdo,0)\le \delta$. Moreover,
by Proposition \ref{prop:mon-D12},
${D}_{\textup{PD}}(0,0)\ge {D}_{\textup{PD}}(\atdo,0)>\delta$. 
Thus, it follows that 
 $\atdt>0$. By applying Proposition \ref{prop:mon-D12} again, ${D}_{\textup{DO}}(0,\atdt)\le {D}_{\textup{DO}}(0,0)\le \delta$, which is a contradiction. As a result, we have $\atdo>0$ and $\textup{DO}(\atdo,0)=\delta
 $. On the other hand, since ${D}_{\textup{DO}}(0,p_{0,0})\le\delta$, we have  $\atdt<p_{0,0}$. 
 Similarly, we can show $0<\atdo<p_{1,1}$ and
 ${D}_{\textup{PD}}(0,\atdt)=\delta
 $.
 
 Now, for any $t_2\in (0,\atdt)$,
we have ${D}_{\textup{DO}}(0,t_2)\ge {D}_{\textup{DO}}(0,\atdt)>\delta$ and
$\textup{DO}(p_{1,1},t_2)\le\delta$. By the continuity of ${D}_{\textup{DO}}(t,t_2)$ as a function of $t$, there exists $t_1\in(0,p_{1,1})$ such that ${D}_{\textup{DO}}(t_1,t_2)=\delta$. 
We then define   $Q_\delta: [0,\atdt]\to (0,p_{1,1})$ 
for all $t_2\in (0,p_{1,1})$ as
$$Q_\delta(t_2) = \inf_{t\in(0,p_{1,1})}\{ {D}_{\textup{DO}}(t,t_2) < \delta\}.$$ 
Clearly, we have $Q_\delta(0) \le \atdo$ and, since
${D}_{\textup{DO}}(0,\atdt)>\delta$, we have $Q_\delta(\atdt)>0$. 

Now, we consider the function 
$P$ defined on $[0,\atdt]$
$t_2\mapsto \P(t_2)=\textup{PD}(Q_\delta(t_2),t_2)$.
Then we have  $\P(0) ={D}_{\textup{PD}}(\atdo,0)>\delta$ and $\P(\atdt) =\textup{PD}(Q_\delta(\atdt),\atdt)\le {D}_{\textup{PD}}(0,\atdt)=\delta$. 
Thus, by the continuity of $D_{\textup{PD}}$, there exists  $\tiodtd\in[0,\atdo]$ such that $\textup{PD}(Q_\delta(\tiodtd),\tiodtd) =\delta$. 
Setting $\tiodod = Q_\delta(\tiodtd)$, we have
$\textup{DO}(\tiodod,\tiodtd)={D}_{\textup{PD}}(\tiodod,\tiodtd) =\delta$. 
Moreover, since $\tiodtd\le \atdt$ and ${D}_{\textup{PD}}(0,\atdt) ={D}_{\textup{PD}}(\tiodod,\tiodtd) =\delta$, we have $\tiodod\ge 0$. On the other hand, since $\tiodtd\ge 0$ and $\textup{DO}(\atdo,0) ={D}_{\textup{PD}}(\tiodod,\tiodtd) =\delta$, we have $\tiodod\le\atdo$.
This finishes the proof.
\end{proof}

We proceed with the proof of Theorem \ref{thm-fb-eq-odd}.
\begin{proof}[Proof of \Cref{thm-fb-eq-odd}]
We demonstrate the result in cases (1), (2), (4), and (5). The results in other cases can be verified similarly: 
case (3) is analogous to case (2), case (6) is analogous to case (5), and case (7) is analogous to case (4). 

The following facts follow directly from the definitions of $\atdo$ and $\atdt$ in \eqref{eq:atd}, 
and Proposition \ref{prop:mon-D12}:
    \begin{itemize}[]
        \item \qquad Fact (1). For $j\in\{\textup{DO},\textup{PD}\}$, $\delta\mapsto \acute{t}_j(\delta)$ is monotone non-decreasing.
        \item \qquad Fact (2). 
        For $j\in\{\textup{DO},\textup{PD}\}$,
        $\acute{t}_j>0$ if and only if $D_j(0,0)>\delta$. 
        Further, $\acute{t}_j<0$ if and only if $D_j(0,0)<-\delta$;
        \item \qquad Fact (3). If ${D}_{\textup{DO}}(0,0)>\delta$, $\textup{DO}(\atdt,0)=\delta$;
        If ${D}_{\textup{DO}}(0,0)<-\delta$, 
        then $\textup{DO}(\atdt,0)=-\delta$.
        Further, ${D}_{\textup{PD}}(0,0)>\delta$, then ${D}_{\textup{PD}}(0,\atdt)=\delta$.
        Finally, if ${D}_{\textup{PD}}(0,0)<-\delta$, then ${D}_{\textup{PD}}(0,\atdt)=-\delta$;
    \end{itemize}
    We now study the cases mentioned for Theorem \ref{thm-fb-eq-odd}.
\begin{itemize}[]
\item Case (1). $|{D}_{\textup{PD}}(\atdo,0)|\le \delta$ and $|{D}_{\textup{DO}}(0,\atdt)\le \delta.$

In this scenario, we have  $t_{\textup{DEO},1}(\delta)= t_{\textup{DEO},2}(\delta)=0$ and, for all $x,a$, $f_{\textup{Dis},0,0}(x,a)=f_{\textup{Dis},t_{\textup{DEO},1}(\delta),t_{\textup{DEO},2}(\delta)}(x,a) = I\lsb\eta_a(x)>1/2\rsb$ is the unconstrained Bayes-optimal classifier. As a result, we only need to prove $|{D}_{\textup{DO}}(0,0)|\le \delta$  and $|{D}_{\textup{PD}}(0,0)|\le \delta$.

If $\atdo>0$, 
then we have ${D}_{\textup{DO}}(0,0)>\delta$. 
By $|{D}_{\textup{DO}}(0,\atdt)|\le\delta$ and Proposition \eqref{prop:mon-D12}, 
we conclude $\atdt<0$, which leads to
${D}_{\textup{PD}}(0,0)<-\delta$. Finally, since $\atdo>0$, we see
$\textup{PD}(\atdo,0)\le {D}_{\textup{PD}}(0,0)<-\delta$, which contradicts $|{D}_{\textup{PD}}(0,0)|\le \delta$. Thus, we can conclude that $\atdo\le 0$. 

With analogous arguments, we can also demonstrate that $\atdo\ge 0$. Then, we must have 
$\atdo=0$. 
As a result, $|{D}_{\textup{DO}}(0,0)|\le \delta$  and $|{D}_{\textup{PD}}(0,0)|\le \delta$.

\item  Case (2). $|{D}_{\textup{DO}}(0,\atdt)|> \delta$ and $|{D}_{\textup{PD}}(\atdo,0)|\le\delta$.

In this scenario, we have  $t_{\textup{DEO},1} = \atdo = \argmin_t\{|t|:|\textup{DO}|\le \delta\}$ and $t_{\textup{DEO},2}=0$.  
On one hand, by Theorem \ref{thm-opt-dp} adapted for equality of opportunity, we have
    $$R(f_{\textup{Dis},t_{\textup{DEO},1}(\delta),t_{\textup{DEO},2}(\delta)}) = R(f_{\textup{Dis},\atdo,0}) = \min_{f\in\mF}\{ R(f): |\textup{DO}(f)|\le\delta\}.$$
    On the other hand, given that $|{D}_{\textup{PD}}(\atdo,0)|\le \delta$, we have
$$ f_{\textup{Dis},t_{\textup{DEO},1}(\delta),t_{\textup{DEO},2}(\delta)} =f_{\textup{Dis},\atdo,0} \in   \{f\in\mF: \max(|\textup{DO}(f)|,|\textup{PD}(f))|\le \delta\} \subset  \{f\in\mF: \textup{DO}(f) \le \delta\}.$$
It follows that
$$R(f_{\textup{Dis},t_{\textup{DEO},1}(\delta),t_{\textup{DEO},2}(\delta)})\ge \min_{f\in\mF}\{ R(f):  \max(|\textup{DO}(f)|,|\textup{PD}(f))|\le \delta\}  \ge \min_{f\in\mF}\{R(f):\textup{DO}(f) \le \delta\}.$$
Then, we can conclude that
$$f_{\textup{Dis},t_{\textup{DEO},1}(\delta),t_{\textup{DEO},2}(\delta)}) \in \argmin_{f\in\mF}\{R(f):  \max(|\textup{DO}(f)|,|\textup{PD}(f))|\le \delta\}.$$

\item Case (4)  ${D}_{\textup{DO}}(0,\atdt)> \delta$ and ${D}_{\textup{PD}}(\atdo,0)>\delta$

In this case,  $(t_{\textup{DEO},1}(\delta),t_{\textup{DEO},2}(\delta))$ is the solution of the equation $({D}_{\textup{DO}}(t_1,t_2),{D}_{\textup{PD}}(t_1,t_2)) = (\delta,\delta)$. 

Following the proof of Lemma \ref{lem:existance-eqodd}, we find that $\atdo>0$ and $\atdt>0$. 
This further implies $\textup{DO} (\atdo,0) =\textup{PD} (0,\atdo)=\delta$.
By claims (2) and (4) of Lemma \ref{conj}, we can conclude that $t_{\textup{DEO},1}(\delta)\ge 0$ and $t_{\textup{DEO},2}(\delta)\ge 0$. 
Otherwise, if $t_{\textup{DEO},2}(\delta)< 0$, since ${D}_{\textup{DO}}(t_{\textup{DEO},1}(\delta),t_{\textup{DEO},2}(\delta))=\delta = \textup{DO}(\atdo,0)$,
we  would have ${D}_{\textup{PD}}(t_{\textup{DEO},1}(\delta),t_{\textup{DEO},2}(\delta))\ge \textup{PD} (\atdo,0)>\delta$, which contradicts ${D}_{\textup{PD}}(t_{\textup{DEO},1}(\delta),t_{\textup{DEO},2}(\delta))=\delta$. 
A similar argument applies to $t_{\textup{DEO},1}(\delta)$.

Now, we have
$t_{\textup{DEO},1}(\delta)\cdot{D}_{\textup{DO}}(t_{\textup{DEO},1},t_{\textup{DEO},2}) =t_{\textup{DEO},1}(\delta)\cdot \delta\ge 0 $ and $t_{\textup{DEO},2}(\delta)\cdot{D}_{\textup{PD}}(t_{\textup{DEO},1},t_{\textup{DEO},2}) =t_{\textup{DEO},2}(\delta)\cdot \delta\ge 0 $. By Lemma \ref{lem:fbeqodd},
\begin{equation*}
f_{\textup{Dis},t_{\textup{DEO},1}(\delta),t_{\textup{DEO},2}(\delta)}\\
\in\argmin_{f\in\mF}\{R(f):  \max(|\textup{DO}(f)|,|\textup{PD}(f))|\le \delta\}.
\end{equation*}

\item Case (5) ${D}_{\textup{DO}}(0,\atdt)> \delta$ and ${D}_{\textup{PD}}(\atdo,0)<-\delta$.

In this case,  $(t_{\textup{DEO},1}(\delta),t_{\textup{DEO},2}(\delta))$ is the solution of the equation $({D}_{\textup{DO}}(t_1,t_2),{D}_{\textup{PD}}(t_1,t_2)) = (\delta,-\delta)$. 

We consider two cases: (i) $\atdo> 0$ and (ii) $\atdo\le 0$.
\begin{itemize}
\item[(i)] When $\atdo> 0$,  we have ${D}_{\textup{DO}}(0,0)> \delta$ and $\textup{DO}(\atdo,0)=\delta$.
Then, from result  (1) of Lemma \ref{conj}, 
it follows that $t_{\textup{DEO},1}(\delta)> 0$ and $t_{\textup{DEO},2}(\delta)< 0$. 
Otherwise, on one hand, if $t_{\textup{DEO},2}(\delta)\ge 0$, since $\textup{DO}(t_{\textup{DEO},1}(\delta),t_{\textup{DEO},2}(\delta))=\delta =\textup{DO}(\atdo,0)$, 
we have $ \textup{PD}(t_{\textup{DEO},1}(\delta),t_{\textup{DEO},2}(\delta))\le \textup{PD}(\atdo,0)<-\delta$, which is a contradiction.
On the other hand, if
$t_{\textup{DEO},1}(\delta)\le 0$, we have $t_{\textup{DEO},1}(\delta)\le 0\le \atdo$. 
Then, by the monotonicity properties of $t\mapsto {D}_{\textup{DO}}$ from Proposition \ref{prop:mon-D12} and the fact that
$\textup{DO}(\atdo,0)=\delta =\textup{DO}(t_{\textup{DEO},1}(\delta),t_{\textup{DEO},2}(\delta))$, 
we have $t_{\textup{DEO},2}(\delta)\ge 0$. 
We again reach a contradiction.

\item[(ii)] When $\atdo\le 0$, 
the argument is similar.
By Proposition \ref{prop:mon-D12}, we have ${D}_{\textup{PD}}(0,0)\le {D}_{\textup{PD}}(\atdo,0)<-\delta$, which implies $\atdt<0$ and ${D}_{\textup{PD}}(0,\atdt)=-\delta$.
Then, by claim (4) of Lemma \ref{conj}, we have $t_{\textup{DEO},1}(\delta)>0$ and $t_{\textup{DEO},2}(\delta)<0$.  
Otherwise, on one hand, if $t_{\textup{DEO},1}(\delta)\le 0$, since $\textup{PD}(t_{\textup{DEO},1}(\delta),t_{\textup{DEO},2}(\delta))=-\delta =\textup{PD}(0,\atdt)$, we would have $ \textup{DO}(t_{\textup{DEO},1}(\delta),t_{\textup{DEO},2}(\delta))\ge \textup{DO}(0,\atdt)>\delta$, which is a contradiction.
On the other hand, if
$t_{\textup{DEO},2}(\delta)\ge 0$, we would have $t_{\textup{DEO},2}(\delta)\ge 0\ge \atdo$. 
Then, by the monotonicity properties of $t\mapsto {D}_{\textup{DO}}$ from Proposition \ref{prop:mon-D12} and the fact that
$\textup{PD}(0,\atdt)=-\delta =\textup{PD}(t_{\textup{DEO},1}(\delta),t_{\textup{DEO},2}(\delta))$, we conclude$t_{\textup{DEO},1}(\delta))\le 0$; 
also a contradiction.
\end{itemize}
In both cases, we have
$t_{\textup{DEO},1}(\delta)\cdot{D}_{\textup{DO}}(t_{\textup{DEO},1},t_{\textup{DEO},2}) =t_{\textup{DEO},1}(\delta)\cdot \delta\ge 0 $ and $t_{\textup{DEO},2}(\delta)\cdot{D}_{\textup{PD}}(t_{\textup{DEO},1},t_{\textup{DEO},2}) =t_{\textup{DEO},2}(\delta)\cdot (-\delta)\ge 0 $. By Lemma \ref{lem:fbeqodd}, we conclude that
\begin{equation*}
f_{\textup{Dis},t_{\textup{DEO},1}(\delta),t_{\textup{DEO},2}(\delta)}\\
\in\argmin_{f\in\mF}\{R(f):  \max(|\textup{DO}(f)|,|\textup{PD}(f))|\le \delta\}.
\end{equation*}
\end{itemize}
The proof of Theorem \ref{thm-fb-eq-odd}
is complete.
\end{proof}

\subsection{Fair Bayes-optimal Classifier with a Multi-class Protected Attribute}
In this section, we consider demographic parity with a multi-class protected attribute. 
We assume that $A\in\A=\{1,2,...,|\A|\}$ for some integer $|\A|>2$.  Our theoretical
results concern perfect fairness. 
Under perfect fairness, we have $|\A|-1$ equality constraints, and the fair Bayes-optimal classifier can be derived using the Neyman-Pearson lemma.
However, for approximate fairness for a multi-class protected attribute, 
it is unknown ahead of time how many constraints become equalities, and how many are strict inequalities. 
A careful analysis of these two types of constraints is required, and we leave this to future work.
Recall that, a classifier satisfies demographic parity if
  $$\P(\widehat{Y}_f  = 1\mid A = a)=\P\lsb \widehat{Y}_f  = 1\mid A=1\rsb , \ \text{ for } \ \ a=2,...,\A.$$
Similar to the expression of DD in \eqref{exp:DD}, these constraints can be expressed as:
$$\int_{\mathcal{A}}\int_{\mathcal{X}}f(x,\ta) \frac{I(\ta=a)}{p_{a}}-\frac{I(\ta=1)}{p_{1}} d\P_{X,A}(x,\ta) \ \text{ for } \ \ a=2,...,\A.$$
Take $\phi_0(x,\ta)=2\eta_{\ta}(x)$ and $\phi_{a}(x,\ta)=I(\ta=a)/{p_{a}}-{I(\ta=1)}/p_{1}$ for $a=2,...,|\A|$, 
the generalized Neyman-Pearson lemma suggests us to consider the group-wise thresholding rule, for all $(x,\ta)\in\X\times\{0,1\}$,
$$
f(x,\tilde a,t_2,\ldots,t_\A)
= I\lsb  \phi_0(x,\ta)>\sum_{a=2}^{|\A|}t_a\phi_a(x,\ta)\rsb.
$$
This further leads to our formal Theorem \ref{thm-opt-dp-multi} in \Cref{ext}.

\subsubsection{Proof of Theorem \ref{thm-opt-dp-multi}}
We first demonstrate the existence of thresholds $\{t_{\textup{Dis},a}\}_{a=1}^{|\A|}$ that satisfy the condition required in \Cref{thm-opt-dp-multi}.

\begin{lemma}[Existence of Thresholds]\label{existence-multi}
Suppose that, for all $a\in\A$, $\eta_a(X)$ has a density function on $\X$. Then, there exist  
$\{t_{\textup{Dis},a}\}_{a=1}^{|\A|}$ such that 
$\sum_{a=1}^{|\A|} t_{\textup{Dis},a}=0$, and
for all $a \in \{2,\ldots,|\A|\}$, \eqref{eq:td-multi} holds.
\end{lemma}

\begin{proof}[Proof of Lemma \ref{existence-multi}]
Define the functions $\underline{Q}_a, \overline{Q}_a:[0,1]\to\R$, such that for $s\in[0,1]$,
\begin{align*}
\begin{split}
    \underline{Q}_a(s)&=\sup\left\{t: \P_{X\mid A=a}\left(\eta_a(X)>\frac12+\frac{t}{p_{a}}\right)>s\right\};  \\
\overline{Q}_a(s)&=\sup\left\{t:\P_{X\mid A=a}\left(\eta_a(X)>\frac12+\frac{t}{p_{a}}\right)\ge s\right\}.
\end{split}\end{align*} 
By definition, we have the following:
\begin{compactitem}
\item Both $\underline{Q}_a$ and $\overline{Q}_a$ are strictly monotonically decreasing, as $\eta_a(X)$ has a density function;
\item $\underline{Q}_a$ is right-continuous and $\underline{Q}_a$ is left-continuous;
\item For all $s_0\in[0,1]$, $\lim_{s\to s_0^-}\underline{Q}_a(s)=\lim_{s\to s_0^-}\underline{Q}_a(s)$ and $\lim_{s\to s_0^+}\underline{Q}_a(s)=\lim_{s\to s_0^+}\underline{Q}_a(s)$ with $\underline{Q}_a(s_0)\le \overline{Q}_a(s_0)$;
\item   For all $t\in[\underline{Q}_a(s),\overline{Q}_a(s)]$, 
$\P_{X\mid A=a}(\eta_a(X)>1/2+\frac{t}{2p_{a}})=s.$
\end{compactitem}
Now, we set
 $$s^\star=\sup\lbb s: \sum_{a=1}^{|\A|}\underline{Q}_a(s)>0\rbb=\sup\lbb s: \sum_{a=1}^{|\A|}\overline{Q}_a(s)>0\rbb.$$
  By the  right-continuity of $\sum_{a=1}^{|\A|}\underline{Q}_a(s)$ and
 the 
  left-continuity of $\sum_{a=1}^{|\A|}\overline{Q}_a(s)$, we have
  $$\sum_{a=1}^{|\A|}\underline{Q}_a(s^\star)\le 0 \le \sum_{a=1}^{|\A|}\overline{Q}_a(s^\star).$$
  Letting 
  $$t_{\textup{Dis},a}=\frac{\sum_{a=1}^{|\A|}\overline{Q}_{a}(s^\star)}{\sum_{a=1}^{|\A|}\overline{Q}_{a}(s^\star)-\sum_{a=1}^{|\A|}\underline{Q}_{a}(s^\star)}\underline{Q}_{a}(s^\star)-\frac{\sum_{a=1}^{|\A|}\underline{Q}_{a}(s^\star)}{\sum_{a=1}^{|\A|}\overline{Q}_{a}(s^\star)-\sum_{a=1}^{|\A|}\underline{Q}_{a}(s^\star)}\overline{Q}_{a}(s^\star),$$
  we have $\sum_{a=1}^{|\A|}t_{\textup{Dis},a}=0$.
  Moreover, as $t_{\textup{Dis},a}\in[\underline{Q}_{a}(s),\overline{Q}_{a}(s)]$ for all $a\in\A$, we have, for all $a\in\A$,
  $$\P_{X\mid A=a}\left(\eta_a(X)>1/2+\frac{t_{\textup{Dis},a}}{2p_{a}}\right)=s^\star.$$
    This completes the proof.
\end{proof}
We now proceed to the proof of \Cref{thm-opt-dp-multi}.
\begin{proof}[Proof of \Cref{thm-opt-dp-multi}]
Recall that, for $(x,\ta)\in\mathcal{X}\times \{0,1\}$, we denote, for $a=2,...,|\A|$,
$\phi_0(x,\widetilde{a})=2\eta_{\widetilde{a}}(x)-1$; and $\phi_a(x,\ta)={I(\ta=a)}/{p_{a}}-{I(\ta=1)}/{p_{1}}$.  Since $t_{\textup{Dis},1}=-\sum_{a=2}^{|\A|}t_{\textup{Dis},a}
$, We have, for all $x,\ta$,
\begin{align*}
&f^\star_{t_{\textup{Dis},1},...,t_{\textup{Dis},|\A|}}(x,\ta) = I\lsb\eta_{\ta}(x)>\frac12 + \frac{t_{\textup{Dis},\ta}}{2p_{\ta}}\rsb = I\lsb 
2\eta_{\ta}(x)-1>  \frac{t_{\ta}}{p_{\ta}}\rsb\\
=& I\lsb 
2\eta_{\ta}(x)-1>  \sum_{a=2}^{|\A|}t_a\lsb \frac{I(\ta=a)}{p_{a}}-\frac{I(\ta=1)}{p_1}\rsb\rsb = I\lsb \phi_0(x,\ta)>\sum_{t=2}^{|\A|}t_a\phi_a(x,\ta)\rsb.
\end{align*}
Moreover, for any classifier $f$, we can write its  misclassification rate and corresponding disparity measures ($\P_{X\mid A=a}(\widehat{Y}_f=1)-\P_{X\mid A=1}(\widehat{Y}_f=1)$) as 
\begin{align*}
\begin{split}
R(f)=\P(\widehat{Y}_f\neq Y)&=-\int_{\A}\int_{\X}f(x,\widetilde{a})\phi_0(x,\widetilde{a})d\P_{X,A}(x,\ta) + \int_{\A}\int_{\X} \eta_{\ta}(x)d\P_{X,A}(x,\ta)
\end{split}
\end{align*}
and, for $a=2,...,|\A|$,
\begin{equation}\label{cons-multi}
\P_{X\mid A=a}(\widehat{Y}_f=1)-\P_{X\mid A=1}(\widehat{Y}_f=1)=\int_{\A}\int_{\X}\phi_a(x,\widetilde{a})f(x,\ta)d\P_{X,A}(x,\widetilde{a}).
\end{equation}
Denote
$$\mathcal{F}_{D,=}=\left\{f:\int_{\A}\int_{\X}f(x,\ta)\left(\frac{I(\ta=a)}{p_{a}}-\frac{I(\ta=1)}{p_{1}}\right)d\P_{X,A}(x,\ta)=0, \ \ a=2,3,...,|\A|\right\},$$
By the construction of $\{t_{\textup{Dis},a}\}_{a=1}^{|\A|}$, we have $f_{t_{\textup{Dis},1},...,t_{\textup{Dis},|\A|}}^\star\in \mathcal{F}_{D,=}$.
Then, by the generalized Neyman Pearson lemma,
$$f_{t_{\textup{Dis},1},...,t_{\textup{Dis},|\A|}}^\star \in \argmax_{f\in\mathcal{F}_{D,=}}\lbb\int_{\A}\int_{\X}f(x,\widetilde{a})\phi_0(x,\widetilde{a})d\P_{X,A}(x,\ta)\rbb=\argmin_{f\in\mathcal{F}_{D,=}} R(f), $$
which completes the proof.
\end{proof}

\subsection{FUDS, FCSC and FPIR Algorithms for Extensions}
\subsubsection{{Algorithms for 
a Protected Attribute Unavailable at Test Time}}
\label{sec:ext:Anotobserve}
In this section, we
provide explicit algorithms
for the case
when the protected attribute, $A$, is unavailable during prediction.
We focus on {the common disparity measures of demographic parity, equality of opportunity, and predictive equality}.
As outlined in Proposition \ref{prop:ddpnoa},
this setting leads to a fairness constraint
which is linear, but not bilinear. Based on Corollary  \ref{cor-opt-dp} and Proposition \ref{prop:ddpnoa}, {the general form of a $\delta$-fair Bayes-optimal classifier with linear disparity measures in this setting is given by:}
\begin{equation}\label{eq:fair_bayes_opt_noA}
f^\star_{X,\textup{Dis},\delta}(x)=  I\left(\eta^Y(x)>\frac12 + \frac{t_{X,\textup{Dis}}(\delta)}{2}w_{X,\textup{Dis}}(x)\right),
\end{equation}

with \begin{equation}\label{eq:tddnoA}
    t_{X,\textup{Dis}}(\delta) = \argmin_t \lbb|t|: \lab\int_{\X} \lmb w_{X,\textup{Dis}}(x) \cdot  I\left(\eta^Y(x)>\frac12 + \frac{t}{2}w_{X,\textup{Dis}}(x)\right) \rmb d\P_X(x)\rab\le \delta\rbb.
\end{equation}
Moreover, the specific form of the 
linear weighting function $w_{X,\textup{Dis}}(x)$ 
for our disparity measures of demographic parity, equality of opportunity, and predictive equality, respectively,
is
\begin{align*}
&w_{\textup{DD}}(x)=\frac{\eta^A(x)}{p_{1}} -\frac{1-\eta^A(x)}{p_{0}} ;\qquad w_{\textup{DO}}(x)=\frac{\eta^Y_{A=1}(x)\eta^A(x)}{p_{1,1}}-\frac{\eta^Y_{A=0}(x)(1-\eta^A(x))}{p_{0,1}};\\
&w_{\textup{PD}}(x)=\frac{(1-\eta^Y_{A=1}(x))\eta^A(x)}{p_{1,0}}-\frac{(1-\eta^Y_{A=0}(x))(1-\eta^A(x))}{p_{0,0}}. 
\end{align*} 
In particular, denote the $\delta$-fair Bayes-optimal classifier under demographic parity as
$f^\star_{X,\textup{DD},\delta}(x)$.
By setting $w_{X,\textup{Dis}}(x) = w_{\textup{DD}}(x)$ in \ref{eq:fair_bayes_opt_noA} and \ref{eq:tddnoA}, respectively, we find it is of the form

\begin{align*}
f^\star_{X,\textup{DD},\delta}(x)=&I\left(\eta^Y(x)> \frac12 +\frac{t_{X,\textup{DD}}(\delta)(\eta^A(x)-p_1)}{2p_1p_0} \right),
\end{align*}
with 
\begin{equation*}
t_{X,\textup{DD}}(\delta) = \argmin_t \lbb|t|: \lab\int_{\X} \lmb \lsb\frac{\eta^A(x)-p_1}{p_1p_0}\rsb \cdot I\lsb
\eta^Y(x)>\frac12 +\frac{t(\eta^A(x)-p_1)}{2p_1p_0}\rsb\rmb d\P_X(x)\rab\le \delta\rbb.
\end{equation*}

Similarly, define the $\delta$-fair Bayes-optimal classifier under equality of opportunity, $f^\star_{X,\textup{DO},\delta}(x)$; and predictive equality, $f^\star_{X,\textup{PD},\delta}(x)$;
with associated thresholding parameters $t_{X,\textup{DO}}(\delta)$ and $t_{X,\textup{PD}}(\delta)$ by setting $w_{X,\textup{Dis}}(x) = w_{\textup{DO}}(x)$ and  $w_{X,\textup{Dis}}(x) = w_{\textup{PD}}(x)$, respectively. 

The following two theorems, proved below, parallel 
Theorems \ref{thm:FUDS} and \ref{thm:FCSC}, and provide 
a theoretical underpinning for our Fair Up-/Down-Sampling and Fair cost-sensitive Classification methods {in the attribute-blind setting.}  

\begin{theorem}[{Bayes-Optimal FUDS for Common Disparity Measures with an Unavailable Protected Attribute}]\label{thm:FUDS1}
For $\delta\ge 0$, let $\tPd_{\textup{DD}}$ be the distribution
satisfying $\tPd_{\widetilde{X}|\widetilde{A}=a,\widetilde{Y}=y}(x)=\P_{{X}|{A}=a,{Y}=y}(x)$ for all $x$, 
and for $(a,y)\in\{0,1\}^2$,
\begin{eqnarray}\nonumber
\tpd_{\textup{DD},a,y}=\tPd_{\textup{DD}}\lsb \widetilde{A}=a,\widetilde{Y}=y\rsb =c_{\textup{DD}}(\delta)\lsb1 + \frac{(2a-1)(1-2y)t_{X,\textup{DD}}(\delta)}{p_a}\rsb p_{a,y},
\end{eqnarray}
where $c_{\textup{DD}}(\delta)>0$ is  such that $\sum_{a\in\{0,1\}}\sum_{y\in\{0,1\}}\tpd_{\textup{DD},a,y} = 1$.
\\

Similarly, let $\tPd_{\textup{DO}}$ be the distribution
satisfying $\tPd_{\widetilde{X}|\widetilde{A}=a,\widetilde{Y}=y}(x)=\P_{{X}|{A}=a,{Y}=y}(x)$ for all $x$, 
and for $(a,y)\in\{0,1\}^2$,
\begin{eqnarray}\nonumber
\tpd_{\textup{DO},a,y}=\tPd_{\textup{DO}}\lsb \widetilde{A}=a,\widetilde{Y}=y\rsb =c_{\textup{DO}}(\delta)\lsb1 + \frac{(1-2a)yt_{X,\textup{DO}}(\delta)}{p_{a,y}}\rsb p_{a,y},
\end{eqnarray}
where $c_{\textup{DO}}(\delta)>0$ is  such that $\sum_{a\in\{0,1\}}\sum_{y\in\{0,1\}}\tpd_{\textup{DO},a,y} = 1$.  \\

Lastly, let $\tPd_{\textup{PD}}$ be the distribution
satisfying $\tPd_{\widetilde{X}|\widetilde{A}=a,\widetilde{Y}=y}(x)=\P_{{X}|{A}=a,{Y}=y}(x)$ for all $x$, 
and for $(a,y)\in\{0,1\}^2$,
\begin{eqnarray}\label{eq:paynoA_Linear} 
\tpd_{\textup{PD},a,y}=\tPd_{\textup{PD}}\lsb \widetilde{A}=a,\widetilde{Y}=y\rsb =c_{\textup{PD}}(\delta)\lsb1 + \frac{(2a-1)(1-y)t_{X,\textup{PD}}(\delta)}{p_{a,y}}\rsb p_{a,y},
\end{eqnarray}
where $c_{\textup{PD}}(\delta)>0$ is  such that $\sum_{a\in\{0,1\}}\sum_{y\in\{0,1\}}\tpd_{\textup{PD},a,y} = 1$. \\

Then, the unconstrained Bayes-optimal classifiers for $\tPd_{\textup{DD}}, \tPd_{\textup{DO}}$, and $\tPd_{\textup{PD}}$
are $\delta$-fair Bayes-optimal classifiers under demographic parity, equality of opportunity, and predictive equality, respectively, for $\P$.

\end{theorem}

Next consider the following fair cost-sensitive risks:
\begin{align}
    \nonumber R^{\textup{FCSC}}_{t_{X,\textup{DD}}(\delta)}(f)&=\sum_{a\in\{0,1\}}\sum_{y\in\{0,1\}}\lsb\frac12+\frac{(2a-1)(1-2y)t_{X,\textup{DD}}(\delta)}{2p_a}\rsb \P(\widehat{Y}_f=1-y,Y=y,A=a)\\
    \nonumber R^{\textup{FCSC}}_{t_{X,\textup{DO}}(\delta)}(f)&=\sum_{a\in\{0,1\}}\sum_{y\in\{0,1\}}\lsb\frac12+\frac{(1-2a)yt_{X,\textup{DO}}(\delta)}{2p_{a,y}}\rsb \P(\widehat{Y}_f=1-y,Y=y,A=a)\\
    R^{\textup{FCSC}}_{t_{X,\textup{PD}}(\delta)}(f)&=\sum_{a\in\{0,1\}}\sum_{y\in\{0,1\}}\lsb\frac12+\frac{(2a-1)(1-y)t_{X,\textup{PD}}(\delta)}{2p_{a,y}}\rsb \P(\widehat{Y}_f=1-y,Y=y,A=a)\label{eq:fcscnoA_Linear},
\end{align}
 and define the cost-sensitive classifiers
\begin{align*}
&f^{\textup{FCSC}}_{X,t_{X,\textup{DD}}(\delta)} \in \argmin R^{\textup{FCSC}}_{X,t_{X,\textup{DD}}(\delta)}(f); \quad f^{\textup{FCSC}}_{X,t_{X,\textup{DO}}(\delta)} \in \argmin R^{\textup{FCSC}}_{X,t_{X,\textup{DO}}(\delta)}(f); \\
&f^{\textup{FCSC}}_{X,t_{X,\textup{PD}}(\delta)} \in \argmin R^{\textup{FCSC}}_{X,t_{X,\textup{PD}}(\delta)}(f) .
\end{align*}

\begin{theorem}[{Bayes-Optimal FCSC for Common Disparity Measures with an Unavailable Protected Attribute}]\label{thm:FCSC1}
For any $\delta\ge 0$, we have that $f^{X,\textup{FCSC}}_{t_{X,\textup{DD}}(\delta)}, f^{X,\textup{FCSC}}_{t_{X,\textup{DO}}(\delta)}$ and $f^{X,\textup{FCSC}}_{t_{X,\textup{PD}}(\delta)}$ are $\delta$-fair Bayes-optimal classifiers under demographic parity, equality of opportunity, and predictive parity, respectively, for $\P$. 
\end{theorem}

Based on these two theorems, 
pre- and in-processing algorithms analogous to those from \Cref{sec:alg}
can be designed for the case that $A$ is not used at test time. For the FUDS and FCSC algorithms, the main difference is that now we must fit a classifier without using the protected attribute $A$ as a predictor on 
the up- and down-sampled dataset $S_t$ for FUDS;
and on the original dataset with a cost-sensitive risk for FSCS. 

For the FPIR algorithm, we construct a plug-in estimator $\widehat{w}_{X,\textup{Dis}}$ 
of the  weighting function 
$w_{X,\textup{Dis}}$ 
by constructing estimates of $\eta^Y(x)=\P(Y=1|X=x)$ and additional plug-in estimators of $\eta^A$ for demographic parity, and $\eta^Y_{A=0}, \eta^Y_{A=1}$ for equality of opportunity and predictive equality.

For all methods, we also now estimate the disparity level 
$\widehat{f}^{\textup{FUDS,no}A}_{t}$ as 
\begin{equation}\label{dise_noA}
  \widehat{\textup{Dis}}^{\textup{FUDS, no}A}(t)=\frac1{n_1}\sum_{j=1}^{n_1} \widehat{f}^{\textup{FUDS,no}A}_{t}(\xoj) \widehat{w}_{X,\textup{Dis}}(\xoj)-\frac1{n_0}\sum_{j=1}^{n_0} \widehat{f}^{\textup{FUDS,no}A}_{t}(\xzj) \widehat{w}_{X,\textup{Dis}}(\xzj)
\end{equation} 

where $\widehat{w}_{X,\textup{Dis}}(x_{a,j}) = \widehat{w}_{\textup{Dis}}(x_{a,j},a) $ for $j\in[n_a]$ and $a\in\{0,1\}$. 

\begin{algorithm} 
   \caption{{Fair Up-/Down-Sampling (FUDS) for Common Disparity Measures with an Unavailable Protected Attribute}}
   \label{alg:FUDS_noA}

   \KwIn {Disparity measure either Demographic parity, Equality of opportunity, Predictive equality; Step size $\alpha>0$; Disparity Level $\delta\ge0$; Error tolerance level $\ep>0$; Dataset $S=S_1\cup S_0$ with  $S=\{x_{i},a_i,y_{i}\}_{i=1}^{n}$, $S_1=\{x_{1,i},y_{1,i}\}_{i=1}^{n_1}$ and $S_0=\{x_{0,i},y_{0,i}\}_{i=1}^{n_0}$; $n_{a,y}  = \#\{a_i=a,y_i=y\}$.}
   \BlankLine
\hrule
   \BlankLine

\textbf{Disparity Estimation sub-routine: Construct $\widehat{f}^{\textup{FUDS,no}A}_t$ and estimate $\textup{Dis}(\widehat{f}^{\textup{FUDS,no}A}_t)$ for a given threshold parameter $t$:}

 1. For all $a,y$ and for all $a$, estimate $p_{a,y}$ and $p_a$ by their empirical probabilities \begin{align*}
     \widehat p_{a,y} = n_{a,y}/n; \qquad\widehat p_a= (n_{a,0} + n_{a,1})/n.
 \end{align*}
 
 2. For all $a,y$ and for the specified disparity measure, let $\widehat{\widetilde{p}}^{\,t}_{a,y}$ be as in \eqref{eq:paynoA_Linear} using the empirical probabilities $\widehat p_a$ and $\widehat p_{a,y}$.

 3. Apply up- and down-sampling to generate a new dataset $S_{t}$ with proportions $\widehat{\widetilde{p}}_{a,y}^{\,t}$ for all $a,y$. 
 
 4. Fit classifier $\widehat{f}^{\textup{FUDS,no}A}_{t}$ on the dataset $S_t$ using any method that does not use the protected attribute $A$ as a predictor
 
 5. Estimate the disparity level of $\widehat{f}^{\textup{FUDS,no}A}_{t}$ as in \eqref{dise_noA}.
 \BlankLine
 \hrule
 \BlankLine

\textbf{Estimate the $\delta$-fair Bayes-optimal classifier:}

   Run Disparity Estimation sub-routine with $t=0$.
    
  \uIf{$\left|\widehat{\textup{Dis}}^{\textup{FUDS,no}A}({0})\right|\le \delta$} {
 $\widehat{t}_{\textup{Dis}}(\delta)=0$.}
  \Else{
  \uIf{$\widehat{\textup{Dis}}^{\textup{FUDS,no}A}({0})> \delta$}{ $\delta'=\delta$; $t_{\min}=0$, $t_{\max}=1$.}
  \Else {$\delta = -\delta$; $t_{\min}=-1$, $t_{\max}=0$.}
\While {$t_{\max}-t_{\min}>\ep$}
        {$t =  (t_{\max}-t_{\min})/2;$ 
        
     Run Disparity Estimation sub-routine with current $t$.
     
 \uIf {$\widehat{\textup{Dis}}^{\textup{FUDS,no}A}(t) > \delta$} {$t_{\max} = t$}
 \Else {$t_{\min} = t$}}
$\widehat{t}_{\textup{Dis}}(\delta)=t.$}

 {\bfseries Output:}  
$\widehat{f}_{\textup{Dis},\delta}^{\textup{no}A}=\widehat{f}^{\textup{FUDS,no}A}_{\widehat{t}_{\textup{Dis}}(\delta)}.$
 \BlankLine

\end{algorithm}

\begin{algorithm}
   \caption{{Fair cost-sensitive Classification (FCSC) for Common Disparity Measures with an Unavailable Protected Attribute}}
   \label{alg:FCSC_noA}

   \KwIn {Disparity measure either Demographic parity, Equality of opportunity, Predictive equality; Step size $\alpha>0$; Disparity Level $\delta\ge0$, Error tolerance level $\ep>0$; Dataset $S=S_1\cup S_0$ with  $S=\{x_{i},a_i,y_{i}\}_{i=1}^{n}$, $S_1=\{x_{1,i},y_{1,i}\}_{i=1}^{n_1}$ and $S_0=\{x_{0,i},y_{0,i}\}_{i=1}^{n_0}$; $n_{a,y}  = \#\{a_i=a,y_i=y\}$; Step size $\alpha\ge 0$.}
   \BlankLine

\textbf{Disparity Estimation sub-routine: Construct $\widehat{f}^{\textup{FCSC,no}A}_t$ and estimate $\textup{Dis}(\widehat{f}^{\textup{FCSC,no}A}_t)$ for a given threshold parameter $t$:}

 1. For all $a,y$ and for all $a$, estimate $p_{a,y}$ and $p_a$ by their empirical probabilities \begin{align*}
     \widehat p_{a,y} = n_{a,y}/n; \qquad\widehat p_a= (n_{a,0} + n_{a,1})/n.
 \end{align*}
 
 2. For all $a,y$ and for the specified disparity measure, read off the cost-sensitive misclassification risk $\widehat c_{a,y}(t)$ from \eqref{eq:fcscnoA_Linear}.   

 3. Use any cost-sensitive classification method to fit $\widehat{f}^{\textup{FCSC,no}A}_{t}(\cdot)$ on $S$. 
 
4. Estimate the disparity level of $\widehat{f}^{\textup{FCSC,no}A}_{t}$ as in \eqref{dise_noA} with $\widehat{f}^{\textup{FCSC,no}A}_{t}$ instead of $\widehat{f}^{\textup{FUDS,no}A}_{t}$.
 \BlankLine
 \hrule
 \BlankLine

  \begin{multicols}{2}
  {
\textbf{Estimate the fair Pareto frontier:}
 \BlankLine
\hrule
   \BlankLine
   Run Disparity Estimation sub-routine with $t=0$.
    
\uIf { $\widehat{\textup{Dis}}^{\textup{FCSC}} (0)>0$} {$\alpha' = \alpha$}
 \Else { $\alpha' = -\alpha$}
 \While{ $\widehat{\textup{Dis}}^{\textup{FCSC}}(t)\cdot \widehat{\textup{Dis}}^{\textup{FCSC}} (0) > 0$}
 { $t = t+\alpha'$. 
 
Run Disparity Estimation sub-routine with current $t$.
 }
 {\bfseries Output:}  $\{\widehat{f}^{\textup{FCSC,no}A}_{t}\}_{t=0,1,2,...}$.}
\columnbreak

\textbf{Estimate the $\delta$-fair Bayes-optimal classifier:}
 \BlankLine
\hrule
   \BlankLine
   Run Disparity Estimation sub-routine with $t=0$.
  \uIf{$\left|\widehat{\textup{Dis}}^{\textup{FCSC}}({0})\right|\le \delta$} {
 $\widehat{t}_{\textup{Dis}}(\delta)=0$.}
  \Else{
  \uIf{$\widehat{\textup{Dis}}^{\textup{FCSC}}({0})> \delta$}{ $\delta'=\delta$; $t_{\min}=0$, $t_{\max}=1$.}
  \Else {$\delta = -\delta$; $t_{\min}=-1$, $t_{\max}=0$.}
\While {$t_{\max}-t_{\min}>\ep$}
        {$t =  (t_{\max}-t_{\min})/2;$ 
        
     Run Disparity Estimation sub-routine with current $t$.
     
 \uIf {$\widehat{\textup{Dis}}^{\textup{FCSC}}(t) > \delta$} {$t_{\max} = t$}
 \Else {$t_{\min} = t$}}
$\widehat{t}_{\textup{Dis}}(\delta)=t.$}

 {\bfseries Output:}  
$\widehat{f}_{\textup{Dis},\delta}^{\textup{no}A}=\widehat{f}^{\textup{FCSC,no}A}_{\widehat{t}_{\textup{Dis}}(\delta)}.$
  \end{multicols}
 \BlankLine

\end{algorithm}

\begin{algorithm}
   \caption{{Fair Plug-in Rule (FPIR) for Common Disparity Measures with an Unavailable Protected Attribute}}
   \label{alg:FPIR_noA}

   \KwIn {Disparity measure either Demographic Parity, Equality of Opportunity, Predictive Equality; Step size $\alpha>0$; Disparity level $\delta\ge0$; Error tolerance level $\ep>0$; Dataset $S=S_1\cup S_0$ with  $S=\{x_{i},a_i,y_{i}\}_{i=1}^{n}$, $S_1=\{x_{1,i},y_{1,i}\}_{i=1}^{n_1}$ and $S_0=\{x_{0,i},y_{0,i}\}_{i=1}^{n_0}$; $n_{a,y}  = \#\{a_i=a,y_i=y\}$.}
   \BlankLine

   \BlankLine   {\textbf{Step 1}:} Construct estimates $\widehat{\eta}^Y, \widehat{\eta}^A$ of $\eta^Y, \eta^A$, respectively, using any approach. If the disparity measure is Equality of Opportunity or Predictive Equality, further construct estimates $\widehat{\eta}_1, \widehat{\eta}_0$ for $\eta^Y_{A=1}, \eta^Y_{A=0}$, respectively, using any approach. 
   \BlankLine

  {\textbf{Step 2}:  Estimate the the fair Pareto frontier or the $\delta$-fair Bayes-optimal classifier: 
}

\textbf{Disparity Estimation sub-routine: Construct $\widehat{f}^{\textup{FPIR,no}A}_t$ and estimate $\textup{Dis}(\widehat{f}^{\textup{FPIR,no}A}_t)$ with threshold parameter  $t$:}

 1. For all $a,y$ and for all $a$, estimate $p_{a,y}$ and $p_a$ by their empirical probabilities \begin{align*}
     \widehat p_{a,y} = n_{a,y}/n; \qquad\widehat p_a= (n_{a,0} + n_{a,1})/n.
 \end{align*}

2. Let $\widehat{f}^{\mathrm{FPIR,no}A}_t$ be defined by 
$\widehat{f}^{\mathrm{FPIR,no}A}_t(x,a) = I(\widehat{\eta}^Y(x)> \frac12 + \frac{t}{2} \widehat w_{X,\textup{Dis}}(x))$ for all $x,a$

3. Evaluate the disparity level of $\widehat{f}^{\textup{FPIR,no}A}_{t}$ as in \eqref{dise_noA} with $\widehat{f}^{\textup{FPIR,no}A}_{t}$ instead of $\widehat{f}^{\textup{FUDS,no}A}_{t}$.
 
 \BlankLine
 \hrule
 \BlankLine

  \begin{multicols}{2}
  {
\textbf{Estimate the fair Pareto frontier:}
 \BlankLine
\hrule
   \BlankLine
   Run Disparity Estimation sub-routine with $t=0$.
    
\uIf { $\widehat{\textup{Dis}}^{\textup{FPIR}} (0)>0$} {$\alpha' = \alpha$}
 \Else { $\alpha' = -\alpha$}
 \While{ $\widehat{\textup{Dis}}^{\textup{FPIR}}(t)\cdot \widehat{\textup{Dis}}^{\textup{FPIR}} (0) > 0$}
 { $t = t+\alpha'$. 
 
Run Disparity Estimation sub-routine with current $t$.
 }
 {\bfseries Output:}  $\{\widehat{f}^{\textup{FPIR,no}A}_{t}\}_{t=0,1,2,...}$.}
\columnbreak

\textbf{Estimate the $\delta$-fair Bayes-optimal classifier:}
 \BlankLine
\hrule
   \BlankLine
   Run Disparity Estimation sub-routine with $t=0$.
    
  \uIf{$\left|\widehat{\textup{Dis}}^{\textup{FPIR}}({0})\right|\le \delta$} {
 $\widehat{t}_{\textup{Dis}}(\delta)=0$.}
  \Else{
  \uIf{$\widehat{\textup{Dis}}^{\textup{FPIR}}({0})> \delta$}{ $\delta'=\delta$; $t_{\min}=0$, $t_{\max}=1$.}
  \Else {$\delta = -\delta$; $t_{\min}=-1$, $t_{\max}=0$.}
\While {$t_{\max}-t_{\min}>\ep$}
        {$t =  (t_{\max}-t_{\min})/2;$ 
        
Run Disparity Estimation sub-routine with current $t$.
     
 \uIf {$\widehat{\textup{Dis}}^{\textup{FPIR}}(t) > \delta$} {$t_{\max} = t$}
 \Else {$t_{\min} = t$}}
$\widehat{t}_{\textup{Dis}}(\delta)=t.$}

 {\bfseries Output:}  
$\widehat{f}_{\textup{Dis}^{\textup{no}A},\delta}=\widehat{f}^{\textup{FPIR,no}A}_{\widehat{t}_{\textup{Dis}}(\delta)}.$
  \end{multicols}
 \BlankLine
\end{algorithm}

\begin{proof}[Proof of Theorem \ref{thm:FUDS1}]
{We begin by proving the result for demographic parity.}
By Bayes' theorem, we have for all $x$ that
\begin{align*}
  &\eta^Y(x)=\P\lsb Y=1|X=x\rsb = \P\lsb Y=1,A=1|X=x\rsb +\P\lsb Y=1,A=0|X=x\rsb \\
  &=  \frac{p_{1,1}d\P_{X\mid A=1,Y=1}(x)}{d\P_X(x)}  +\frac{p_{0,1}d\P_{X\mid A=0,Y=1}(x)}{d\P_X(x)}\\
   &=  \frac{p_{1,1}d\P_{X\mid A=1,Y=1}(x)+p_{0,1}d\P_{X\mid A=0,Y=1}(x)}{p_{1,1}d\P_{X\mid A=1,Y=1}(x) + p_{1,0}d\P_{X\mid A=1,Y=0}(x) + p_{0,1}d\P_{X\mid A=0,Y=1}(x)+p_{0,0}d\P_{X\mid A=0,Y=0}(x)}   .
\end{align*}
Denoting $\tPd_{\textup{DD}}$ as $\tPd$ and $\tpd_{\textup{DD},a,y}$ as $\tpd_{a,y}$ for $(a,y)\in\{0,1\}^2$, we similarly have for all $x$ that
\begin{align*}
  &{\tetady}(x)=\tPd\lsb \tY=1|\tX=x\rsb = \tPd\lsb \tY=1,\tA=1|\tX=x\rsb +\tPd \lsb \tY=1,\tA=0|\tX=x\rsb \\
  &=  \frac{\tpd_{1,1}d\tPd_{\tX|\tA=1,\tY=1}(x)+\tpd_{0,1}d\tPd_{\tX|\tA=0,\tY=1}(x)}{\tpd_{1,1}d\tPd_{\tX|\tA=1,\tY=1}(x)+\tpd_{1,0}d\tPd_{\tX|\tA=1,\tY=0}(x)+\tpd_{0,1}d\tPd_{\tX|\tA=0,\tY=1}(x)+\tpd_{0,0}d\tPd_{\tX|\tA=0,\tY=0}(x)}.
\end{align*}
Moreover,
\begin{align*}
  &\eta^A(x)=\P\lsb A=1|X=x\rsb = \P\lsb A=1,Y=1|X=x\rsb +\P\lsb A=1,Y=0|X=x\rsb \\
  &=  \frac{p_{1,1}d\P_{X\mid A=1,Y=1}(x)}{d\P_X(x)}  +\frac{p_{1,0}d\P_{X\mid A=0,Y=1}(x)}{d\P_X(x)}\\
     &=  \frac{p_{1,1}d\P_{X\mid A=1,Y=1}(x)+p_{1,0}d\P_{X\mid A=1,Y=0}(x)}{p_{1,1}d\P_{X\mid A=1,Y=1}(x) + p_{1,0}d\P_{X\mid A=1,Y=0}(x) + p_{0,1}d\P_{X\mid A=0,Y=1}(x)+p_{0,0}d\P_{X\mid A=0,Y=0}(x)}. 
\end{align*}
Then, by our construction, for all $x$,
\begin{align}
\nonumber&f_{X,\mathrm{DD}, t_{X,\textup{DD}}(\delta)}^{\textup{FUDS}}(x)= I\lsb\tetady(x)>\frac12\rsb\\
\nonumber&=I\lsb \frac{\tpd_{1,1}d\tPd_{\tX|\tA=1,\tY=1}(x)+\tpd_{0,1}d\tPd_{\tX|\tA=0,\tY=1}(x)}{\tpd_{1,1}d\tPd_{\tX|\tA=1,\tY=1}(x)+\tpd_{1,0}d\tPd_{\tX|\tA=1,\tY=0}(x)+\tpd_{0,1}d\tPd_{\tX|\tA=0,\tY=1}(x)+\tpd_{0,0}d\tPd_{\tX|\tA=0,\tY=0}(x)}>\frac12\rsb\\ \label{eq:FUDS1_all}
&=I\lsb  {\tpd_{1,1}d\tPd_{\tX|\tA=1,\tY=1}(x)+\tpd_{0,1}d\tPd_{\tX|\tA=0,\tY=1}(x)>\tpd_{1,0}d\tPd_{\tX|\tA=1,\tY=0}(x)+\tpd_{0,0}d\tPd_{\tX|\tA=0,\tY=0}(x)}\rsb.
\end{align}
This also equals
\begin{align}
\nonumber&I\left( c_{\textup{DD}}(\delta)\cdot\lsb 1-\frac{t_{X,\textup{DD}}(\delta)}{p_1}\rsb p_{1,1} d\P_{X\mid A=1,Y=1}(x)+c_{\textup{DD}}(\delta)\cdot\lsb 1+\frac{t_{X,\textup{DD}}(\delta)}{p_0}\rsb p_{0,1} d\P_{X\mid A=0,Y=1}(x)\right.\\ \label{eq:DDFUDS1}
&>\left.c_{\textup{DD}}(\delta)\cdot\lsb 1+\frac{t_{X,\textup{DD}}(\delta)}{p_1}\rsb p_{1,0} d\P_{X\mid A=1,Y=0}(x)+c_{\textup{DD}}(\delta)\cdot\lsb 1-\frac{t_{X,\textup{DD}}(\delta)}{p_0}\rsb p_{0,0} d\P_{X\mid A=0,Y=0}(x)
\right). 
\end{align}
Denoting $g_{a,y}(x)=p_{a,y}d\P_{X\mid A=a,Y=y}(x)$, \eqref{eq:DDFUDS1} further equals
\begin{align*}
&I\left( g_{1,1}(x)+g_{0,1}(x) - g_{1,0}(x)-g_{0,0}(x)>\frac{t_{X,\textup{DD}}(\delta)}{p_1}(g_{1,1}(x)+g_{1,0}(x))-\frac{t_{X,\textup{DD}}(\delta)}{p_0}(g_{0,1}(x)+g_{0,0}(x))
\right)\\
&=I\bigg( 2(g_{1,1}(x)+g_{0,1}(x)) - \lsb g_{1,1}(x)+g_{0,1}(x) +g_{1,0}(x)+g_{0,0}(x)\rsb
\\
&>\lsb \frac{t_{X,\textup{DD}}(\delta)}{p_1} +\frac{t_{X,\textup{DD}}(\delta)}{p_0} \rsb(g_{1,1}(x)+g_{1,0}(x))-\frac{t_{X,\textup{DD}}(\delta)}{p_0}(g_{1,1}(x)+g_{1,0}(x)+g_{0,1}(x)+g_{0,0}(x))
\bigg).
\end{align*}
Finally, this equals
\begin{align*}
&I\left( \frac{2(g_{1,1}(x)+g_{0,1}(x))}{g_{1,1}(x)+g_{0,1}(x) +g_{1,0}(x)+g_{0,0}(x)}
> \frac{t_{X,\textup{DD}}(\delta)(g_{1,1}(x)+g_{1,0}(x))}{p_1p_0\lsb g_{1,1}(x)+g_{0,1}(x) +g_{1,0}(x)+g_{0,0}(x)\rsb}+1
-\frac{t_{X,\textup{DD}}(\delta)}{p_0}
\right)\\
&=I\left(2\eta^Y(x)>\frac{t_{X,\textup{DD}}(\delta)}{p_1p_0} \eta^A(x)+1- \frac{t_{X,\textup{DD}}(\delta)}{p_0}\right) = I\lsb
\eta^Y(x)>\frac12+\frac{t_{X,\textup{DD}(\delta)}\lsb\eta^A(x)-p_1\rsb}{2p_1p_0}\rsb = f^\star_{X,\textup{DD},\delta}(x).
\end{align*}
This completes the proof for demographic parity. We show the result for equality of opportunity, as the proof for predictive equality is analogous to this setting. Denote $\tPd_{\textup{DO}}$ as $\tPd$ and $\tpd_{\textup{DO},a,y}$ as $\tpd_{a,y}$ for $(a,y)\in\{0,1\}^2$. Then by the exact same derivation for Bayes' rule, we arrive at \eqref{eq:FUDS1_all}. We then have, 
\begin{align}
\nonumber&f_{X,\mathrm{DO}, t_{X,\textup{DO}}(\delta)}^{\textup{FUDS}}(x)\\ 
\nonumber &=I\lsb  {\tpd_{1,1}d\tPd_{\tX|\tA=1,\tY=1}(x)+\tpd_{0,1}d\tPd_{\tX|\tA=0,\tY=1}(x)>\tpd_{1,0}d\tPd_{\tX|\tA=1,\tY=0}(x)+\tpd_{0,0}d\tPd_{\tX|\tA=0,\tY=0}(x)}\rsb \\
\nonumber&=I\left( c_{\textup{DO}}(\delta)\cdot\lsb 1-\frac{t_{X,\textup{DO}}(\delta)}{p_{1,1}}\rsb p_{1,1} d\P_{X\mid A=1,Y=1}(x)+c_{\textup{DO}}(\delta)\cdot\lsb 1+\frac{t_{X,\textup{DO}}(\delta)}{p_{0,1}}\rsb p_{0,1} d\P_{X\mid A=0,Y=1}(x)\right.\\
\nonumber >&\left.c_{\textup{DO}}(\delta)\cdot p_{1,0} d\P_{X\mid A=1,Y=0}(x)+c_{\textup{DO}}(\delta)\cdot p_{0,0}  d\P_{X\mid A=0,Y=0}(x)\right). 
\end{align}
This also equals
\begin{align}
\nonumber&=I\left( (1-\frac{t_{X,\textup{DO}}(\delta)}{p_{1,1}})g_{1,1}(x) + (1+\frac{t_{X,\textup{DO}}(\delta)}{p_{0,1}})g_{0,1}(x) > g_{1,0}(x) + g_{0,0}(x)\right)\\
\nonumber&=I\left( \frac12 (g_{1,1}(x) + g_{0,1}(x)) > \frac12 (g_{1,0}(x) + g_{0,0}(x)) + \frac12 t_{X,\textup{DO}}(\delta) \lsb \frac{g_{1,1}(x)}{p_{1,1}} - \frac{g_{0,1}(x)}{p_{0,1}} \rsb \right)
\end{align}
Finally, this equals
\begin{align}
\nonumber&= I\left( \frac{g_{1,1}(x) + g_{0,1}(x)}{g_{1,1}(x) + g_{0,1}(x) + g_{1,0}(x) + g_{0,0}(x)} > \frac12 + \frac{t_{X,\textup{DO}}(\delta)}{2} \cdot  \right. \\
\nonumber &\left. \lsb \frac{g_{1,1}(x)}{g_{1,1}(x) + g_{0,1}(x) + g_{1,0}(x) + g_{0,0}(x)} \cdot \frac{1}{p_{1,1}}  -\frac{g_{0,1}(x)}{g_{1,1}(x) + g_{0,1}(x) + g_{1,0}(x) + g_{0,0}(x)} \cdot \frac{1}{p_{0,1}} \rsb \right)\\
\nonumber&=I\left(\eta^Y(x) > \frac12 + \frac{t_{X,\textup{DO}}(\delta)}{2} \cdot \lsb \frac{\eta^Y_{A=1}(x) \eta^A(x)}{p_{1,1}} - \frac{\eta^Y_{A=0}(x)(1-\eta^A(x))}{p_{0,1}} \rsb \right).
\end{align}
This finishes the proof.
\end{proof}

\begin{proof}[Proof of Theorem \ref{thm:FCSC1}]
{We begin by proving the result for demographic parity.} Denoting for all $a,y$,
$c_{a,y}(\delta) =1/2 +(2a-1)(1-2y)t_{X,\textup{DD}}(\delta)/(2p_a)$, we have,
\begin{align*}
&R^{\textup{FCSC}}_{X,t_{X,\textup{DD}}(\delta)}(f)=\sum_{a\in\{0,1\}}\lmb c_{a,0}(\delta)\cdot \P(\widehat{Y}_f=1,Y=0,A=a)+c_{a,1}(\delta) \cdot \P(\widehat{Y}_f=0, Y=1,A=a)\rmb\\
 &= \sum_{a\in\{0,1\}}p_a\lmb
   c_{a,0}(\delta)\cdot \P(\widehat{Y}_f=1,Y=0|A=a)+c_{a,1}(\delta)\cdot \P(\widehat{Y}_f=0, Y=1|A=a)\rmb\\
   &=\sum_{a\in\{0,1\}}p_a  c_{a,0}(\delta)\cdot\int_\X \P (\widehat{Y}_f=1,Y=0|A=a,X=x)  d\Pa(x)\\
   &+\sum_{a\in\{0,1\}}p_a c_{a,1}(\delta)\cdot \int_\X\P(\widehat{Y}_f=0, Y=1|X=x,A=a)  d\Pa(x).
     \end{align*}
This further equals
  \begin{align} 
   \nonumber&\sum_{a\in\{0,1\}}p_a 
    \int_\X \lmb c_{a,0}(\delta)f(x)(1-\eta_a(x)) +c_{a,1}(\delta) (1-f(x))\eta_a(x) \rmb d\Pa(x)\\
   \nonumber&=\sum_{a\in\{0,1\}}p_a \int_\X\lmb  
    c_{a,0}(\delta)f(x)
    -c_{a,0}f(x)\eta_a(x)+c_{a,1}(\delta)\eta_a(x) -c_{a,1}(\delta)f(x)\eta_a(x)  \rmb d\Pa(x)\\
    \label{eq:FSCS1_all}
    &=\sum_{a\in\{0,1\}}p_a \int_\X  
    \lsb c_{a,0}(\delta) -(c_{a,1}(\delta)+c_{a,0}(\delta))\eta_a(x)\rsb f(x)
     d\Pa  + \sum_{a\in\{0,1\}}p_a \int_\X  
    c_{a,1}(\delta) \eta_a(x)
     d\Pa(x)\\
    \nonumber&=\sum_{a\in\{0,1\}}p_a \int_\X  \lsb c_{a,0}(\delta) -\eta_a(x)\rsb f(x) d\Pa  + \sum_{a\in\{0,1\}}p_a \int_\X  c_{a,1}(\delta) \eta_a(x)  d\Pa(x).
\end{align}
This also equals
\begin{align}
     \nonumber&= \int_\X \lsb\sum_{a\in\{0,1\}}p_a \lsb c_{a,0}(\delta) -\eta_a(x)\rsb \frac{d\Pa(x)}{d\P_{X}(x)}\rsb f(x)d\P_X(x)  + \sum_{a\in\{0,1\}}p_a \int_\X  c_{a,1}(\delta) \eta_a(x)  d\Pa(x)\\
     \nonumber&= \int_\X \lsb\sum_{a\in\{0,1\}} \lsb c_{a,0}(\delta) -\eta_a(x)\rsb \eta^A(x)\rsb f(x)d\P_X(x)  + \sum_{a\in\{0,1\}}p_a \int_\X  c_{a,1}(\delta) \eta_a(x)  d\Pa(x).
     \end{align}
The second term does not depend on $f$ and the first term can be expressed as:
\begin{align}
 \nonumber&   \sum_{a\in\{0,1\}}p_a \int_\X  \lsb c_{a,0}(\delta) -\eta_a(x)\rsb f(x) d\Pa 
    =  \sum_{a\in\{0,1\}}p_a  \int_\X \lsb c_{a,0}(\delta) -\eta_a(x)\rsb f(x) \frac{d\Pa(x)}{d\P_X(x)}d\P_X(x) \\
\nonumber&=  \int_\X \sum_{a\in\{0,1\}}  \lmb\lsb c_{a,0}(\delta) -\eta_a(x)\rsb\P(A=a|X=x) \rmb f(x) d\P_X(x) \\
\label{eq:FSCS1_all2}
&=  \int_\X \sum_{a\in\{0,1\}}   \lmb\lsb c_{a,0}(\delta)\P(A=a|X=x) - \P(Y=1,A=a|X=x)\rsb\rmb f(x) d\P_X(x) \\
\nonumber&=  \int_\X \lmb \lsb   c_{1,0}(\delta)\eta^A(x) + c_{0,0}(\delta)(1-\eta^A(x))\rsb - \eta^Y(x) \rmb f(x) d\P_X(x)
\end{align}
In turn, this equals
\begin{align}
\nonumber&=  \int_\X \lmb \lsb   \lsb\frac12+  \frac{t_{X,\textup{DD}}(\delta)}{2p_1}\rsb\eta^A(x) +  \lsb\frac12-  \frac{t_{X,\textup{DD}}(\delta)}{2p_0}\rsb(1-\eta^A(x))\rsb - \eta^Y(x) \rmb f(x) d\P_X(x)\\
\nonumber&=  \int_\X \lmb \frac12 +\frac{t_{X,\textup{DD}}(\delta) \lsb\eta^A(x)-p_1\rsb}{2p_1p_0} - \eta^Y(x) \rmb f(x) d\P_X(x).
\end{align}
Clearly, this quantity is minimized  by taking $f(x)=1$ if $\eta^Y(x)>1/2 +{t_{\textup{X,DD}(\delta)}\lsb\eta^A(x)-p_1\rsb}/({2p_1p_0})$ and $f(x)=0$ if $\eta^Y(x)\le 1/2 +{t_{X,\textup{DD}}\lsb\eta^A(x)-p_1\rsb}/({2p_1p_0})$. 
We can thus conclude that for all $x,a$,
$$f^\textup{FCSC}_{X,t_{X,\textup{DD}}(\delta)}(x) = I\lsb\eta_a(x)>\frac12 +\frac{t_{X,\textup{DD}}\lsb\eta^A(x)-p_1\rsb}{2p_1p_0}\rsb = f^\star_{X,\textup{DD},\delta}(x).$$

This finishes the proof for demographic parity. We show the result for equality of opportunity, as the proof for predictive equality is analogous to this setting. Denoting for all $a,y$,
$c_{a,y}(\delta) =1/2 +(2a-1)yt_{X,\textup{DO}}(\delta)/(2p_{a,y})$, then by a similar derivation leading up to \eqref{eq:FSCS1_all2} where the factor of $c_{a,1}(\delta) + c_{a,0}(\delta)$ no longer necessarily sums to 1 in \eqref{eq:FSCS1_all}, we have minimizing $R^{\textup{FCSC}}_{X,t_{X,\textup{DO}}(\delta)}(f)$ is equivalent to minimizing
\begin{align*}
    &\int_\X \sum_{a\in\{0,1\}}   \lmb\lsb c_{a,0}(\delta)\P(A=a|X=x) - (c_{a,1}(\delta) + c_{a,0}(\delta)) \P(Y=1,A=a|X=x)\rsb\rmb f(x) d\P_X(x) \\
    &=\int_\X \lmb \frac12 - \sum_{a\in\{0,1\}} [1+\lsb \frac{2a-1}{2p_{a,1}}  \rsb t_{X,\textup{DO}}(\delta) \P(Y=1,A=a|X=x)] \rmb f(x) d\P_X(x)\\
    &=\int_\X \lmb \frac12 - (1-\frac{t_{X,\textup{DO}}(\delta)}{2p_{0,1}}) \cdot \lsb \eta^Y_{A=0}(x) (1-\eta^A(x)) \rsb - (1+\frac{t_{X,\textup{DO}}(\delta)}{2p_{1,1}}) \cdot \lsb \eta^Y_{A=1}(x) \eta^A(x) \rsb \rmb f(x) d\P_X(x) \\
    & = \int_\X \lmb \frac12 + \frac{t_{X,\textup{DO}}(\delta)}{2p_{0,1}} \lsb \eta^Y_{A=0}(x)(1-\eta^A(x)) \rsb + \frac{t_{X,\textup{DO}}(\delta)}{2p_{1,1}} \lsb \eta^Y_{A=1}(x) \eta^A(x) \rsb - \eta^Y(x) \rmb f(x) d\P_X(x).
\end{align*}

This quantity is minimized by taking $f(x)=1$ if $\eta^Y(x) > \frac12 + \frac{t_{X,\textup{DO}}(\delta)}{2p_{0,1}} \lsb \eta^Y_{A=0}(x)(1-\eta^A(x)) \rsb + \frac{t_{X,\textup{DO}}(\delta)}{2p_{1,1}} \lsb \eta^Y_{A=1}(x) \eta^A(x) \rsb$ and $f(x) =0$ if $\eta^Y(x) \leq \frac12 + \frac{t_{X,\textup{DO}}(\delta)}{2p_{0,1}} \lsb \eta^Y_{A=0}(x)(1-\eta^A(x)) \rsb + \frac{t_{X,\textup{DO}}(\delta)}{2p_{1,1}} \lsb \eta^Y_{A=1}(x) \eta^A(x) \rsb$. We thus conclude that for all $x,a$, 
$$f^\textup{FCSC}_{X,t_{X,\textup{DO}}(\delta)}(x) = I\left(\eta^Y(x) > \frac12 + \frac{t_{X,\textup{DO}}(\delta)}{2} \cdot \lsb \frac{\eta^Y_{A=1}(x) \eta^A(x)}{p_{1,1}} - \frac{\eta^Y_{A=0}(x)(1-\eta^A(x))}{p_{0,1}} \rsb \right) = f^\star_{X,\textup{DO},\delta}(x).$$
This finishes the proof.
\end{proof}

\subsection{Form of Fair Bayes-optimal Classifiers Without Distributional Assumptions}
\label{gd}
In our main text, we assume that,  for $a\in \{0,1\}$, $\eta_a(X)$
has  a probability density function on $\X$, 
In this section, we consider the more general case with no distributional assumptions on features. Without this assumption,
the boundary case (where the features are exactly on the ``fair'' boundary) does not necessarily have zero probability. In this case, the optimal classifiers must be carefully randomized.

To handle this technical difficulty, we consider the randomized classifier. Let $t\in \R$ and $\tau: \X\times\A\to [0,1]$   be a measurable function. For $x\in\X$ and $a\in\{0,1\}$, we define $f_{\textup{Dis},t,\tau}$ as,
\begin{equation}
 f_{{\textup{Dis}},t,\tau}(x,a) = I\lsb \eta_a(x)>\frac12+\frac{t}{2}\wD(x,a)\rsb   +\tau(x,a) I\lsb\eta_a(x)= \frac12 +\frac{t}{2}\wD(x,a)\rsb.
\end{equation}
We further define the  functions $D_{\textup{Dis},\min}:\R\to \R$ and $D_{\textup{Dis},\max}:\R\to \R$
to measure the minimal and maximal disparities of $f_{_{\textup{Dis}},t,\tau}(x,a)$, 
such that for all $t\in\R$ and $\tau: \X\times\{0,1\}\to [0,1]$:
\begin{equation*}
D_{\textup{Dis},\min} (t)=\min_{\tau}\lbb\textup{Dis}(f_{_{\textup{Dis}},t,\tau})\rbb  
\,\, \text{ and }\,\, D_{\textup{Dis},\max} (t)=\max_{\tau}\lbb\textup{Dis}(f_{{\textup{Dis}},t,\tau})\rbb.
\end{equation*}
We note that $D_{\textup{Dis},\min}(0)$ and $D_{\textup{Dis},\max}(0)$ are, respectively, 
the infimum and supremum of the disparity over all unconstrained Bayes-optimal classifiers. Moreover, ${D}_{\textup{Dis},\min}$ and ${D}_{\textup{Dis},\max}$ satisfy the following monotonicity properties.
\begin{proposition}[Monotonicity of $D_{\textup{Dis},\min}$ and $D_{\textup{Dis},\max}$ ]\label{mono-dis}
As functions of $t$, both ${D}_{\textup{Dis},\min}$ and 
${D}_{\textup{Dis},\max}$ are monotone non-increasing.
%

\end{proposition}
The proofs of all results are presented later in this section.
For any $\delta>0$,
 define
 the following quantity, 
 which can be viewed as  an ``inverse" of 
the disparity functions:
\begin{equation}\label{eq:td_inf}
  \td= 
    \argmin_t\lbb|t|:  \inf_{\tau}|\textup{Dis}(f_{{\textup{Dis}},t,\tau})|\le \delta\rbb=\left\{\begin{array}{ll}
      \inf\left\{t: D_{\textup{Dis},\min} (t)\le \delta\right\}, & 
    D_{\textup{Dis},\min} (t) >\delta;\\
 \sup\left\{t: D_{\textup{Dis},\max} (t) \ge -\delta\right\}, &
 D_{\textup{Dis},\max} (t)<-\delta;\\
    0, &  \textnormal{otherwise}.\\
  \end{array}  \right.
\end{equation}
Again, we want to find the $t$ with minimal $|t|$ such that, for some $\tau$, $f_{{\textup{Dis}},t,\tau}$ satisfies the fairness constraint. Here, we identify three different cases: (1) $\delta \ge \max(D_{\textup{Dis},\min} (t), -D_{\textup{Dis},\max} (t))$, (2) $\delta < D_{\textup{Dis},\min} (t)$, and (3) $\delta < -D_{\textup{Dis},\max} (t)$.
These three cases are disjoint and cover all possible scenarios. 
Clearly, either case (2) or case (3) is true when case (1) does not hold. 
Furthermore, since $\delta\ge 0$ and $D_{\textup{Dis},\min} (t) \le D_{\textup{Dis},\max} (t)$ by definition, cases (2) and (3) cannot occur simultaneously.
We obtain the following result.


\begin{theorem}[Fair Bayes-optimal Classifiers Without Distributional Assumptions]\label{thm-opt-dp-degenerate}
For any $\delta\ge 0$, there is a measurable function $\tau_{\textup{Dis},\delta}: \X\times\{0,1\} \to [0,1]$, such that   
$f^\star_{\textup{Dis},\delta}=f_{\td,\tau_{\textup{Dis},\delta}}$ is a $\delta$-fair Bayes-optimal classifier.
Here, 
 $\tau_{\textup{Dis},\delta}$
is determined to satisfy the following constraints:
\begin{itemize}[]
    \item (1). 
    When $D_{\textup{Dis},\min} (t)>\delta$,
$$       \textup{Dis}(f_{{\textup{Dis}},\td,\tau_{\textup{Dis},\delta}}) =\delta.
$$
    \item (2). 
    When $D_{\textup{Dis},\max} (t)<-\delta$,
$$       \textup{Dis}(f_{{\textup{Dis}},\td,\tau_{\textup{Dis},\delta}}) =-\delta.
$$
 \item (3).
 When $\delta \ge \max(D_{\textup{Dis},\min} (t), -D_{\textup{Dis},\max} (t))$,
$$       |\textup{Dis}(f_{{\textup{Dis}},\td,\tau_{\textup{Dis},\delta}})| \le\delta.
$$
\end{itemize}
\end{theorem}
To specify a particular form for $\tau_{\textup{Dis},\delta}$ in Theorem \ref{thm-opt-dp-degenerate}, we first introduce the following proposition, which offers a more comprehensive understanding of $D_{\textup{Dis},\min}$ and $D_{\textup{Dis},\max}$.
Recall that for any function $f:t\mapsto f(t)$, we denote its left-hand limit and right-hand limit at point $t\in\R$ as $\lim_{t'\to t^-}f(t')$ and $\lim_{t'\to t^+}f(t')$, respectively. 
\begin{proposition}\label{prop:deepD}
As functions of $t$,
\begin{itemize}[]
\item (1)  
For 
$G_{a,+}$ and $G_{a,-}$ in \eqref{eq:gapm}, we have, for all $t\in\R$, and all $x,a$
\begin{align*}
&\tau_{\textup{Dis},\min}(x,a): = I((x,a)\in G_{a,-})\in \argmin_\tau\{ \textup{Dis}({f_{\textup{Dis},t,\tau}}) \};\\
&\tau_{\textup{Dis},\max}(x,a): = I((x,a)\in G_{a,+})\in \argmax_\tau\{ \textup{Dis}({f_{\textup{Dis},t,\tau}}) \}.
\end{align*}
Moreover,
\begin{align*}
\Dmint =&\sum_{a\in\{0,1\}} p_a\int_{G_{a,+}} \lmb  w_{\textup{Dis}}(x,a) I\lsb\frac{2\eta_a(x)-1}{\wD(x,a)}>t \rsb\rmb d\Pa(x)\\
&+\sum_{a\in\{0,1\}} p_a\int_{G_{a,-}} \lmb  w_{\textup{Dis}}(x,a) I \lsb\frac{2\eta_a(x)-1}{\wD(x,a)}\le t\rsb\rmb d\Pa(x);\\
\Dmaxt =&\sum_{a\in\{0,1\}} p_a\int_{G_{a,+}} \lmb  w_{\textup{Dis}}(x,a) I\lsb\frac{2\eta_a(x)-1}{\wD(x,a)}\ge t \rsb\rmb d\Pa(x)\\
&+\sum_{a\in\{0,1\}} p_a\int_{G_{a,-}} \lmb  w_{\textup{Dis}}(x,a) I \lsb\frac{2\eta_a(x)-1}{\wD(x,a)}< t\rsb\rmb d\Pa(x).
\end{align*}
\item (2) We have that $D_{\textup{Dis},\min}$ is right-continuous and $D_{\textup{Dis},\min}$ is left-continuous.

\item (3) 
We have,
for $t\in\R$, 
\begin{align*}
& \lim_{t'\to t^-}{D}_{\textup{Dis},\min}(t')={D}_{\textup{Dis},\max}(t).
\end{align*}

\end{itemize}
\end{proposition}

Now, we  provide a specific choice of $\tau_{\textup{Dis},\delta}$. Let 
$\tau_0$
be the identically zero function with $\tau_0(x,a):=0$, for all $x,a$, and 
for $t\in\R$, $D_{\textup{Dis},0}(t)=\textup{Dis}(f_{\textup{Dis},t,\tau_0})$. Clearly, for all $t$, $D_{\textup{Dis},\min}(t)\le D_{\textup{Dis},0}(t)\le D_{\textup{Dis},\max}(t)$.
We consider three cases in order.
    \renewcommand*{\arraystretch}{1.7}
\begin{itemize}[]
    \item (1). 
    When $D_{\textup{Dis},\min} (t)>\delta$, we have
    $\td= \inf\left\{t: D_{\textup{Dis},\min} (t)\le \delta\right\}.$ Since $D_{\textup{Dis},\min}$ is right-continuous, we have
    $D_{\textup{Dis},\min} (t)\le \delta \le \lim_{t'\to t^{-}} D_{\textup{Dis},\min} (t') = D_{\textup{Dis},\max} (t')$. We can take, for all $x\in\X$ and $a\in\{0,1\}$,
$$\tau_{\textup{Dis},\delta}(x,a) =\left\{\begin{array}{ll}
\frac{  [D_{\textup{Dis},0}(\td)-\delta]\cdot I((x,a)\in G_{a,-})}{  D_{\textup{Dis},0}(\td)-  D_{\textup{Dis},\min}(\td)},  &  D_{\textup{Dis},\min} (\td) \le \delta \le D_{\textup{Dis},0}(\td);\\
 \frac{  [\delta - D_{\textup{Dis},0}(\td)]\cdot I((x,a)\in G_{a,+})}{  D_{\textup{Dis},\max}(\td)-  D_{\textup{Dis},0}(\td)},    & D_{\textup{Dis},0} (\td) <\delta \le D_{\textup{Dis},\max}(\td).
\end{array}\right.
$$
    \item (2). 
    When $D_{\textup{Dis},\max} (t)<-\delta$, we have
    $\td= \sup\left\{t: D_{\textup{Dis},\max} (t)\ge -\delta\right\}.$ Since $D_{\textup{Dis},\max}$ is left-continuous, we have
    $D_{\textup{Dis},\min} (t) =  \lim_{t'\to t^{+}} D_{\textup{Dis},\max} (t')\le -\delta \le D_{\textup{Dis},\max} (t)$. We can take, for all $x\in\X$ and $a\in\{0,1\}$,
$$\tau_{\textup{Dis},\delta}(x,a) =\left\{\begin{array}{ll}
\frac{ [D_{\textup{Dis},0}(\td)+\delta]\cdot I((x,a)\in G_{a,-})}{  D_{\textup{Dis},0}(\td)-  D_{\textup{Dis},\min}(\td)},  &  D_{\textup{Dis},\min} (\td) \le -\delta \le D_{\textup{Dis},0}(\td);\\
 \frac{  [-\delta - D_{\textup{Dis},0}(\td)]\cdot I((x,a)\in G_{a,+})}{  D_{\textup{Dis},\max}(\td)-  D_{\textup{Dis},0}(\td)},    & D_{\textup{Dis},0} (\td) <-\delta \le D_{\textup{Dis},\max}(\td).
\end{array}\right.
$$

 \item (3). 
 When $\delta \ge \max(D_{\textup{Dis},\min} (t), -D_{\textup{Dis},\max} (t))$, 
 since
 $D_{\textup{Dis},0}\le D_{\textup{Dis},0}\le  D_{\textup{Dis},0}$, we have
 $D_{\textup{Dis},\min}\le\delta$ when $ \delta<D_{\textup{Dis},0}$ and  $ D_{\textup{Dis},\max}\ge -\delta$ when $ \delta<-D_{\textup{Dis},0}$. Thus, we can take, for all $x\in\X$ and $a\in\{0,1\}$,
$$\tau_{\textup{Dis},\delta}(x,a) =
\left\{\begin{array}{ll}
\frac{  [D_{\textup{Dis},0}(\td)-\delta]\cdot I((x,a)\in G_{a,-})}{  D_{\textup{Dis},0}(\td)-  D_{\textup{Dis},\min}(\td)}, 
&  \delta < D_{\textup{Dis},0} (\td);\\
\frac{  [-\delta - D_{\textup{Dis},0}(\td)]\cdot I((x,a)\in G_{a,+})}{  D_{\textup{Dis},\max}(\td)-  D_{\textup{Dis},0}(\td)},  
&  \delta < -D_{\textup{Dis},0} (\td);\\
0;&  \text{Otherwise}.
\end{array}\right.
$$
\end{itemize}
We  now present the proofs of the results discussed in this section.
\begin{proof}[Proof of Proposition \ref{mono-dis}]
We define, for $a\in\{0,1\}$, 
\begin{equation}\label{eq:gapm}
   {G}_{a,+} = \{x \in\mathcal{X},  w_{\textup{Dis}}(x,a) > 0\};\, {G}_{a,0} = \{x\in\mathcal{X},  w_{\textup{Dis}}(x,a) = 0\}; \,{G}_{a,-} = \{x\in\mathcal{X}, w_{\textup{Dis}}(x,a) < 0\}.
\end{equation}
Let $t_1<t_2$, as well as $\tau_{1}:\X\times\{0,1\}\to [0,1]$ and $\tau_{2}:\X\times\{0,1\}\to [0,1]$. For
$x\in\X$ and $a\in\{0,1\}$ with $\wD(x,a)\neq 0$, we let $q_{\textup{Dis}}(x,a)=(2\eta_a(x)-1)/\wD(x,a)$. Then, we have,
for all $x,a$,
    \renewcommand*{\arraystretch}{1}
\begin{align}\label{eq:diffinf1}
\nonumber& f_{\textup{Dis},t_1,\tau_{1,1},\tau_{0,1}}(x,a)-f_{\textup{Dis},t_2,\tau_{1,2},\tau_{0,2}}(x,a)\\
\nonumber&=\left\{
\begin{array}{ll}
  I\lsb t_1< q_{\textup{Dis}}(x,a)\le t_2\rsb   +  \tau_{a,1}(x,a) I\lsb  q_{\textup{Dis}}(x,a) =t_1\rsb  -\tau_{a,2}(x,a)  I\lsb   q_{\textup{Dis}}(x,a)  =t_2\rsb  ,  & x\in G_{a,+};\\
  - I\lsb t_1\le  q_{\textup{Dis}}(x,a)
  < t_2 \rsb  + \tau_{a,1}(x,a) I\lsb  q_{\textup{Dis}}(x,a)
    =t_1\rsb  -\tau_{a,2}(x,a)  I\lsb   q_{\textup{Dis}}(x,a)
  =t_2\rsb ,  & x\in G_{a,-};\\
(\tau_{a,1}(x,a)-\tau_{a,2}(x,a)) I\lsb  \eta_a(x)=\frac12\rsb,   &  x\in G_{a,0};
\end{array}
\right.\\
&\ge \left\{
\begin{array}{ll}
  I\lsb t_1< \frac{2\eta_a(x)-1}{  \wD(x,a)}
< t_2\rsb,  & x\in G_{a,+};\\
  - I\lsb t_1\le \frac{2\eta_a(x)-1}{  \wD(x,a)}
  \le t_2 \rsb  ,  & x\in G_{a,-};\\
(\tau_{a,1}-\tau_{a,2}) I\lsb  \eta_a(x)=\frac12\rsb,   &  x\in G_{a,0}.
\end{array}
\right.
\end{align}
It thus follows that
\begin{align*}
&D_{\textup{Dis},\min}(t_1)-D_{\textup{Dis},\min}(t_2)= \inf_{\tau_1}\textup{Dis}(f_{\textup{Dis},t_1,\tau_{1}})-\inf_{\tau_2}\textup{Dis}(f_{\textup{Dis},t_2,\tau_{2}})\\
&\ge \inf_{\tau_1,\tau_2}\int_{\A}\int_{\X}\lmb
f_{\textup{Dis},t_1,\tau_1}(x,a)-f_{\textup{Dis},t_2,\tau_2}(x,a)
\rmb w_{\textup{Dis}}(x,a) d\P_{X,A}({x,a})\\
&=\inf_{\tau_1,\tau_2}\lmb \sum_{a\in\{0,1\}} p_{a}\int_{\X}\lmb
f_{\textup{Dis},t_1}(x,a)-f_{\textup{Dis},t_2}(x,a)
\rmb w_{\textup{Dis}}(x,a) d\Pa\rmb\\
&\ge\sum_{a\in\{0,1\}} p_{a}\int_{G_{a,+}}   I\lsb t_1< \frac{2\eta_a(x)-1}{  \wD(x,a)}
 <  t_2\rsb   w_{\textup{Dis}}(x,a) d\P_{X\mid A=a}\\
  &-\sum_{a\in\{0,1\}} p_{a}\int_{G_{a,-}}   I\lsb t_1\le \frac{2\eta_a(x)-1}{  \wD(x,a)}
  \le  t_2\rsb   w_{\textup{Dis}}(x,a) d\P_{X\mid A=a}\\
  &+\sum_{a\in\{0,1\}} p_{a}\int_{G_{a,=}}   I\lsb \eta_a(x)=\frac12\rsb   \cdot 0 d\P_{X\mid A=a} 
  \ge 0.
\end{align*}
The last inequality holds since the indicator function is non-negative, and
$  w_{\textup{Dis}}$ is positive on $G_{a,+}$ and negative on $G_{a,-}$.  
Using a similar argument, we can verify that 
$D_{\textup{Dis},\max}$
  is also monotonically non-increasing.
\end{proof}
\begin{proof}[Proof of Theorem \ref{thm-opt-dp-degenerate}]

We analyze the following three cases: (1) $\delta\ge  \max\{D_{\textup{Dis},\min}(0), -D_{\textup{Dis},\max}(0)\}$, (2) $D_{\textup{Dis},\min}(0)>\delta$ and (3) $D_{\textup{Dis},\max}(0)<-\delta$. 
Since the proof for case (3) 
is analogous to that for case (2), we omit the discussion of case (3).

Case 1: $\delta\ge  \max\{D_{\textup{Dis},\min}(0), -D_{\textup{Dis},\max}(0)\}$.
 In this case, there exists at least one unconstrained Bayes-optimal classifier satisfying the fairness constraint. 
 As a result, we have $\td=0$ and the value of function $\tau_{\textup{Dis},\delta}$ has no effect on the excess risk of the classifier. Thus, we only need to choose one such that
$$\left|\P\lsb\widehat{Y}_{f^\star_\delta}=1\mid A=1\rsb-\P\lsb\widehat{Y}_{f^\star_\delta}=0|A=1\rsb\right|\le \delta.$$

Case 2: $D_{\textup{Dis},\min}(0)>\delta$.
By  the definition of $\td$ in \eqref{eq:td_inf} and Proposition \ref{mono-dis}, we have $\td>0$. Let, for all $x,a$, 
$\phi_0(x,a) = 2\eta^Y_{A=1}(x,a)-1$, $\phi_1(x,a)={ w_{\textup{Dis}}(x)}$.
By Lemma \ref{lem:misclassification} and \eqref{eq:exp dis}, we can write $\textup{Acc}(f)$ and ${D}_{\textup{Dis}}(f)$ as
\begin{align*} 
&   \textup{Acc}(f) = \int_\A\int_\X f(x,a) \phi_0(x,a) d\P_{X,A}(x,a) + \int_\A\int_\X (1-\eta_a(x))  d\P_{X,A}(x,a)  ;\\
& \textup{Dis}(f)=  \int_{\X\times\A}f(x,a)\phi_1(x,a) d\P_{X,A}(x,a).
\end{align*}
Define
\begin{equation*}
\begin{array}{l}
\mathcal{F}_{=} = \lbb f:  {\textup{Dis}}(f) = \delta \rbb;  
\mathcal{F}_{|\cdot|,\le} = \lbb f:  \lab{\textup{Dis}}(f)\rab \le \delta\rbb;  \text{ and } 
\mathcal{F}_{\le} = \lbb f:{\textup{Dis}}(f) \le \delta\rbb.
\end{array}\end{equation*}
By our construction, we have $f_{\textup{Dis},\td,\tau_{\textup{Dis},\delta}}\in\mathcal{F}_{=}\subset\mathcal{F}_{|\cdot|,\le}\subset\mathcal{F}_{\le}$. Since $\td\ge0$, by the generalized Neyman-Pearson lemma (Lemma \ref{NP_lemma}),

$$\underset{f\in\mathcal{F}_{\le}}{\max}\, \textup{Acc}(f) = \textup{Acc}(f_{\textup{Dis},\td,\tau_{\textup{Dis},\delta}})\le  \underset{f\in\mathcal{F}_{|\cdot|,\le}}{\max}\, \textup{Acc}(f) \le  \underset{f\in\mathcal{F}_{\le}}{\max}\, \textup{Acc}(f).$$
Thus, we can conclude that 
$$f_{\textup{Dis},\td,\tau_{\textup{Dis},\delta}} = \underset{f\in\mathcal{F}_{|\cdot|,\le}}{\argmax}\,  \textup{Acc}(f)
= \argmin_{f\in\mF}\lbb R(f): { |\textup{Dis}(f)|} \le \delta
\rbb. $$
The proof is thus completed.
\end{proof}

\begin{proof}[Proof of Proposition \ref{prop:deepD}]

    We only prove the results for $D_{\textup{Dis},\min}$ and similar arguments apply to the result for $D_{\textup{Dis},\max}(t)$. For (1), by definition, we have
\begin{align}\label{eq:expdinf}
\nonumber&\textup{Dis}(f_{\textup{Dis},t,\tau})=  \sum_{a\in\{0,1\}} p_a\int_{\X} \lmb  w_{\textup{Dis}}(x,a) f_{\textup{Dis},t,\tau}(x,a)\rmb d\Pa(x)\\
\nonumber=&\sum_{a\in\{0,1\}} p_a\int_{\X} \lmb  w_{\textup{Dis}}(x,a) I\lsb\eta_a(x)>\frac{1}{2}+\frac{t}{2}\wD(x,a)\rsb\rmb d\Pa(x)\\
\nonumber&+\sum_{a\in\{0,1\}} p_a\int_{\X} \lmb  w_{\textup{Dis}}(x,a) \tau(x,a) I\lsb\eta_a(x)=\frac{1}{2}+\frac{t}{2}\wD(x,a)\rsb\rmb d\Pa(x)\\
\nonumber=&\sum_{a\in\{0,1\}} p_a\int_{G_{a,+}} \lmb  w_{\textup{Dis}}(x,a) I\lsb\eta_a(x)>\frac{1}{2}+\frac{t}{2}\wD(x,a)\rsb\rmb d\Pa(x)\\
\nonumber&+\sum_{a\in\{0,1\}} p_a\int_{G_{a,-}} \lmb  w_{\textup{Dis}}(x,a) I\lsb\eta_a(x)>\frac{1}{2}+\frac{t}{2}\wD(x,a)\rsb\rmb d\Pa(x)\\
\nonumber&+\sum_{a\in\{0,1\}} p_a\int_{G_{a,+}} \lmb  w_{\textup{Dis}}(x,a) \tau(x,a) I\lsb\eta_a(x)=\frac{1}{2}+\frac{t}{2}\wD(x,a)\rsb\rmb d\Pa(x)\\
&+\sum_{a\in\{0,1\}} p_a\int_{G_{a,-}} \lmb  w_{\textup{Dis}}(x,a) \tau(x,a) I\lsb\eta_a(x)=\frac{1}{2}+\frac{t}{2}\wD(x,a)\rsb\rmb d\Pa(x).
\end{align}
Since the indicator function is always non-negative, the infimum of \eqref{eq:expdinf} with respect to $\tau: \X\times\{0,1\}\to [0,1]$ is achieved by taking
$\tau_{\textup{Dis},\min}(x,a) = I\lsb (x,a)\in G_{a,-}\rsb$ for all $x,a$.
Moreover,
we have,
\begin{align*}
\Dmint =&\sum_{a\in\{0,1\}} p_a\int_{G_{a,+}} \lmb  w_{\textup{Dis}}(x,a) I\lsb\eta_a(x)>\frac{1}{2}+\frac{t}{2}\wD(x,a)\rsb\rmb d\Pa(x)\\
&+\sum_{a\in\{0,1\}} p_a\int_{G_{a,-}} \lmb  w_{\textup{Dis}}(x,a) I\lsb\eta_a(x)\ge\frac{1}{2}+\frac{t}{2}\wD(x,a)\rsb\rmb d\Pa(x)\\
 =&\sum_{a\in\{0,1\}} p_a\int_{G_{a,+}} \lmb  w_{\textup{Dis}}(x,a) I\lsb\frac{2\eta_a(x)-1}{\wD(x,a)}>t \rsb\rmb d\Pa(x)\\
&+\sum_{a\in\{0,1\}} p_a\int_{G_{a,-}} \lmb  w_{\textup{Dis}}(x,a) I \lsb\frac{2\eta_a(x)-1}{\wD(x,a)}\le t\rsb\rmb d\Pa(x).
\end{align*}
For (2), denote
\begin{align*}
D_{\textup{Dis},+,>}(t) &=\sum_{a\in\{0,1\}} p_a\int_{G_{a,+}} \lmb  w_{\textup{Dis}}(x,a) I\lsb\frac{2\eta_a(x)-1}{\wD(x,a)}>t \rsb\rmb d\Pa(x)\\
D_{\textup{Dis},-,\le}(t) &=\sum_{a\in\{0,1\}} p_a\int_{G_{a,-}} \lmb  w_{\textup{Dis}}(x,a) I \lsb\frac{2\eta_a(x)-1}{\wD(x,a)}\le t\rsb\rmb d\Pa(x).
\end{align*}
We study the left- and right-hand limit of $D_{\textup{Dis},+,>}$ at a point $t\in\R$ as follows.
Let $(t_n)_{n=1}^\infty$ be monotone decreasing sequence tending to zero  as $n\to\infty$. 
For fixed $t\in\R$ and $a\in\{0,1\}$, we consider the sets
\begin{align*}
&J_{a,-,t}=\lbb(x,a):\frac{2\eta_a(x)-1}{\wD(x,a)}\ge t
\rbb;  \,\,\qquad J_{n,a,-,t}=\lbb(x,a):\frac{2\eta_a(x)-1}{\wD(x,a)}> t-t_n\rbb;\\ &J_{a,+,t}=\lbb(x,a):\frac{2\eta_a(x)-1}{\wD(x,a)}>t
\rbb;  \,\,  \qquad  J_{n,a,+,t}=\lbb(x,a):\frac{2\eta_a(x)-1}{\wD(x,a)}> t+t_n
\rbb.
\end{align*} and define the following functions on $\X\times\{0,1\}$. For all $x,a$, let,
\begin{align*}
 &j_{n,a,-,t}(x,a)= w_{\textup{Dis}}(x,a) I\lsb G_{a,+}\cap J_{n,a,-,t}\rsb;\quad
 j_{a,-,t}(x,a)= w_{\textup{Dis}}(x,a) I\lsb G_{a,+}\cap J_{a,-,t}\rsb;\\
  &j_{n,a,+,t}(x,a)= w_{\textup{Dis}}(x,a) I\lsb G_{a,+}\cap J_{n,a,+,t}\rsb;\quad
   j_{a,+,t}(x,a)= w_{\textup{Dis}}(x,a) I\lsb G_{a,+}\cap J_{a,+,t}\rsb.
\end{align*} 
It follows that
\begin{align*}
    &j_{1,a,-,t}(x,a)\ge j_{2,a,-,t}(x,a)\ge j_{3,a,-,t}(x,a)\ge\ldots   \,\,\, \text{ with } \lim_{n\to\infty}j_{n,a,-,t}(x,a)=j_{a,-,t}(x,a);\\
    &j_{1,a,+,t}(x,a)\le j_{2,a,+,t}(x,a)\le j_{3,a,+,t}(x,a)\le\ldots   \,\,\, \text{ with } \lim_{n\to\infty}j_{n,a,+,t}(x,a)=j_{a,+,t}(x,a).
\end{align*}
Then, by the monotone convergence theorem \citep{royden1968real},
\begin{align*}
   &\lim_{t'\to t-} D_{\textup{Dis},+,>}(t')=\lim_{n\to\infty}D_{\textup{Dis},+,> }(t-t_n)\\
   =&\lim_{n\to\infty}\lmb\sum_{a\in\{0,1\}} p_a\int_{G_{a,+}} \lmb  w_{\textup{Dis}}(x,a) I\lsb\frac{2\eta_a(x)-1}{\wD(x,a)}> t-t_n \rsb\rmb d\Pa(x)\rmb\\
   =&\sum_{a\in\{0,1\}}\lbb p_a\lim_{n\to\infty}\lmb\int_{\X}  j_{n,a,-,t}(x,a)d\Pa(x)\rmb\rbb = \sum_{a\in\{0,1\}}\lbb p_a \int_{\X} j_{a,-,t}(x,a)d\Pa(x)\rbb \\
  =& \sum_{a\in\{0,1\}} p_a\int_{G_{a,+}} \lmb  w_{\textup{Dis}}(x,a) I\lsb\frac{2\eta_a(x)-1}{\wD(x,a)}\ge t\rsb\rmb d\Pa(x).
   \end{align*}
   Similarly, we can verify that
 \begin{align*}
   &\lim_{t'\to t+} D_{\textup{Dis},+,> }(t')=\lim_{n\to\infty}D_{\textup{Dis},+,> }(t+t_n)\\
  =& \sum_{a\in\{0,1\}} p_a\int_{G_{a,+}} \lmb  w_{\textup{Dis}}(x,a) I\lsb\frac{2\eta_a(x)-1}{\wD(x,a)}> t\rsb\rmb d\Pa(x)= D_{\textup{Dis},+,>}(t). 
   \end{align*}  
Thus, $D_{\textup{Dis},+,>}$ is right-continuous and for $t\in \R$,
\begin{equation*}
\lim_{t'\to t^{-}}D_{\textup{Dis},+,>}(t')= \sum_{a\in\{0,1\}} p_a\int_{G_{a,+}} \lmb  w_{\textup{Dis}}(x,a) I\lsb\frac{2\eta_a(x)-1}{\wD(x,a)}\ge t\rsb\rmb d\Pa(x).
\end{equation*}
With the similar arguments above, we can show that  $D_{\textup{Dis},-,\le}$ is also right-continuous, moreover, for $t\in \R$
\begin{align*}
&\lim_{t'\to t^{-}}D_{\textup{Dis},-,\le}(t') 
=& \sum_{a\in\{0,1\}} p_a\int_{G_{a,-}} \lmb  w_{\textup{Dis}}(x,a) I\lsb\frac{2\eta_a(x)-1}{\wD(x,a)}< t\rsb\rmb d\Pa(x).
\end{align*}
As a consequence, $D_{\textup{Dis},\min}=D_{\textup{Dis},+,>}+ D_{\textup{Dis},-,\le}$ is right-continuous.

Finally,
(3) holds since,  for $t\in\R$,
\begin{align*}
& \lim_{t'\to t^{+}}D_{\textup{Dis},\min}(t') = 
\lim_{t'\to t^{+}}D_{\textup{Dis},+,>}(t') + \lim_{t'\to t^{+}}D_{\textup{Dis},-,\le}(t')
\\
=& \sum_{a\in\{0,1\}} p_a\int_{G_{a,+}} \lmb  w_{\textup{Dis}}(x,a) I\lsb\frac{2\eta_a(x)-1}{\wD(x,a)}\ge t\rsb\rmb d\Pa(x)\\
&+\sum_{a\in\{0,1\}} p_a\int_{G_{a,-}} \lmb  w_{\textup{Dis}}(x,a) I\lsb\frac{2\eta_a(x)-1}{\wD(x,a)}< t\rsb\rmb d\Pa(x) =D_{\textup{Dis},\max}(t) .
\end{align*}
This finishes the proof.
\end{proof}

\section{Bayes-optimal Classifiers for Synthetic Data}
\label{exp_det}
In this section,  we derive the $\delta$-fair Bayes-optimal classifiers for our synthetic model used in Section \ref{synth}. 
In particular, we consider the following data distribution of $(X,A,Y)$: for  $(a,y)\in\{0,1\}^2$, we let $\P(A=a,Y=y)=p_{a,y}$ and $X|A=a,Y=y\sim \mathcal{N}(\mu_{a,y},\sigma^2I_p)$.

Denote, for all $x$,
by $g_{a,y}(x) = (2\pi)^{-p/2}\sigma^{-p}\exp(-\|x-\mu_{a,y}\|^2/(2\sigma^2))$ 
the conditional probability density function of $X$ given $A=a$ and $Y=y$.  We have
\begin{align*}
   \eta_a(x)&=\P(Y=1|X=x,A=a)
   =\frac{p_{a,1}g_{a,1}(x) }{p_{a,1} g_{a,1}(x)+p_{a,o}g_{a,0}(x)}\\
   &=\frac{p_{a,1}\exp(-\frac1{2\sigma^2}\|x-\mu_{a,1}\|^2)}{p_{a,1}\exp(-\frac1{2\sigma^2}\|x-\mu_{a,1}\|^2)+p_{a,0}\exp(-\frac1{2\sigma^2}\|x-\mu_{a,0}\|^2)}.
\end{align*}
In the following, for $a\in\{0,1\}$, we denote
$\Delta_{\mu,a}=\mu_{a,1}-\mu_{a,0}$. Moreover, for $t\in[0,1]$ and $a\in\{0,1\}$, we denote, $q_a(t) = {tp_{a,0}}/({(1-t)p_{a,1}})$. Then, we have, for $(a,y)\in\{0,1\}^2$, 
\begin{align*}
&\Pay(\eta_{a}(X)>t)
=\Pay\left(\frac{p_{a,1}g_{a,1}(x) }{p_{a,1} g_{a,1}(x)+p_{a,0}g_{a,0}(x)}>t\right)\\
&=\Pay\left((1-t_a)p_{a,1}g_{a,1}(x)>t_ap_{a,0}g_{a,0}(x)\right)=\Pay\left(g_{a,1}(x)>q_a(t))g_{a,0}(x)\right)\\
&=\Pay\left(\exp\lsb -\frac1{2\sigma^2}\lnm x-\mu_{a,1}\rnm^2\rsb > q_a(t)\cdot\exp\lsb -\frac1{2\sigma^2}\lnm x-\mu_{a,0}\rnm^2\rsb\right)\\
&=\Pay\left(\lnm x-\mu_{a,0}\rnm^2-\lnm x-\mu_{a,1}\rnm^2>{2\sigma^2}\log(q_a(t))\right)\\
&=\Pay\left(2x^\top\Delta_{\mu,a}>{2\sigma^2}\log(q_a(t))+\|\mu_{a,1}\|^2-\|\mu_{a,0}\|^2\right)\\
&=\P_{Z\sim \N(0,I_p)}\left(2(\sigma \cdot Z +\mu_{a,y})^\top\Delta_{\mu,a}>{2\sigma^2}\log(q_a(t))+\|\mu_{a,1}\|^2-\|\mu_{a,0}\|^2\right).
\end{align*}
From the definition of the distribution of $(X,A,Y)$, it follows that this further equals
\begin{align*}
&\P_{Z\sim \N(0,I_p)}\left(\sigma \cdot Z ^\top\Delta_{\mu,a}>{\sigma^2}\log(q_a(t))+\frac{\|\mu_{a,1}\|^2-\|\mu_{a,0}\|^2 -2\mu_{a,y}^\top\Delta_{\mu,a}}{2}\right)\\
&=\P_{Z\sim \N(0,I_p)}\left(  Z ^\top\Delta_{\mu,a}>{\sigma}\log(q_a(t))+(1-2y)\frac{\|\Delta_{\mu,a}\|^2}{2\sigma}\right)\\
&=\P_{Z'\sim \N (0,1)}\left(  Z'>\frac{{\sigma}\log(q_a(t))}{\|\Delta_{\mu,a}\|}+(1-2y)\frac{\|\Delta_{\mu,a}\|}{2\sigma}\right)=1-\Phi\left(\frac{\sigma\log(q_a(t))}{\lnm \Delta_{\mu,a}\rnm}+(1-2y)\frac{\lnm \Delta_{\mu,a}\rnm}{2\sigma}\right),
\end{align*}
where $\Phi(\cdot)$ is the cumulative distribution function of the standard normal distribution. Hereafter, we denote, for $t\in[0,1]$, and $(a,y)\in\{0,1\}^2$,
$$\psi_{a,y}(t)=\Phi\left(\frac{\sigma\log(q_a(t))}{\lnm \Delta_{\mu,a}\rnm}+(1-2y)\frac{\lnm \Delta_{\mu,a}\rnm}{2\sigma}\right).
$$

Next, we derive $\delta$-fair Bayes-optimal classifiers under demographic parity, equality of opportunity and predictive parity.
\subsection{Demographic Parity}
For demographic parity, the weight function is given by, for all $(x,a)$, $ w_{DD}(x,a)=1/p_a$. This time, the disparity function 
$\textup{DD}$ take values, for $t\in\tset_{\textup{DD}}: =[-\min(p_1,p_0),\min(p_1,p_0)]$,
\begin{align*}
\textup{DD}(t)=&\sum_{a\in\{0,1\}} (2a-1)
\Pa\lsb \eta_a(X)>\frac12+\frac{t}{2p_a}\rsb\\
=&\sum_{a\in\{0,1\}}\sum_{y\in\{0,1\}} \lmb(2a-1)\Pa\lsb Y=y, \eta_a(X)>\frac12+\frac{t}{2p_a}\rsb \rmb \\
=&\sum_{a\in\{0,1\}}\sum_{y\in\{0,1\}} \lmb(2a-1)\P(Y=y|A=a)\cdot\Pay\lsb \eta_a(X)>\frac12+\frac{t}{2p_a}\rsb \rmb \\
=&\sum_{a\in\{0,1\}}\sum_{y\in\{0,1\}} \lmb\frac{(2a-1)p_{a,y}}{p_a} \lsb1-\psi_{a,y}\lsb
\frac12+\frac{t}{2p_a}
\rsb\rsb\rmb\\
=& -\frac{p_{1,1}}{p_1}\psi_{1,1}\lsb\frac{p_1+t}{2p_1}\rsb - \frac{p_{1,0}}{p_1}\psi_{1,0}\lsb\frac{p_1+t}{2p_1}\rsb + \frac{p_{0,1}}{p_0}\psi_{0,1}\lsb\frac{p_0-t}{2p_0}\rsb +\frac{p_{0,0}}{p_0}\psi_{0,0}\lsb\frac{p_0-t}{2p_0}\rsb. 
\end{align*}
By Corollary \ref{cor:fbbinilir}, a $\delta$-fair Bayes-optimal classifier is, for all $x,a$,
$$f^{\star}_{\textup{DD},\delta}(x,a) = I\lsb \eta_a(x)>\frac{1}{2}+\frac{t_{\textup{DD}}(\delta)}{p_{a}}\rsb,$$
with
$t_{\textup{DD}}(\delta) = \argmin_{t\in\tset_{\textup{DD}}}\lbb |t|: \textup{DD}(t)\le \delta\rbb.$ Moreover, we have,
\begin{align*}
&R(f^{\star}_{\textup{DD},\delta}) = \P\lsb \widehat{Y}_{f^{\star}_{\textup{DD},\delta}}\neq Y\rsb=\sum_{y\in\{0,1\}} \P\lsb \widehat{Y}_{f^{\star}_{\textup{DD},\delta}}=y,Y=1-Y\rsb\\
=& \sum_{a\in\{0,1\}}\sum_{y\in\{0,1\}} p_{a,y}\cdot\Pay\lsb \widehat{Y}_{f^{\star}_{\textup{DD},\delta}}=1-y\rsb \\
=& \sum_{a\in\{0,1\}}\lmb p_{a,1}\cdot\P_{A=a,y=1}\lsb \eta_a(X)\le \frac12 +\frac{t}{2p_a}\rsb +
p_{a,0}\cdot\P_{A=a,y=0}\lsb \eta_a(X)> \frac12 +\frac{t}{2p_a}\rsb\rmb\\
=&p_{1,0} +p_{0,0} + p_{1,1}\psi_{1,1}\lsb\frac{p_1+t}{2p_1}\rsb- {p_{1,0}}\psi_{1,0}\lsb\frac{p_1+t}{2p_1}\rsb + {p_{0,1}}\psi_{0,1}\lsb\frac{p_0-t}{2p_0}\rsb -{p_{0,0}}\psi_{0,0}\lsb\frac{p_0-t}{2p_0}\rsb. 
 \end{align*}

\subsection{Equality of Opportunity}
For equality of opportunity, the weight function is given for all $(x,a)$, by 
$ w_{DO}(x,a)=\eta_a(x)/p_{a,1}$. 
The disparity function 
$\textup{DO}$ take values, for $t\in\tset_{\textup{DO}}: =[-p_{0,1},p_{1,1}]$,
\begin{align*}
{D}_{\textup{DO}}(t)=&\sum_{a\in\{0,1\}} (2a-1)
\P_{A=a,Y=1}\lsb \eta_a(X)>\frac12+\frac{t \eta_a(X)}{p_{a,1}}\rsb\\
=&\sum_{a\in\{0,1\}}  \lmb(2a-1)\P_{A=a,Y=1}\lsb  \eta_a(X)>\frac{p_{a,1}}{2p_{a,1}-t}\rsb \rmb \\
=& {p_{1,1}}\lsb 1-\psi_{1,1}\lsb\frac{p_{1,1}}{2p_{1,1}-t}\rsb\rsb - {p_{0,1}}\lsb1-\psi_{0,1}\lsb\frac{p_{0,1}}{2p_{0,1}+t}\rsb \rsb. 
\end{align*}
By Corollary \ref{cor:fbbinilir}, a $\delta$-fair Bayes-optimal classifier is, for all $x,a$,
$$f^{\star}_{\textup{DO},\delta}(x,a) = I\lsb \eta_a(x)>\frac{p_{a,1}}{2p_{a,1}-t_{\textup{DO}}(\delta)}\rsb,$$
with
$t_{\textup{DO}}(\delta) = \argmin_{t\in\tset_{\textup{DO}}}\lbb |t|: {D}_{\textup{DO}}(t)\le \delta\rbb.$ Similarly, we have,
\begin{align*}
&R(f^{\star}_{\textup{DO},\delta}) = \P\lsb \widehat{Y}_{f^{\star}_{\textup{DO},\delta}}\neq Y\rsb=\sum_{y\in\{0,1\}} \P\lsb \widehat{Y}_{f^{\star}_{\textup{DO},\delta}}=y,Y=1-Y\rsb\\
=& p_{1,0}+ p_{0,0} +p_{1,1}\psi_{1,1}\lsb\frac{p_{1,1}}{2p_{1,1}-t}\rsb - {p_{1,0}}\psi_{1,0}\lsb\frac{p_{1,1}}{2p_{1,1}-t}\rsb + {p_{0,1}}\psi_{0,1}\lsb\frac{p_{0,1}}{2p_{0,1}+t}\rsb -{p_{0,0}}\psi_{0,0}\lsb\frac{p_{0,1}}{2p_{0,1}+t}\rsb. 
 \end{align*}

\subsection{Predictive Equality}
For predictive equality, the weight function is given by, for all $(x,a)$, $ w_{PD}(x,a)=(1-\eta_a(x))/p_{a,0}$. This time, the disparity function 
$\textup{PD}$ take values, for $t\in\tset_{\textup{PD}}: =[-p_{1,0},p_{0,0}]$,
\begin{align*}
{D}_{\textup{PD}}(t)=&\sum_{a\in\{0,1\}} (2a-1)
\P_{A=a,Y=0}\lsb \eta_a(X)>\frac12+\frac{t (1-\eta_a(X))}{p_{a,0}}\rsb\\
=&\sum_{a\in\{0,1\}}  \lmb(2a-1)\P_{A=a,Y=1}\lsb  \eta_a(X)>\frac{p_{a,0}+t}{2p_{a,0}+t}\rsb \rmb \\
=& {p_{1,0}}\lsb 1-\psi_{1,0}\lsb\frac{p_{1,0}+t}{2p_{1,0}+t}\rsb\rsb - {p_{0,0}}\lsb1-\psi_{0,0}\lsb\frac{p_{0,0}-t}{2p_{0,0}-t}\rsb \rsb. 
\end{align*}
By Corollary \ref{cor:fbbinilir}, a $\delta$-fair Bayes-optimal classifier is, for all $x,a$,
$$f^{\star}_{\textup{PD},\delta}(x,a) = I\lsb \eta_a(x)>\frac{p_{a,0}+t_{\textup{PD}}(\delta)}{2p_{a,0}+t_{\textup{PD}}(\delta)}\rsb,$$
with
$t_{\textup{PD}}(\delta) = \argmin_{t\in\tset_{\textup{PD}}}\lbb |t|: {D}_{\textup{PD}}(t)\le \delta\rbb.$ Moreover,
\begin{align*}
&R(f^{\star}_{\textup{PD},\delta}) = \P\lsb \widehat{Y}_{f^{\star}_{\textup{PD},\delta}}\neq Y\rsb=\sum_{y\in\{0,1\}} \P\lsb \widehat{Y}_{f^{\star}_{\textup{PD},\delta}}=y,Y=1-Y\rsb\\
=& p_{1,0}+ p_{0,0} +p_{1,1}\psi_{1,1}\lsb\frac{p_{1,0}+t}{2p_{1,0}+t}\rsb - {p_{1,0}}\psi_{1,0}\lsb\frac{p_{1,0}+t}{2p_{1,0}+t}\rsb + {p_{0,1}}\psi_{0,1}\lsb\frac{p_{0,0}-t}{2p_{0,0}-t}\rsb -{p_{0,0}}\psi_{0,0}\lsb\frac{p_{0,0}-t}{2p_{0,0}-t}\rsb.  
 \end{align*}

}
\end{sloppypar}
\end{document}